\documentclass{article}
\usepackage{fullpage}
\usepackage{float}
\usepackage[square,numbers,sort]{natbib}
\usepackage[ruled,vlined]{algorithm2e}
\usepackage{0.packages} %
\usepackage{enumitem}
\usepackage{setspace} %
\usepackage[square]{natbib}
\setcitestyle{authoryear,open={(},close={)}}

\newif\iffinal
\finaltrue  %
\iffinal
    \newcommand{\chicheng}[1]{}
    \newcommand{\hao}[1]{}
    \newcommand{\edit}[2]{}
\else
    \newcommand{\chicheng}[1]{{\color{red}{Chicheng: #1}}}
    \newcommand{\hao}[1]{{\color{orange}{Hao: #1}}}
    \newcommand{\edit}[2]{{\color{seared2}{#1}{\color{magenta}{#2}}}}
    \newcommand{\mask}[1]{\green{masked content}}
\fi

\title{Taming the Monster Every Context: Complexity Measure and Unified Framework for Offline-Oracle Efficient Contextual Bandits}
\author{
\begin{tabular}{c}
{\Large Hao Qin~~~~~~~~~~Chicheng Zhang}\\
The University of Arizona\\
{\small\texttt{\{hqin,chichengz\}@arizona.edu}}
\end{tabular}
}
\date{}

\renewcommand{\hat}[1]{\widehat{#1}}

\begin{document}

\maketitle
\newcommand{\Edim}{\mathrm{Edim}}
\providecommand{\poly}{\ensuremath{\mathrm{poly}}}
\providecommand{\RP}{\ensuremath{\mathsf{RP}}}
\providecommand{\NP}{\ensuremath{\mathsf{NP}}}
\def\Mfrak{\mathfrak{M}}
\newcommand{\st}{\ensuremath{\text{s.t. }}\xspace}

\newcommand{\SEC}{\ensuremath{\mathsf{SEC}}}
\newcommand{\WSEC}{\ensuremath{\mathsf{W}\text{-}\mathsf{SEC}}}
\newcommand{\coverage}{\ensuremath{\mathsf{Coverage}}}
\newcommand{\pcoverage}{\ensuremath{{\mathsf{P}\text{-}\mathsf{Coverage}}}}
\newcommand{\WL}{\mathrm{SqErr}}
\newcommand{\dec}{\ensuremath{\mathsf{dec}}}
\newcommand{\doec}{\ensuremath{\mathsf{doec}}}

\newcommand{\zooming}{\ensuremath{\textsc{zooming dimension}}\xspace}
\newcommand{\spannerg}{\ensuremath{\textsc{SpannerGreedy}}\xspace}
\newcommand{\spannerigw}{\ensuremath{\textsc{SpannerIGW}}\xspace}
\newcommand{\smoothg}{\ensuremath{\textsc{SmoothGreedy}}\xspace}
\newcommand{\smoothigw}{\ensuremath{\textsc{SmoothIGW}}\xspace}
\newcommand{\maxenttl}{\ensuremath{\textsc{MaxEntRL}}\xspace}
\newcommand{\ilovetoconbandits}{\ensuremath{\textsc{ILOVETOCONBANDITS}}\xspace}
\newcommand{\squarecb}{\ensuremath{\textsc{SquareCB}}\xspace}
\newcommand{\falcon}{\ensuremath{\textsc{Falcon}}\xspace}
\newcommand{\linearfalcon}{\ensuremath{\textsc{Linear-Falcon}}\xspace}
\newcommand\barbelow[1]{\stackunder[1.2pt]{$#1$}{\rule{.8ex}{.075ex}}}
\newcommand{\generalfalcon}{\ensuremath{\textsc{OE2D}}\xspace}
\newcommand{\glmoetwod}{\ensuremath{\textsc{GLM-OE2D}}\xspace}
\newcommand{\smoothedoetwod}{\ensuremath{\textsc{Smoothed-OE2D}}\xspace}
\newcommand{\sfa}{\ensuremath{\textsc{S-OE2D}}\xspace}
\newcommand{\sig}{\ensuremath{\textsc{S-IGW}}\xspace}

\newcommand{\generalsquarecb}{\ensuremath{\textsc{SquareCB.F}}\xspace}
\newcommand{\squarecbf}{\ensuremath{\textsc{SquareCB.F}}\xspace}
\newcommand{\etwod}{\ensuremath{\textsc{E2D}}\xspace}
\newcommand{\etwodoff}{\ensuremath{\textsc{E2D.Off}}\xspace}
\newcommand{\alg}{\ensuremath{\textsc{Alg}}\xspace}

\newcommand{\linear}{\mathrm{Linear}}
\newcommand{\glinear}{\mathrm{G.Linear}}

\newcommand{\Regsq}{\mathrm{Reg_{sq}}}
\newcommand{\Regkl}{\mathrm{Reg_{KL}}}
\newcommand{\logbarrier}{\mathrm{LogBarrier}}
\newcommand{\LLOG}{\ensuremath{\textsc{LOG}}\xspace}
\newcommand{\REG}{\mathrm{Reg}}
\newcommand{\HREG}{\widehat{\mathrm{Reg}}}
\newcommand{\diag}{\mathrm{diag}}
\newcommand{\Polylog}[1]{\ensuremath{\textsc{Polylog}(#1)}}
\newcommand{\trace}{\mathrm{trace}}
\newcommand{\cov}{\ensuremath{\mathsf{Cov}}}
\newcommand{\elu}{\mathrm{elu}}
\newcommand{\co}{\mathrm{co}}
\newcommand{\HRcal}{\widehat{\Rcal}}

\newcommand{\onlinereg}{\mathrm{Reg_{on}}}
\newcommand{\offlinereg}{\mathrm{Reg_{off}}}
\newcommand{\onlineoracle}{\mathcal{O}_{\mathrm{on}}}
\newcommand{\offlineoracle}{\mathcal{O}_{\mathrm{off}}}

\newcommand{\defgamma}{\sqrt{\frac{\overline{\SEC} \tau_m}{\offlinereg}}}
\newcommand{\defepsilon}{\frac{\overline{\SEC}}{\gamma_M^2 M}}
\newcommand{\terminate}{\lfloor 1/\varepsilon \rfloor}

\newcommand{\polylog}{\mathrm{polylog}}

\begin{abstract}
    We propose an algorithmic framework, Offline Estimation to Decisions (\generalfalcon), that efficiently reduces contextual bandit learning with general reward function approximation to offline regression. The framework allows near-optimal regret for contextual bandits with large action spaces with
    $O(\log T)$ calls to an offline regression oracle over $T$ rounds, and makes $O(\log\log T)$ calls when $T$ is known.
    The design of \generalfalcon algorithm generalizes \falcon~\citep{simchi2022bypassing} and its linear reward version~\citep[][Section 4]{xu2020upper} in that it finds an action distribution that we term ``exploitative F-design'' that simultaneously guarantees low regret and good coverage, striking a balance between exploration and exploitation.
    Central to our regret analysis is a new complexity measure, the Decision-Offline Estimation Coefficient (DOEC), which we show is small in many settings, including bounded Eluder dimension per-context and the smoothed regret setting. We also establish a relationship between DOEC and Decision Estimation Coefficient (DEC)~\citep{foster2021statistical}, bridging the design principles of offline- and online-oracle efficient contextual bandit algorithms for the first time.
\end{abstract}

\section{Introduction and Related Work}
The online contextual bandit learning problem, a one-step version of online reinforcement learning (RL), has garnered a lot of attention due to its usage in modern applications such as online advertising, recommendations, mobile health~\citep{tewari2017ads,li2010contextual}.
In this problem, a learning agent at each time step $t$ receives a context $x_t$ from the context space $\Xcal$, takes an action $a_t$ from the action space $\Acal$, and receives the reward $r_t$ of the action taken.
The goal of the learning agent is to take actions adaptively based on its historical information, so as to learn to maximize its total reward over a time horizon of $T$.
At the beginning, the true reward distribution associated with each context and action is unknown to the learning agent, and thus it needs to take informative actions  for learning the reward function (exploration) while reaping high rewards (exploitation).

Early research in contextual bandits mainly focused on designing algorithms that search over a class of policies~\citep{auer03nonstochastic,langford2007epoch,agarwal14taming}, aiming to learn the optimal policy in the class based on data collected adaptively online.
Recently, regression-based contextual bandit algorithms~\citep{foster20beyond,simchi2022bypassing} emerge as a practical alternative that allows computationally-efficient implementation with impressive empirical performance~\citep{bietti18contextual,foster2020instance,foster2021efficient}.
Herein, the learner has access to a class of regression functions $\Fcal$ that predicts reward from context and action (which approximates the ground truth reward function $f^*(x,a) = \EE\sbr{ r_t \mid x_t = x, a_t = a }$), maintains some reward function estimate $\hat{f}_t$ based on historical data, and uses $\hat{f}_t$ to guide the selection of action $a_t$ to balance between exploration and exploitation.
A standard performance measure of a contextual bandit algorithm is its \emph{regret}, i.e., the difference between the best cumulative reward achievable had we known $f^*$ ahead of time, and the cumulative reward of the learning agent.

To design efficient contextual bandit algorithms, the research community adopted the
\emph{oracle-efficiency} framework: assuming access to some computational oracles that can solve basic regression problems, we aim to design online contextual bandit algorithms that have a small total running time, with a small number of calls to the oracles. Such regression oracles can be implemented by standard machine learning libraries in practice~\citep{varoquaux2015scikit,paszke2019pytorch}.
Two types of regression oracles are of main interest.
First, an \textit{online regression oracle} receives a stream of (context, action, reward) tuples and
maintains a reward predictor on the fly, such that its online prediction error is small, e.g., $o(T)$. Second, an \textit{offline regression oracle} takes a batch of iid (context, action, reward) tuples and outputs a reward predictor that approximately minimizes its prediction error on the population where these samples are drawn from. Recent research has designed contextual bandit algorithms with an error-limiting property~\citep{beygelzimer2005error}: its regret would be small as long as the error of the online or offline regression problem is small.
Comparing these two types of oracles, assuming access to an offline regression oracle with low error is milder, due to: (1) the richer availability of guarantees for offline regression from statistical learning~\citep{yang1999information,rakhlin2017empirical}; (2) the simplicity of algorithms in implementing an offline regression oracle~\citep{foster2024online}. On the other hand, guarantees for offline-oracle efficient contextual bandit algorithms have been established in the more restrictive setting that the contexts are  iid~\citep{simchi2022bypassing,xu2020upper}.

Despite impressive progress, there still remain quite a few open questions in the design and analysis of oracle-efficient contextual bandit learning algorithms with general reward function approximation:

\begin{itemize}
\item There is no unified framework for designing offline oracle-efficient contextual bandit algorithms with only a few offline oracle calls, especially in the presence of large action spaces.
It is known that the elegant optimism principle~\citep{abbasi2011improved,dani08stochastic,russo13eluder} may fail to achieve a sublinear regret under general function approximation~\cite[][Example 3.1]{foster2022statistical}.
Although~\citet{xu2020upper} provides a unified framework for offline oracle-efficient contextual bandits, their method makes $O(T)$ calls to the offline regression oracle, which is computationally costly and known to be suboptimal -- for example, in the finite action space setting,~\citet{simchi2022bypassing} achieves the same optimal regret by making $O(\log T)$ or even $O(\log\log T)$ calls to the regression oracle. How can we design general offline oracle-efficient contextual bandit algorithms that only make a few calls to the offline regression oracle?

\item Connections between the design principles underlying online- and offline-oracle efficient algorithms are currently missing.
For example, when the action space $\Acal$ is small and discrete, the inverse gap weighting (IGW) exploration strategy has been analyzed in both the online-oracle efficient algorithm of~\citet{foster20beyond} and the offline-oracle efficient algorithm of~\citet{simchi2022bypassing}. But their analyses look drastically different:~\citet{foster20beyond} views IGW as an instance of the Estimation-to-Decision (\etwod) principle~\citep{foster2021statistical},
while \citet{simchi2022bypassing}'s analysis draws a parallel to a policy search-based algorithm ``Taming the Monster''~\citep{agarwal14taming}.
This raises the question: are these two facets of IGW a coincidence, or do the design principles underlying online- and offline-oracle efficient contextual bandit exploration actually have some connections?
\end{itemize}

\noindent In this paper, we make significant progress in answering these questions:
\begin{itemize}
\item We design a unified algorithm, \generalfalcon (Offline-Estimation-to-Decision, Alg.~\ref{alg:general-falcon}), for offline regression oracle-efficient contextual bandits with general reward function classes, which generalizes \falcon~\citep{simchi2022bypassing} and its linear-reward extension~\citep[][Section 4]{xu2020upper}. Key to our algorithm design is the computation of an exploration distribution that satisfies both low-regret and good-coverage properties, implicitly defined for every context $x$. It generalizes the pure-exploration-oriented F-design~\citep{agarwal2024non} to accommodate exploitation, and thus we name our exploration distribution ``exploitative F-design''. For computational efficiency considerations, we also identify and solve tractable relaxations of exploitative F-design in several examples.

\item We establish regret guarantees of \generalfalcon (Theorem~\ref{thm:main-regret}), and show that it not only recovers existing guarantees of offline-oracle-efficient contextual bandit algorithms~\citep{simchi2022bypassing,xu2020upper}
but also obtains many new guarantees. To highlight: we obtain \smoothedoetwod and \glmoetwod, the first offline-oracle-efficient algorithms with $O(\log T)$ calls to the offline regression oracle
in the $h$-smoothed-regret~\citep{krishnamurthy2020contextual} and  per-context generalized  linear reward settings,
using the \generalfalcon framework.
In our analysis, we
reveal a regret bound of \generalfalcon that \emph{simultaneously holds for every context} $x$, which may be of independent interest. We also show that our modular theorem statement allows \generalfalcon to handle variants of the contextual bandits problem, such as model misspecification, reward corruption, and context distribution shifts (See Appendix~\ref{sec:extensions} for the exact assumptions).

\item Key to our regret analysis is a new complexity measure, Decision-Offline Estimation Coefficient (DOEC, Definition~\ref{def:doec}), that characterizes the statistical cost of reducing online contextual bandit learning to offline regression.
Different from the Decision Estimation Coefficient~\cite[DEC,][]{foster2021statistical} whose objective refers to the square error of the central model, DOEC's definition does not refer to the square error of the central model, enabling reduction to offline estimation.
We establish new structural results of DOEC (Theorem~\ref{thm:doec-leq-sec}), specifically relating it to a modification of the Sequential Extrapolation Coefficient (SEC)~\citep{xie2022role}, which we name relaxed $\varepsilon$-SEC (Definition~\ref{def:delta-sec}).
We show that relaxed $\varepsilon$-SEC is bounded in many examples.
In view that $\varepsilon$-SEC is a ``passive'' measure of exploration complexity, while DOEC allows active experimental design,
we also present an example that demonstrates that $\varepsilon$-SEC alone may not be adequate to characterize when regret-efficient contextual bandit with offline oracle is feasible.

\item We establish a general theorem that links DOEC to DEC (Theorem~\ref{thm:dec-leq-doec}): any exploration distribution $p$ that certifies small DOEC also certifies small DEC, and DEC is at most DOEC up to lower order additive terms.
We believe this is an interesting observation, as this bridges the design principles of offline- and online- oracle efficient contextual bandit algorithms for the first time.

\end{itemize}

We discuss additional related work in Appendix~\ref{sec:related-work}.

\section{Preliminaries}
\label{sec:preliminaries}

\paragraph{Basic Notations}
Denote by $[N] := \cbr{1,\ldots,N}$.
We say $f = \iupbound{g}$ if $f \leq cg$ for some constant $c$.
We say $f = \iupboundlog{g}$ if $f =\iupbound{g \cdot \mathrm{polylog}(g)}$.
When $f = \iupboundlog{g}$ we also write $f \lesssim g$ or $g \gtrsim f$.
The convex hull of a set $\Scal$ is defined as $\co(\Scal) = \cbr{\sum_{i} \alpha_i s_i \mid s_i \in \Scal, \alpha_i \geq 0, \sum_i \alpha_i = 1}$.
Denote $\Delta(\Ycal)$ as the space of all probability distributions over $\Ycal$.
For nonnegative measures $p, q$ on a common measurable space, we write $p \succeq q$ (equivalently $q \preceq p$) if $p(S) \geq q(S)$ for every measurable set $S$.

\paragraph{Basic Assumptions}
The learning agent has access to a class of reward functions $\Fcal$ where each $f \in \Fcal$ is a mapping from context-action pairs to rewards, i.e., $f: \Xcal \times \Acal \to [0,1]$\footnote{Throughout the paper, for simplicity of presentation, we focus on settings where the action space $\Acal$ is finite, but we aim to achieve regret independent of $|\Acal|$, making it suitable for large action space applications.}.
Throughout the paper, unless otherwise specified, we assume \emph{realizability} and no context-distribution shift, such that the ground truth reward function $f^*$ is in $\Fcal$, and the contexts are drawn iid from a distribution $\Dcal_X$.
We assume $\Fcal$ does not have any specific structure but we give some examples of structured reward function classes after presenting the main theoretical results, including per-context linear reward model~\citep{demirer2019semi,zhu2022contextualb} and per-context generalized linear model~\citep{xu2020upper}, which generalizes the globally linear and generalized linear reward models~\citep{filippi10parametric,abbasi2011improved}.

\paragraph{Main Performance Measure: Regret}
To provide a unified treatment on standard regret and smoothed regret~\citep{zhu2022contextuala} in the literature, we consider a general notion of regret called $\Lambda$-Regret, which measures the performance of a contextual bandit algorithm \alg against the best action distribution per context in a predefined benchmark space of distributions $\Lambda \subset \Delta(\Acal)$:
    \[
        \Regret_\Lambda(T, \alg)
        =
        \sum_{t=1}^T
        \REG(p_t \mid x_t),
    \]
where $\REG(p \mid x) = \max_{\lambda \in \Lambda}\EE_{a \sim \lambda}\sbr{f^*(x, a)} - \EE_{a \sim p}\sbr{f^*(x, a)}$ is the instantaneous regret of $p$ on context $x$.
\footnote{Our framework also handles the setting where the benchmark distributions are context-dependent; we present relevant results in Appendix~\ref{app:context-dependent-benchmark}.
}

\paragraph{Running Examples.} To help illustrate our results, we will frequently refer to the following three examples throughout this paper:
\begin{enumerate}
\item Discrete action space, standard regret~\citep{foster20beyond,simchi2022bypassing}. $\Lambda$ is the set of all point mass distributions over the action space $\Acal$, i.e., $\Lambda = \icbr{\delta_a: a \in \Acal}$.
\item Per-context generalized linear reward structure~\citep{xu2020upper}. Let $\Lambda = \icbr{\delta_a: a \in \Acal}$, and assume that
$\Fcal = \cbr{f_\theta(x,a) = \sigma\del{\phi(x,a)^\top \theta(x)} \mid \theta \in \Theta}$,
where $\sigma$ is a known link function with derivative in $[\underline{L}, \overline{L}]$ for constants $0 < \underline{L} \leq \overline{L}$ (and $\kappa := \overline{L}/\underline{L}$), $\phi: \Xcal \times \Acal \to \RR^d$ is a known feature map with $\norm{\phi(x,a)}_2 \leq 1$, and $\Theta$ is a class of maps $\Xcal \to \RR^d$ with per-context diameter $\sup_{\theta, \theta' \in \Theta} \norm{\theta(x) - \theta'(x)}_2 \leq B$ for all $x \in \Xcal$. Taking $\sigma$ to be the identity recovers the per-context linear reward structure~\citep{demirer2019semi,zhu2022contextualb,zhang2022feel}.
\item $h$-smoothed regret~\citep{krishnamurthy2020contextual,zhu2022contextuala} $\Lambda = \Delta_h^{\mu}(\Acal) := \icbr{\lambda \in \Delta(\Acal): \frac{\diff \lambda}{\diff \mu}(a) \leq 1/h, \forall a \in \Acal}$  denotes the set of all $h$-smoothed distributions w.r.t. $\mu$.
\end{enumerate}
Note that Examples 2 and 3 above can be viewed as orthogonal ways to generalize Example 1: Example 2 with $\sigma$ being the identity function, $d = |\Acal|$, and $\phi(x,a) = e_a$, the $a$-th canonical basis, recovers Example 1; Example 3 with $\mu$ being the uniform distribution over $\Acal$ and $h = 1/|\Acal|$ also gives back Example 1.

\paragraph{Regression Oracles}
We consider two types of regression oracles that have been widely used in contextual bandit literature: %
\noindent 1) an offline regression oracle $\offlineoracle(\Fcal)$ takes in (context, action, reward) tuples drawn i.i.d. from distribution $\Dcal$ and outputs a function $\hat{f} \in \Fcal$, s.t., with probability at least $1-\delta$,
    $
        \EE_\Dcal \isbr{\idel{ \hat{f}(x, a) - f^*(x, a) }^2} \lesssim \offlinereg(\Fcal, T, \delta).
    $
\noindent 2) An online regression oracle $\onlineoracle(\Fcal)$ takes as input a sequence of (context, action, reward) tuples generated in an online manner and outputs a sequence of functions $\cbr{f_t}_{t=1}^T$ s.t., with probability at least $1-\delta$,
    $
        \EE\isbr{\sum_{t=1}^T (\hat{f}_t(x_t, a_t) - r_t)^2 - \inf_{f \in \Fcal} \sum_{t=1}^T (f^*(x_t, a_t) - r_t)^2} \leq \onlinereg(\Fcal, T, \delta).
    $
It is well-known that when $|\Fcal|$ is finite, $\offlinereg(\Fcal, T, \delta)$ and $\onlinereg(\Fcal, T, \delta)$ can be $\upbound{\log(\frac{|\Fcal|}{\delta})/T}$~\citep{agarwal12contextual} and $\upbound{\log(\frac{|\Fcal|}{\delta})}$~\citep{cesa-bianchi06prediction}, respectively. Since most learning algorithms have better guarantees with a larger sample size, we assume that $\offlinereg$ is monotonically decreasing in $T$ and increasing in $\delta$ in general.

\paragraph{Decision-Estimation Coefficient (DEC)} The Decision-Estimation Coefficient~\citep{foster2021statistical} is a key measure to characterize the complexity of exploration in online decision making problems. Here we present a version using square loss and benchmark distribution space~\citep{zhu2022contextuala}:
\begin{definition}
\label{def:dec}
Given a reward function class $\Gcal$ mapping from $\Acal$ to $\RR$, benchmark distribution space $\Lambda$, $\gamma > 0$, reward function $\hat{g}$,
their Decision-Estimation Coefficient (DEC) is defined as
    \begin{align}
        \dec_{\gamma}(\Gcal, \hat{g}, \Lambda) = \inf_{p \in \Delta(\Acal)} \sup_{g^* \in \Gcal} \EE_{a \sim p} \sbr{ \max_{\lambda' \in \Lambda}\EE_{a' \sim \lambda'}[g^*(a')] - g^*(a) - \gamma \del{\hat{g}(a) - g^*(a)}^2}
    \end{align}
    and $\dec_{\gamma}(\Gcal, \Lambda) = \max_{\hat{g} \in \Gcal} \dec_{\gamma}(\Gcal, \hat{g}, \Lambda)$.
\end{definition}
It is well-known that the \etwod algorithm~\citep{foster2021statistical} has a small regret whenever the $\dec(\Fcal_x, \Lambda)$ is small for every $x$, and the online regression oracle has a low regret:
$
\REG(T, \etwod) \leq \min_{\gamma > 0} ( T \max_x \dec_{\gamma}(\Fcal_x, \Lambda) + \gamma \onlinereg(\Fcal, T, \delta) ).
$
For the above three examples, $\max_x \dec_{\gamma}(\Fcal_x, \Lambda)$ are $\lesssim \frac{|\Acal|}{\gamma}$, $\frac{d}{\gamma}$, and $\frac{1}{\gamma h}$ respectively~\citep{foster2021statistical,zhu2022contextuala}. If the class $\Fcal$ is finite so that setting online regression oracle to be exponential weight has $\onlinereg(\Fcal, T, \delta) = O(\log|\Fcal|)$, \etwod has regret bounds $\lesssim \sqrt{|\Acal| T \log|\Fcal|}$, $\sqrt{d T \log|\Fcal|}$, and $\sqrt{\frac{T}{h} \log|\Fcal|}$ respectively.

\paragraph{Coverage between Distributions}
Our work will use basic tools in offline policy evaluation (OPE) for contextual bandits. In the simplest (non-contextual) structured bandit setting, OPE aims to estimate the expectation of  reward function $g^*$ over a \emph{target action distribution} $q \in \Delta(\Acal)$,
by drawing noisy measurements of $g^*(a)$ for $a$'s sampled from some \emph{behavior action distribution} $p \in \Delta(\Acal)$.
Following prior works in offline reinforcement learning with function approximation~\citep{song2022hybrid,xie2022role} and experimental design~\citep{agarwal2024non}, we define the coverage of $p$ over $q$ with respect to function class $\Gcal$ and a constant $\varepsilon > 0$ as
    \begin{align}
        \coverage_\varepsilon(p,q; \Gcal) = \sup_{g, g' \in \Gcal} \frac{\del{\EE_{a \sim q}\sbr{g(a) - g'(a)}}^2}{\varepsilon + \EE_{a \sim p} \sbr{(g(a) - g'(a))^2}}
            \label{eqn:coverage}
    \end{align}

A smaller $\coverage_{\varepsilon}(p, q; \Gcal)$ indicates that samples from $p$ provide more information in evaluating the expected reward of distribution $q$.
Given function class $\Gcal$ and benchmark distribution $\Lambda$, the nonlinear F-design~\citep{agarwal2024non} aims to find a distribution $p$ such that it minimizes the worst-case coverage over all possible $q$'s in $\Lambda$:
$
p^* = \argmin_{p \in \co(\Lambda)} \max_{q \in \Lambda} \coverage_\varepsilon(p, q; \Gcal)
$
and we denote the optimal objective as $\Vcal_\varepsilon^*(\Gcal, \Lambda)$~\footnote{Our definition slightly generalizes~\citet{agarwal2024non} in that we incorporate a benchmark distribution class $\Lambda$; when $\Lambda = \cbr{\delta_a: a \in \Acal}$ our definition becomes theirs.}.

\section{The \generalfalcon contextual bandit algorithm and its guarantees}
\label{sec:offline-oracle-efficient-algorithm}

We present our main algorithm, \generalfalcon (\Cref{alg:general-falcon}) that efficiently deals with large action spaces and general reward function classes in contextual bandit problems using a few calls to offline regression oracles.
In addition to standard inputs, it also takes in a relaxed coverage function $\overline{\coverage}$, an upper bound of $\coverage$; this is for computational efficiency considerations, as we will discuss shortly.
The algorithm proceeds in epochs.
At the beginning of each epoch $m \geq 2$, we call the offline regression oracle to obtain a reward function estimate $\hat{f}_m$ by minimizing the square loss over data collected in the previous epoch $m-1$ (lines~\ref{line:for}-\ref{line:reg}); we use the convention that $\hat{f}_1 \equiv 0$.
During epoch $m$, at each step $t$, we use $\hat{f}_m$ together with the observed context $x_t$ to construct an action distribution $p_t$ that solves a minimax optimization problem (Eq.~\eqref{eqn:monster-distn}) (line~\ref{line:doec}); we will discuss its rationale in the next paragraph.
We then sample an action $a_t$ from $p_t$ and observe its reward $r_t$ (line~\ref{line:sample}).
This process is repeated for each round within the epoch, and the algorithm proceeds to the next epoch until the time horizon is reached.

\paragraph{Relaxed Exploitative F-design}
The two terms in the objective function in Eq.~\eqref{eqn:monster-distn} encourage exploitation and exploration, respectively.
Specifically, minimizing $\EE_{a \sim \lambda}\isbr{\hat{f}_m(x,a)} - \EE_{a \sim p}\isbr{\hat{f}_m(x,a)}$ encourages $p_t$ to be greedy with respect to reward function $\hat{f}_m$; minimizing $\tfrac{1}{\gamma_m}\overline{\coverage}_{\varepsilon_m}(p, \lambda; \Fcal_{x_t})$ encourages $p_t$ to provide decent coverage to all distributions $\lambda \in \Lambda$.
Parameter $\gamma_m$ serves as a tuning hyperparameter: a larger (resp. smaller) $\gamma_m$ implies a higher degree of exploitation (resp. exploration).
When $\gamma_m \to 0$, minimizing Eq.~\eqref{eqn:monster-distn} amounts to minimizing
$\max_{\lambda \in \Lambda} \overline{\coverage}_{\varepsilon_m}(p, \lambda; \Fcal_{x_t})$, which is equivalent to finding a relaxed variant of the F-design~\citep{agarwal2024non}, previously proven useful for pure exploration in contextual bandits.
In light of this connection, we name our optimization problem \eqref{eqn:monster-distn}
``relaxed exploitative F-design'' due to its additional encouragement of exploitation.
Its unrelaxed counterpart, the \emph{exploitative F-design}, in which the original coverage $\coverage$ takes the place of $\overline{\coverage}$, is recovered as the special case of the trivial relaxation $\overline{\coverage} = \coverage$.
Finally, the constraint that $p_t \in \co(\Lambda)$ allows $p_t$'s expected reward to be evaluated with good accuracy.

\paragraph{Relaxed Coverage and Computational Efficiency}
Recall $\coverage_{\varepsilon}$ defined in Eq.~\eqref{eqn:coverage}: obtaining its value at even a single pair $(p, \lambda)$ requires solving an optimization problem, namely a maximization over pairs of functions $g, g' \in \Gcal$. For general function classes, this is a nonconcave maximization problem that can be computationally costly to solve, let alone the minimax problem (Eq.~\eqref{eqn:monster-distn}) built on top of it; see
Propositions~\ref{prop:coverage-hardness} and~\ref{prop:doec-hardness} in Appendix~\ref{app:computational-hardness} for more details.

This motivates our definition of \emph{relaxed coverage}, $\overline{\coverage}$: a tractable upper bound of the original coverage, i.e., $\overline{\coverage}_{\varepsilon}(p, \lambda; \Gcal) \geq \coverage_{\varepsilon}(p, \lambda; \Gcal)$ for all $p \in \co(\Lambda)$ and $\lambda \in \Lambda$. An appropriate choice of $\overline{\coverage}$ allows us to design a computationally efficient algorithm while not inflating the downstream regret guarantee by too much. Such relaxations exist for many settings of interest. For example, in our three running examples, we can define:

\begin{enumerate}
    \item Discrete action space, $\overline{\coverage}_{\varepsilon}(p, \lambda; \Fcal_x) \coloneqq \sum_{a \in \Acal} \frac{\lambda(a)}{p(a)}$;
    \item Per-context generalized linear reward, $\overline{\coverage}_{\varepsilon}(p, \lambda; \Fcal_x) \coloneqq \kappa^2 \tr\del{\Sigma_p^{-1} \Sigma_\lambda}$, where $\Sigma_q := \EE_{a \sim q}\sbr{\phi(x,a)\phi(x,a)^\top}$;
    \item $h$-smoothed regret,
    $\overline{\coverage}_{\varepsilon}(p, \lambda; \Fcal_x) \coloneqq \frac{1}{h} \EE_{a \sim \mu}\sbr{\frac{\lambda(a)}{p(a)}}$, where $\lambda(a)$ and $p(a)$ denote the densities of $\lambda$ and $p$ with respect to the base measure $\mu$.
\end{enumerate}

The validity of these relaxations is straightforward and we defer the detailed proof to Lemma~\ref{lem:relaxed-coverage-valid} in Appendix~\ref{app:relaxed-coverage}. In all three running examples, $\overline{\coverage}$ is convex in $p$ and linear in $\lambda$, which makes the induced minimax problem (Eq.~\eqref{eqn:monster-distn}) efficiently solvable.
Furthermore, the solution to Eq.~\eqref{eqn:monster-distn} in each of the three examples coincides with that of a barrier-regularized linear optimization problem (see Lemma~\ref{lem:relaxed-doec-bound}); similar observations have been made in prior works~\citep{kiefer1960equivalence,xu2020upper,levy2023efficient}.
Meanwhile, the trivial relaxation $\overline{\coverage} = \coverage_{\varepsilon}$ is always valid, and remains interesting whenever the original coverage is itself tractable; for instance, in the ``tabular'' structured bandit setting~\citep{lattimore2014bounded,jun2020crush, tirinzoni2020novel}
, where $\Gcal$ is represented by a matrix whose rows are functions in $\Gcal$ and columns are actions in $\Acal$, any coverage can be computed in time polynomial in the size of the matrix.

\begin{algorithm}[t]
\setstretch{0.95}
\setlength{\abovedisplayskip}{2pt}
\setlength{\belowdisplayskip}{2pt}
\setlength{\abovedisplayshortskip}{0pt}
\setlength{\belowdisplayshortskip}{2pt}
\LinesNumbered
\SetAlgoLined
\SetAlgoVlined
    \KwIn{learning parameter $\cbr{\gamma_m, \varepsilon_m}_{m=1}^M$, epoch schedule $0 = \tau_0 < \tau_1 < \cdots < \tau_M$, benchmark distributions $\Lambda$, reward function class $\Fcal$, offline regression oracle $\offlineoracle$,
    relaxed coverage $\overline{\coverage}_{\varepsilon}$.}
    \For{$t = 1$ \KwTo $\tau_{1}$ \algcomment{0}{No historical data for the first epoch, pure exploration only}}
    { Observe context $x_t \in \Xcal$, and sample action $a_t$ from a relaxed F-design distribution:
    \[
        p_t = \argmin_{p \in \co(\Lambda)} \max_{\lambda \in \Lambda} \overline{\coverage}_{\varepsilon_1}(p, \lambda; \Fcal_{x_t}).
    \]
    }

    \For{$m = 2$ \KwTo $M$ \algcomment{0}{Construct reward estimator via offline regression oracle}\label{line:for}}{
        Compute $\hat{f}_m \gets \offlineoracle(\Fcal) ( \cbr{ (x_i, a_i, r_i) }_{i=\tau_{m-2}+1}^{\tau_{m-1}} ) $.

        \For{$t = \tau_m + 1$ \KwTo $\tau_{m+1}$\label{line:reg} \algcomment{0}{Construct distribution $p_t$ which is the solution of relaxed exploitative F-design}}{
            Observe context $x_t \in \Xcal$.
            Let $p_t$ be the solution of the following problem: \label{line:doec}
            \begin{align}
                p_t = \argmin_{p \in \co(\Lambda)}
                \max_{\lambda \in \Lambda }
                \rbr{
                \EE_{a \sim \lambda}\sbr{\hat{f}_m(x_t,a)}
                -
                \EE_{a \sim p}\sbr{\hat{f}_m(x_t,a)}
                +
                \frac{\overline{\coverage}_{\varepsilon_m}(p, \lambda; \Fcal_{x_t})}{\gamma_m}
                }.
                \label{eqn:monster-distn}
            \end{align}
            Sample action $a_t \sim p_t$ and observe reward $r_t$.\label{line:sample}
        }
    }
    \caption{\generalfalcon: Offline Estimation to Decision
    } \label{alg:general-falcon}
\end{algorithm}

\paragraph{Decision-Offline Estimation Coefficient (DOEC): a New Complexity Measure}
The minimax value of the (relaxed) exploitative F-design (Eq.~\eqref{eqn:monster-distn}) turns out to be important for establishing the regret guarantee of \generalfalcon, which we formally introduce as follows:
\begin{definition}[DOEC and Relaxed DOEC]
\label{def:doec}
Given a class of functions $\Gcal: \Acal \to [0,1]$, a class of benchmark distributions $\Lambda \subset \Delta(\Acal)$, a $\hat{g} \in \Gcal$,
their
\emph{Decision-Offline-Estimation Coefficient (DOEC)} is:
\[
    \doec_{\gamma, \varepsilon}(\hat{g}, \Gcal, \Lambda)
    =
    \min_{p \in \co(\Lambda)}
    \max_{\lambda \in \Lambda }
    \rbr{
    \EE_{a \sim \lambda}\sbr{\hat{g}(a)}
    -
    \EE_{a \sim p}\sbr{\hat{g}(a)}
    +
    \tfrac{1}{\gamma}\coverage_{\varepsilon}(p, \lambda; \Gcal
    )},
\]
and
$
    \doec_{\gamma, \varepsilon}(\Gcal, \Lambda) = \max_{\hat{g} \in \Gcal} \doec_{\gamma, \varepsilon}(\hat{g}, \Gcal, \Lambda)
$.
Given a choice of relaxed coverage $\overline{\coverage}$, we define its associated relaxed DOEC, $\overline{\doec}_{\gamma, \varepsilon}(\hat{g}, \Gcal, \Lambda)$ and $\overline{\doec}_{\gamma, \varepsilon}(\Gcal, \Lambda)$ to be their unrelaxed counterparts with $\coverage$ replaced by $\overline{\coverage}$.
\end{definition}

With $\Gcal = \Fcal_{x_t}$ and $\hat{g} = \hat{f}_m(x_t, \cdot)$, $\overline{\doec}_{\gamma_m, \varepsilon_m}(\hat{g}, \Gcal, \Lambda)$ is exactly the minimax value of the relaxed exploitative F-design (Eq.~\eqref{eqn:monster-distn}). Since any valid relaxed coverage upper bounds the original coverage, their associated relaxed DOEC upper bounds the original DOEC as well.
And thus, any minimizer $p$ of the relaxed DOEC optimization problem also \emph{certifies} an upper bound of the original DOEC.

Our name ``Decision-Offline-Estimation Coefficient'' is inspired by DEC (Definition~\ref{def:dec}).
Similar to DEC, the DOEC objective consists of a ``decision'' term, $\EE_{a \sim \lambda}\sbr{\hat{g}(a)} - \EE_{a \sim p}\sbr{\hat{g}(a)}$, measuring the greediness of decision $p$, and an ``estimation'' term, $\frac{1}{\gamma}\coverage_{\varepsilon}(p, \lambda; \Gcal)$ measuring the usefulness of $p$ in estimating the ground truth reward function. However, there are some key differences: DEC takes an additional maximum over all possible ground truth reward functions $g^*$, and the estimation cost is measured by the central model $\hat{g}$'s error  $\EE_{a \sim p}\sbr{(\hat{g}(a) - g^*(a))^2}$. This error term enables a reduction from contextual bandits to \emph{online} regression, and it is not clear how we can use it to reduce contextual bandits to offline regression.
In contrast, DOEC measures the estimation overhead in terms of coverage, which avoids the reference to the central model $\hat{g}$, enabling reduction to offline estimation.
For the three running examples, with the relaxed coverages introduced above, we show in Lemma~\ref{lem:relaxed-doec-bound} that $\overline{\doec}_{\gamma, \varepsilon}(\Fcal_x, \Lambda)$ is \emph{equal to} $\frac{|\Acal|}{\gamma}, \frac{\kappa^2 d}{\gamma}$, and $\frac{1}{\gamma h}$, respectively.

We now come back to explain the utility of the relaxed exploitative F-design (Eq.~\eqref{eqn:monster-distn}). Define
$\HREG_m(p \mid x) = \max_{\lambda' \in \Lambda} \EE_{a' \sim \lambda'} \isbr{\hat{f}_m(x, a')} - \EE_{a \sim p}\isbr{ \hat{f}_m(x, a) }$ to be the \emph{empirical regret} of distribution $p$ in context $x$, according to reward estimate $\hat{f}_m$.
With this notation, the  objective function in Eq.~\eqref{eqn:monster-distn} can be equivalently written as the sum of two nonnegative terms\footnote{To see why the second term is nonnegative, note that we can choose $\lambda = \argmax_{\lambda' \in \Lambda} \EE_{a' \sim \Lambda} \hat{f}_m(x, a') $, so that $\HREG_m(\lambda \mid x_t) = 0$.}:
\begin{equation}
V_m(x_t, p)
=
\HREG_m(p \mid x_t)
+
\max_{\lambda \in \Lambda} \rbr{ \tfrac{1}{\gamma_m} {\overline{\coverage}_{\varepsilon_m}(p,\lambda; \Fcal_{x_t})}-\HREG_m(\lambda \mid x_t)}.
\label{eqn:vm-xt}
\end{equation}

The choice of $p_t$ makes the objective bounded by {$\overline{\doec}_{\gamma_m, \varepsilon_m}(\Fcal_{x_t}, \Lambda)$}, and thus each  of the two terms is also bounded by the same number. Since relaxed coverage upper bounds the original coverage, this suggests that $p_t$ achieves an exploration-exploitation tradeoff in a precise sense:

\begin{lemma}
\label{lem:lr_gc}
$p_t \in \co(\Lambda)$ satisfies the following two properties simultaneously:
        \begin{align*}
            \HREG_m(p_t \mid x_t)
            &\leq \overline{\doec}_{\gamma_m, \varepsilon_m}(\Fcal_{x_t}, \Lambda)
            &\quad \text{Low Regret (LR)}
            \\
            \coverage_{\varepsilon_m}(p_t, \lambda; \Fcal_{x_t})
            &\leq \gamma_m \overline{\doec}_{\gamma_m, \varepsilon_m}(\Fcal_{x_t}, \Lambda) + \gamma_m \HREG_m(\lambda \mid x_t), \forall \lambda \in \Lambda,
            &\text{Good Coverage (GC)}
        \end{align*}
\end{lemma}
Specifically, LR asserts that $p_t$ nearly maximizes the reward estimate $\hat{f}_m(x_t, \cdot)$ among all distributions in $\Lambda$.
Meanwhile, GC states that $p_t$ ``covers'' all distributions $\lambda$ in the benchmark class $\Lambda$ well in a heterogeneous manner, with ``greedier'' distributions covered better.
~\citet{simchi2022bypassing} and~\citet[][Section 4]{xu2020upper} were the first to observe the existence of such distributions for discrete action spaces and action spaces with linear reward structure,
respectively, given access to a reward regressor, which enables subsequent regret analysis similar to analysis of policy-search-based algorithms~\citep{agarwal14taming}.
\cite{sahaefficient} recently made a similar observation in discrete-action and linear-reward contextual dueling bandits.
Here we generalize their findings to general reward function classes and benchmark distribution classes.
We now present the regret guarantee for \generalfalcon:

\begin{restatable}{theorem}{mainregret}
    \label{thm:main-regret}
    Suppose $\tau_m \geq 2\tau_{m-1}$ for all $m$. Let $\delta_m = \frac{\delta}{m(m+1)}$.
    For any $\delta \in (0, 1)$, $\Lambda \subseteq \Delta(\Acal)$, with probability at least $1 - \delta$, the $\Lambda$-Regret of \generalfalcon is bounded as: $\EE \sbr{\Regret_\Lambda(T,\generalfalcon)} \leq$
    \begin{align*}
        \upboundlog{
        \tau_1 + \max_{m \in \cbr{2,\ldots,M}} \frac{\tau_{m}}{\gamma_m}
        \cdot
        \rbr{ \twolines{ \max_{n \in [M]}
        \gamma_n \EE_{x \sim \Dcal_X}\sbr{\,\overline{\doec}_{\gamma_n, \varepsilon_n}(\Fcal_x, \Lambda)}
        }
        {
        +
        \max_{n \in \cbr{2,\ldots,M}}
        \gamma_{n}^2 \del{
        \offlinereg(\Fcal, \tau_{n-1}/2, \delta_n)
        + \varepsilon_{n-1}
        }}}
        },
    \end{align*}
    where we use the convention that $\gamma_1 = 0$ and $\gamma_1 \overline{\doec}_{\gamma_1,\varepsilon_1}(\Fcal_x, \Lambda)$ is $\lim_{\gamma \to 0} \gamma \overline{\doec}_{\gamma,\varepsilon_1}(\Fcal_x, \Lambda) = \overline{\Vcal}_{\varepsilon_1}^*(\Fcal_x, \Lambda)$, the value of the relaxed F-design problem solved in the first epoch.
\end{restatable}

The regret bound can be understood as follows: in the first epoch, \generalfalcon does pure exploration, and thus incurs a regret of $\tau_1$ at most; in the subsequent epochs $m \geq 2$, at every time step,
\generalfalcon has an instantaneous regret of
\[
\upboundlog{
\frac{1}{\gamma_m}
        \rbr{ \max_{n \in [m]}\gamma_n \EE\sbr{\,\overline{\doec}_{\gamma_n, \varepsilon_n}(\Fcal_x, \Lambda)}}
        +
        \max_{n \in \cbr{2,\ldots,m}} \gamma_{n}^2 \del{
        \offlinereg(\Fcal, \tau_{n-1} - \tau_{n-2}, \delta_n)
        + \varepsilon_{n-1}
        }
}
\]
which depends on the average context-wise complexity of exploration  $\EE\sbr{\,\overline{\doec}_{\gamma_n, \varepsilon_n}(\Fcal_x, \Lambda)}$ and historical reward function estimates' qualities $\offlinereg(\Fcal, \tau_{n-1}/2, \delta)$.
The $\max_{n \in [m]}$ and $\max_{n \in \cbr{2,\ldots,m}}$ operators come from the interdependence of the reward estimator and data collection in historical epochs: $\hat{f}_n$ is trained from data collected in epoch $n-1$, which in turn is determined by $\hat{f}_{n-1}$, etc.~\footnote{For simplicity of presentation, we further relax $\max_{n \in [m]}$ to $\max_{n \in [M]}$ and $\max_{n \in \cbr{2,\ldots,m}}$ to $\max_{n \in \cbr{2,\ldots, M}}$
when presenting the regret bound in Theorem~\ref{thm:main-regret}.}

\begin{remark}[\generalfalcon with inexact minimizers]
Theorem~\ref{thm:main-regret} continues to hold when the relaxed exploitative F-design problem is solved inexactly: as long as there exists some function $\overline{V}$, for all $t$, $V_m(x_t, p_t) \leq \overline{V}_{\gamma_m,\varepsilon_m}(\Fcal_{x_t}, \Lambda)$, its regret guarantee will still hold with $\overline{\doec}$ replaced by $\overline{V}$.
The proof of Theorem~\ref{thm:main-regret} carries over verbatim by replacing $\overline{\doec}$ with $\overline{V}$.
We will use this fact in Section~\ref{sec:doec-dec}.
\end{remark}

We now discuss the implications of this theorem in our three running examples. For simplicity, we assume a finite reward function class $\Fcal$ and the offline regression oracle performs ERM; in this case, $\offlinereg(\Fcal, \tau_{m-1}/2,\delta) = \iupboundlog{\frac{\log|\Fcal|}{\tau_{m-1}}}$.
In addition, all three running examples satisfy that $\max_x \overline{\doec}_{\gamma,\varepsilon}(\Fcal_x, \Lambda) = O \idel{\frac{D}{\gamma}}$, using the relaxed coverages introduced at the beginning of this section, for $D = |\Acal|,\kappa^2 d, 1/h$ respectively (see Lemma~\ref{lem:relaxed-doec-bound}). Thus:

\begin{itemize}
\item Setting $\tau_m = 2^m$, $\gamma_m = \isqrt{D/\offlinereg(\Fcal, \tau_{m-1}/2, \delta_m)}$, $\varepsilon_m = 1/T$ gives a regret bound of $\iupboundlog{ \sqrt{ D T \log|\Fcal| } }$ (\Cref{corol:regret-doubling-schedule}).
\item If the total horizon $T$ is known, we use a \emph{small-epoch schedule} (\Cref{def:small-epoch}) with $M = \log\log(T)$, $\tau_m = \floor{2T^{1 - 2^{-m}}}$,
$\gamma_m = \isqrt{D /\offlinereg(\Fcal, \tau_{m-1} - \tau_{m-2}, \delta_m)}$, $\varepsilon_m = 1/T$ to achieve the same regret guarantee (\Cref{corol:regret-small-epoch-schedule}).
\end{itemize}
See Appendix~\ref{app:regret-schedules} for the detailed calculation of these regret bounds.
Specifically,
instantiating the relaxed coverage earlier in this section in the running examples turns \generalfalcon to a host of offline oracle-efficient contextual bandit algorithms, recovering existing ones and obtaining new ones:
\begin{itemize}
\item In the discrete action space setting, we recover \falcon~\citep{simchi2022bypassing} with regret $\iupboundlog{\sqrt{|\Acal|T \log|\Fcal|}}$.

\item In the per-context linear reward setting, we recover \linearfalcon~\cite[][Section~4]{xu2020upper} with regret $\iupboundlog{\sqrt{d T \log|\Fcal|}}$. For the per-context generalized linear reward setting, we obtain a new algorithm, which we name \glmoetwod. It has a regret bound of $\iupboundlog{\kappa \sqrt{d T \log|\Fcal|}}$, which matches that of UCCB~\citep{xu2020upper}; however, its number of calls to the offline regression oracle is $O(\log T)$, which is much smaller.

\item In the $h$-smoothed regret setting,
we obtain \smoothedoetwod, which has regret $\iupboundlog{\sqrt{T/h \log|\Fcal|}}$; this matches the regret guarantee of
\smoothigw~\citep{zhu2022contextuala}, an online-oracle-efficient algorithm; this regret guarantee is information-theoretically optimal~\cite[][Remark under Theorem 3]{krishnamurthy2020contextual}.
This is the first time that an optimal $h$-smoothed regret bound is attained by an offline regression oracle-efficient algorithm. In Appendix~\ref{app:experiments}, we empirically show that \smoothedoetwod has competitive performance with \smoothigw.
\end{itemize}

\paragraph{Proof Idea of \Cref{thm:main-regret}.} Our analysis takes a ``virtual policy'' view of \generalfalcon: at epoch $m$, \generalfalcon implicitly uses a policy $\pi_m$ to do exploration, which is defined based on the estimated reward function $\hat{f}_m$. We then show that the cumulative regret of the historical sequence of policies \emph{on every context} $x$ can be precisely controlled:

\begin{restatable}{lemma}{percontextub}
    \label{lemma:general-falcon-per-context-regret-bound}
    Given access to any offline regression oracle, the $\Lambda$-Regret of the sequence of policies used by \Cref{alg:general-falcon} on every context $x$ is bounded as:
    \begin{align*}
        \sum_{t=1}^T \REG(\pi_{m(t)} \mid x)
        \leq
        \upbound{ \tau_1 + M \cdot
        \max_{m \in \cbr{2,\ldots,M}} \frac{\tau_{m}}{\gamma_m}
        \cdot
        \rbr{
        \twolines{
        \sum_{s=1}^M
        \gamma_s \overline{\doec}_{\gamma_s, \varepsilon_s}(\Fcal_x, \Lambda)
        }{
        +
        \sum_{s=2}^M
        \gamma_{s}^2 \del{ \WL_{s-1}(\hat{f}_s, x) + \varepsilon_{s-1} }
        }}
        },
    \end{align*}
    where $m(t)$ denotes the epoch time step $t$ is in; $\WL_s(f, x) \coloneq \EE_{a \sim \pi_s(\cdot | x)}\sbr{ (f(x,a) - f^*(x,a) )^2}$ is the square loss of $f$ on context $x$ under the policy used in epoch $s$.
\end{restatable}

\Cref{thm:main-regret} follows immediately by taking expectation over $x \sim \Dcal_X$ and using the linearity of expectation.
To show Lemma~\ref{lemma:general-falcon-per-context-regret-bound}, we establish per-context concentration bounds on the empirical regret to the true regret, for any action distribution in $\Lambda$ (Lemma~\ref{lemma:regret-concentration}).
We think our proof technique may be of independent interest, as most previous analyses on offline-regression-oracle efficient contextual bandit algorithms~\citep{simchi2022bypassing} establish low-regret and good-coverage guarantees at the policy and population level (averaged across all $x$'s).
We speculate that this lemma may have practical usages, e.g., for guaranteeing treatment quality for every patient or recommendation quality for every customer. We defer the full proof to Appendix~\ref{app:regret-analysis-offline}.

\paragraph{Extensions: Misspecification, Corruption, and Distribution Shifts} We show that our \generalfalcon algorithm can also achieve provable regret guarantees in several variants of the original realizable and iid context setting:

\begin{enumerate}
\item  In the model misspecification setting~\citep{lattimore2020learning} with a uniform misspecification level $\inf_{f \in \Fcal} \sup_{x \in \Xcal,\, a \in \Acal} (f(x,a) - f^*(x,a))^2 =: B$ (see \Cref{assum:misspecification-induced-by-policies}), with appropriate setting of its parameters, the regret bound of \Cref{alg:general-falcon} generalizes the state-of-the-art discrete action space result of~\citet{krishnamurthy2021adapting} to per-context linear model and smoothed regret settings (\Cref{corol:miss-total-regret-bound}).

\item  In the corruption setting~\citep[e.g.,][]{kapoor2019corruption}, an adversary is allowed to corrupt the observed reward in at most $C$ rounds.
In this case, \Cref{alg:general-falcon} achieves a regret bound of $\iupboundlog{\isqrt{T D (C + \log\del{\abs{\Fcal}/\delta}}}$ (\Cref{corol:corruption-total-regret-bound}). This matches the results of deploying \citet{foster2020adapting}'s approach to this corruption setting using online regression oracles.

\item  In a mild context distribution-shift setting (see Assumption~\ref{assum:context-distribution-shift}), \Cref{alg:general-falcon} achieves $\iupboundlog{\isqrt{ A^3 T D \log\del{\abs{\Fcal}/\delta}}}$ regret (\Cref{corol:shift-total-regret-bound}), where $A \geq 1$ measures the worst-case density ratio between the context distribution at every time step and some base context distribution. To the best of our knowledge, this is the first result showing offline-oracle efficient algorithms can work beyond the iid context setting.
\end{enumerate}

In all settings above, our per-context analysis framework allows us to carry out the analyses with little extra work; we defer all details to Appendix~\ref{sec:extensions}.

\section{Efficient Algorithms for Finding Exploration Strategies Certifying Low Relaxed DOEC}

\label{sec:structure-doec}

So far, we have shown that a small relaxed DOEC leads to offline-regression-oracle efficient algorithms with low regret, and have given specialized algorithms finding good exploration distributions certifying small DOEC in some examples. This raises a natural question: under what conditions is the relaxed DOEC small, and can we design efficient general-purpose algorithms to compute exploration distributions that certify small DOEC?

We answer these questions affirmatively in this section. Our key finding is that a new complexity measure of the reward function class $\Gcal$ and the benchmark class of distributions $\Lambda$ named \emph{$\varepsilon$-sequential extrapolation coefficient ($\varepsilon$-SEC)} (together with its relaxed versions) can be used to bound the (relaxed) DOEC, and the boundedness of $\varepsilon$-SEC is grounded in useful examples, including our three running examples above. Moreover, our proof is \emph{algorithmic}: we give an efficient coordinate descent-based algorithm (\Cref{alg:extended-f-design}) that finds an exploration distribution certifying the DOEC to be smaller than $\varepsilon$-SEC. The guarantee holds for any relaxed coverage satisfying a short list of structural properties (Assumption~\ref{assum:admissible-relaxed-coverage}).
We first formally define (relaxed)  $\varepsilon$-SEC:

\begin{definition}[Relaxed $\varepsilon$-SEC]
\label{def:delta-sec}
    For a class of functions $\Gcal$ from $\Acal$ to $[0,1]$, and a family of benchmark distributions $\Lambda$,  and $\varepsilon > 0$, and a relaxed coverage $\overline{\coverage}_{\varepsilon}$ (recall Section~\ref{sec:offline-oracle-efficient-algorithm}; extended to take an unnormalized nonnegative measure as its first argument with cushion parameter $\varepsilon$),
    their \emph{relaxed $\varepsilon$-sequential extrapolation coefficient}, $\overline{\SEC}_\varepsilon(\Gcal, \Lambda)$ is defined as (here $\lambda_{1:i} \coloneq \sum_{j=1}^i \lambda_j$):
    \begin{align}
        \overline{\SEC}_\varepsilon(\Gcal, \Lambda) =
        \sup_{N \in \NN} \sup_{
            \substack{
            \lambda_1, \ldots, \lambda_N \in \Lambda}}
            \sum_{i=1}^N
            \overline{\coverage}_{N\varepsilon}(\lambda_{1:i}, \lambda_i; \Gcal)
            \label{eqn:sec}
    \end{align}
Under the trivial relaxation $\overline{\coverage}_{\varepsilon} = \coverage_{\varepsilon}$, we call the resulting quantity the (unrelaxed) \emph{$\varepsilon$-SEC}, denoted by $\SEC_\varepsilon(\Gcal, \Lambda)$.
\end{definition}

Since $\overline{\coverage}_{N\varepsilon} \geq \coverage_{N\varepsilon}$ pointwise, $\overline{\SEC}_\varepsilon(\Gcal, \Lambda) \geq \SEC_\varepsilon(\Gcal, \Lambda)$ for any relaxed coverage.
Our definition of $\varepsilon$-SEC is inspired by SEC~\citep{xie2022role}, which was used to provide a unified analysis of optimism-based RL algorithms in settings with low coverability or small Bellman Eluder dimension~\citep{jin2021bellman}.
We remark a few important differences:
first, our ``cushion parameter'' $N\varepsilon$ for coverage is proportional to $N$, the number of coverage terms, whereas SEC has a fixed regularization parameter 1;
second, we measure coverage of $\lambda_{1:i}$ on $\lambda_i$, whereas SEC measures the coverage of $\lambda_{1:i-1}$ to $\lambda_i$.
Such ``off-by-1'' adjustment turns out to be important for constructing a useful exploration distribution that certifies a low DOEC, which can provide regret guarantees where optimism-based approaches are not adequate~\citep[][Example 3.1]{foster2022statistical}.

\paragraph{Examples of small relaxed $\varepsilon$-SEC}
We instantiate the relaxed $\varepsilon$-SEC in our three running examples.
We first note that our original definitions of relaxed coverages in Section~\ref{sec:offline-oracle-efficient-algorithm} all have infinite $\varepsilon$-relaxed-SEC -- by taking $\lambda_i \equiv \lambda$ for some $\lambda$, all $\overline{\coverage}_{N \varepsilon}(\lambda_{1:i}, \lambda_i; \Gcal) = \frac{C}{i}$ for $C = |\Acal|, \kappa^2 d, \frac{1}{h}$ respectively, which makes $\sum_{i=1}^N \overline{\coverage}_{N \varepsilon}(\lambda_{1:i}, \lambda_i; \Gcal)$ diverge when $N$ is large. Motivated by this, we consider slightly refined definitions of relaxed coverage which crucially utilize the cushion parameter $\varepsilon$:

\begin{enumerate}
    \item Discrete action space: $\overline{\coverage}_{\varepsilon}(p, \lambda; \Gcal) = \sum_{a \in \Acal} \frac{\lambda(a)}{p(a) + \varepsilon/|\Acal|}$ gives $\overline{\SEC}_{\varepsilon}(\Gcal, \Lambda) \leq |\Acal| \log(1 + 1/\varepsilon)$;
    \item Per-context (generalized) linear reward: with $\varepsilon' := \varepsilon/(\underline{L}^2 B^2)$, $\overline{\coverage}_{\varepsilon}(p, \lambda; \Gcal) = \kappa^2 \tr((\Sigma_p + \varepsilon' I)^{-1} \Sigma_\lambda)$ gives $\overline{\SEC}_{\varepsilon}(\Gcal, \Lambda) \leq \kappa^2 d \log(1 + 1/\varepsilon')$; here,  $\Sigma_q := \EE_{a \sim q} \sbr{\phi(a)\phi(a)^\top}$.
    \item $h$-smoothed regret: $\overline{\coverage}_{\varepsilon}(p, \lambda; \Gcal) = \frac1h \EE_{a \sim \mu}\sbr{\frac{\lambda(a)}{p(a) + \varepsilon}}$ gives $\overline{\SEC}_{\varepsilon}(\Gcal, \Lambda) \leq \frac1h \log(1 + 1/\varepsilon)$.
\end{enumerate}

We verify the validity of these relaxed coverages in Lemma~\ref{lem:relaxed-coverage-admissible}, Appendix~\ref{app:relaxed-sec-examples} and prove their associated SEC bounds in Lemma~\ref{lem:relaxed-sec-running-examples}, Appendix~\ref{app:smoothed-regret}.

\paragraph{Relating SEC to DOEC, and efficiently finding distribution certifying low DOEC} We will show $\overline{\doec}_{\gamma,\varepsilon}(\hat{g}, \Gcal, \Lambda) \lesssim \frac{1}{\gamma}\overline{\SEC}_\varepsilon(\Gcal, \Lambda)$ via a coordinate descent-based algorithm,~\Cref{alg:extended-f-design}, which finds a distribution $p^*$ certifying this.
To motivate its design, we first use our observation before that the relaxed DOEC objective (Definition~\ref{def:doec}) is equivalent to
\[
\widehat{R}(p)
+
\max_{\lambda \in \Lambda} \rbr{ \frac{1}{\gamma}\overline{\coverage}_\varepsilon(p,\lambda; \Gcal) -\widehat{R}(\lambda) }
,
\]
where $\widehat{R}(p) := \max_{\lambda' \in \Lambda}\EE_{a \sim \lambda'}[\hat{g}(a)] - \EE_{a \sim p}[\hat{g}(a)]$. It suffices to find $p$ such that both terms are $\lesssim \frac{1}{\gamma}\overline{\SEC}_\varepsilon(\Gcal, \Lambda)$ simultaneously; following the convention in the preceding section, we name the two requirements (LR) and (GC).

\begin{algorithm}[t]
\setstretch{0.95}
\setlength{\abovedisplayskip}{3pt}
\setlength{\belowdisplayskip}{3pt}
\setlength{\abovedisplayshortskip}{0pt}
\setlength{\belowdisplayshortskip}{3pt}
\LinesNumbered
\SetAlgoLined
\SetAlgoVlined
    \KwIn{benchmark distributions $\Lambda$, reward function class $\Gcal$, estimated reward function $\hat{g} : \Lambda \to [0, 1]$, relaxed coverage $\overline{\coverage}$ satisfying Assumption~\ref{assum:admissible-relaxed-coverage} with step-size threshold $\bar{\Delta}$, step sizes $\Delta_t \in (0, \bar{\Delta}]$, parameters $\gamma > 0$ and $\varepsilon \in (0, 1)$.}
    \KwOut{a distribution $p^* \in \Lambda$ that certifies $\overline{\doec}_{\gamma,\varepsilon}(\Gcal, \Lambda) \leq \frac{10}{\gamma}\overline{\SEC}_\varepsilon(\Gcal, \Lambda)$, i.e.,  Eq.~\eqref{eqn:monster-g} holds.}
    Compute greedy distribution $\widehat{\lambda} = \argmax_{\lambda \in \Lambda} \EE_{a \sim \lambda}[\hat{g}(a)]$\label{line:greedy-dist}, and
    denote $\widehat{R}(p) := \EE_{a \sim \widehat{\lambda}}[\hat{g}(a)] - \EE_{a \sim \lambda}[\hat{g}(a)]$
    \;

    Let $p_0 := 0$.\label{line:init-pt} \algcomment{0}{Initialize with zero distribution}\;
    \For{$t=1,2,\ldots$}{
        \label{line:compute-lambdat} Compute
        $
            \lambda_t = \argmax_{\lambda \in \Lambda} \sbr{ \EE_{a\sim \lambda} [\hat{g}(a)] + \frac{1}{\gamma} \overline{\coverage}_{\varepsilon}(p_{t-1}, \lambda; \Gcal)},
        $
        \;

        \eIf{$-\widehat{R}(\lambda_t) + \frac{1}{\gamma} \overline{\coverage}_{\varepsilon}(p_{t-1}, \lambda_t; \Gcal) > \frac{8}{\gamma} \overline{\SEC}_\varepsilon(\Gcal, \Lambda)$\label{line:check-violation}
        }{
        Run a coordinate descent step:
                $
                p_t \gets p_{t-1} + \Delta_t \lambda_{t},
                $\label{line:update-pt} \algcomment{0}{Update $p_{t-1}$ when not satisfying~Eq.\eqref{eqn:monster-g}}
        }{
            $t_0 \gets t-1$,
            \Return{$p^* = p_{t_0} + (1 - \| p_{t_0} \|_1) \widehat{\lambda}$}\label{line:return-p-star} \algcomment{0}{Terminate and return the final distribution}
        }
    }
    \caption{Finding $p$ that certifies small relaxed \doec} \label{alg:extended-f-design}
\end{algorithm}

Algorithm~\ref{alg:extended-f-design} first computes the greedy distribution $\widehat{\lambda}$ that maximizes the estimated reward function $\hat{g}$ over $\Lambda$ (line~\ref{line:greedy-dist}) for later use in the final output.
It maintains a nonnegative measure $p_t$ over time, with $p_0$ initialized to be the zero measure (line~\ref{line:init-pt}).
Then, at each iteration $t$, it first finds the distribution $\lambda_t$ that maximizes the amount of violation of the (GC) property.
Then it checks whether the violation of the GC property (on $\lambda_t$) exceeds a threshold (line~\ref{line:check-violation}).
If so, it performs a coordinate descent update along the direction of $\lambda_t$  (line~\ref{line:update-pt}).
Otherwise, (GC) is satisfied, so it terminates the iteration and returns the final distribution $p^*$, which is a convex combination of $p_{t_0}$ and the greedy distribution $\widehat{\lambda}$ (line~\ref{line:return-p-star}).

\Cref{alg:extended-f-design} is inspired by previous works in computing exploration distributions over a class of policies~\citep{agarwal14taming}, and linear reward structure setting~\citep[Section 4]{xu2020upper}. Different from these works, we make a few key adjustments:
1) we do not scale the distributions during updates to preserve property (LR) more easily at each iteration;
2) we initialize $p_0$ as the zero measure rather than a normalized probability measure. A key step in proving its guarantee is to show that our iterates $p_t$'s are always sub-probability distributions and can be extended to a valid probability distribution $p^* \in \co(\Lambda)$ by mixing with the greedy distribution.

The theorem below is the main result of this section, in which we show that as long as $\overline{\coverage}_{\varepsilon}$ satisfies a few structural properties we termed \emph{admissibility} (Assumption~\ref{assum:admissible-relaxed-coverage} in Appendix~\ref{sec:coord-desc-converge}), \Cref{alg:extended-f-design} can efficiently find a distribution certifying that $\overline{\doec}_{\gamma,\varepsilon}(\Gcal, \Lambda) \lesssim \frac{1}{\gamma}\overline{\SEC}_\varepsilon(\Gcal, \Lambda)$.
Importantly, the admissibility assumption is grounded in our $\varepsilon$-cushioned relaxed coverages defined for the three running examples, as well as the original unrelaxed coverage. We defer these checks to Lemma~\ref{lem:relaxed-coverage-admissible} in Appendix~\ref{app:relaxed-sec-examples}.

\begin{restatable}{theorem}{doecleqsec}
\label{thm:doec-leq-sec}
For any reward function class $\Gcal: \Acal \to [0,1]$, benchmark distribution class $\Lambda$, relaxed coverage $\overline{\coverage}$ satisfying Assumption~\ref{assum:admissible-relaxed-coverage} with step-size threshold $\bar{\Delta}$, $\gamma > 0$, and $\varepsilon \in (0, 1)$, \Cref{alg:extended-f-design} with step size $\Delta_t = \bar{\Delta}$ terminates within $\lfloor 1/\bar{\Delta} \rfloor$ iterations and outputs a distribution $p^* \in \co(\Lambda)$ such that
\begin{equation}
    \max_{\lambda \in \Lambda }
    \rbr{
    \EE_{a \sim \lambda}\sbr{\hat{g}(a)}
    -
    \EE_{a \sim p^*}\sbr{\hat{g}(a)}
    +
    \frac{1}{\gamma}\overline{\coverage}_{\varepsilon}(p^*, \lambda; \Gcal)
    }
    \leq
    \frac{10}{\gamma}\overline{\SEC}_\varepsilon(\Gcal, \Lambda).
    \label{eqn:monster-g}
\end{equation}
\end{restatable}

Applying Theorem~\ref{thm:doec-leq-sec} with the relaxed coverages with cushion $\varepsilon$ of Lemma~\ref{lem:relaxed-coverage-admissible}, \Cref{alg:extended-f-design} terminates within $\iupboundlog{\frac{|\Acal|}{\varepsilon}}$, $\iupboundlog{\frac{\underline{L}^2 B^2}{\varepsilon}}$, and $\iupboundlog{\frac{1}{h\varepsilon}}$ iterations in the three running examples, and certifies relaxed DOEC bounds of $\iupbound{\frac{|\Acal|}{\gamma}\log\frac{|\Acal|}{\varepsilon}}$, $\iupbound{\frac{\kappa^2 d}{\gamma}\log\frac{\underline{L}^2 B^2}{\varepsilon}}$, and $\iupbound{\frac{1}{\gamma h}\log\frac{1}{h\varepsilon}}$, respectively -- matching the bounds given by the certificates obtained via barrier-regularized optimization (Lemma~\ref{lem:relaxed-doec-bound}) up to a logarithmic factor. The added value of Theorem~\ref{thm:doec-leq-sec} is its generality: it provides a certifying procedure for \emph{any} admissible relaxed coverage, including the original coverage, for which no closed-form or optimization-based certificate is available. This allows us to plug Algorithm~\ref{alg:extended-f-design} into \generalfalcon to obtain new offline-oracle efficient algorithms with end-to-end running time guarantees; we present one such example in the per-context generalized linear reward setting in Appendix~\ref{app:alg2-compute-cost}.

Specialized to the pure-exploration setting $\gamma \to 0$ and the trivial relaxation $\overline{\coverage}_\varepsilon = \coverage_\varepsilon$, Theorem~\ref{thm:doec-leq-sec} implies that $\Vcal_\varepsilon^*(\Gcal, \Lambda) \lesssim \SEC_\varepsilon(\Gcal, \Lambda)$; this generalizes~\cite[][Theorem 4.2]{agarwal2024non} to allow general benchmark distribution class $\Lambda$. Our result is also more general in that we are able to bound $\doec_{\gamma,\varepsilon}(\Gcal, \Lambda)$ with arbitrary $\gamma > 0$, making it a good fit for the regret minimization setting.

The proof of Theorem~\ref{thm:doec-leq-sec} can be found in Appendix~\ref{sec:coord-desc-converge}. Its key idea is to construct a potential function that linearly combines a (negative) sequential extrapolation error term and a regret term, and show that the potential function decreases steadily while respecting a lower bound. The sequential extrapolation error serves as the same role as log barrier or log-determinant barrier regularizers in previous works~\citep{agarwal14taming,xu2020upper}.
In Appendix~\ref{app:large-step-size}, we show additionally that when $\Lambda = \cbr{\delta_a: a \in \Acal}$, we can solve Eq.~\eqref{eqn:monster-g} using \Cref{alg:extended-f-design} with a more aggressive step size and a much lower $\upbound{\SEC_\varepsilon(\Gcal,\Lambda)}$ iterations; this generalizes~\cite[][Section 4]{xu2020upper} to nonlinear reward function classes.

\subsection{Bounding DOEC beyond \texorpdfstring{$\varepsilon$}{delta}-SEC: the importance of active exploration}

Despite Theorem~\ref{thm:doec-leq-sec}'s generality, we observe that the (relaxed) $\varepsilon$-SEC has a worst-case nature: it evaluates the maximum extrapolation error in all sequences.
However, the definition of DOEC allows \emph{active experimental design}: we are free to choose action distribution $p \in \co(\Lambda)$ that balances exploration and exploitation.
This suggests that Theorem~\ref{thm:doec-leq-sec} alone may be insufficient to tightly characterize DOEC, as
we confirm in the following proposition:

\begin{proposition}
\label{prop:doec-small-sec-large}
Let $\Lambda = \cbr{\delta_a: a \in \Acal}$. For any $k \geq 1$,
there exists $\Acal$ and
$\Gcal$ such that
$
\doec_{\gamma,\varepsilon}(\Gcal, \Lambda) \lesssim \isqrt{\frac{k}{\gamma}} + \frac{k}{\gamma}$
whereas
$
\SEC_{\varepsilon}(\Gcal, \Lambda) \geq \min\rbr{ 2^{k-2}, \frac{1}{2\varepsilon}}
$.
\label{prop:sec-doec-loose}
\end{proposition}

We defer the proof and detailed discussions to Appendix~\ref{sec:appendix-doec-sec-loose} -- the proof is based on a ``cheating code''  structured bandit instance~\cite[e.g.,][]{agarwal2024non}.
For that instance, Theorem~\ref{thm:doec-leq-sec} at best can give an upper bound on $\doec_{\gamma,\varepsilon}(\Gcal, \Lambda)$ of $\ilowbound{\min\idel{ \frac{2^k}{\gamma}, \frac{1}{\varepsilon \gamma} }}$, much greater than its actual value. Plugging this bound into Theorem~\ref{thm:main-regret} gives a regret bound at least $\min(T, \isqrt{2^k})$
for \generalfalcon, which is vacuous when $k = \ilowbound{\log T}$.
In contrast, using the tight bound $\doec_{\gamma,\varepsilon}(\Gcal, \Lambda) \lesssim \isqrt{\frac{k}{\gamma}} + \frac{k}{\gamma}$ gives a meaningful regret bound of $\iupboundlog{ (k \ln|\Fcal|)^{1/3} T^{2/3} }$ for \generalfalcon. In summary, this example reveals that using worst-case quantities such as SEC to characterize DOEC may not be adequate, and we leave finer structural characterizations of DOEC as an interesting open question.

\section{Connecting DOEC to Contextual Bandits with Online Regression Oracles}
\label{sec:doec-dec}

We bridge DOEC and DEC (recall~\Cref{def:dec}), a primary statistical complexity measure for contextual bandit learning with online regression oracles. Our main finding is that any exploration distribution that certifies small DOEC must also certify a small DEC:

\begin{restatable}{theorem}{decleqdoec}
    \label{thm:dec-leq-doec}
    For any function class $\Gcal: \Xcal \times \Acal \to [0, 1]$, any set of action distributions $\Lambda \subseteq \Delta(\Acal)$, any $\gamma > 0$ and any $\varepsilon \geq 0$, we have
    that any distribution $p$ that certifies
    $\doec_\gamma(\Gcal, \Lambda) \leq V$ also certifies that
    $\dec_\gamma(\Gcal, \Lambda) \leq V + \frac{1}{\gamma} + \gamma \varepsilon.$
    As a consequence,
        $
            \dec_\gamma(\Gcal, \Lambda)
            \leq
            \doec_{\gamma, \varepsilon}(\Gcal, \Lambda)
            + \frac{1}{\gamma} + \gamma \varepsilon.
        $
\end{restatable}

Combining Theorem~\ref{thm:dec-leq-doec} with Theorem~\ref{thm:doec-leq-sec}, we obtain that
$\dec_\gamma(\Gcal, \Lambda)
  \lesssim \inf_{\varepsilon>0}\rbr{
  \frac{\SEC_{\varepsilon}(\Gcal, \Lambda)}{\gamma}
  + \frac{1}{\gamma} + \gamma \varepsilon}$. A similar result was obtained in~\cite[][Theorem 6.1]{foster2021statistical} using a nonconstructive proof, via minimax theorem and bounding the dual DEC. Our proof is constructive, as we use Algorithm~\ref{alg:extended-f-design} to construct a distribution to certify this upper bound, which may be of independent interest.

Furthermore, when $\SEC_{\varepsilon}(\Gcal, \Lambda) \leq D \cdot \polylog(\frac{1}{\varepsilon})$ with $D \geq 1$, Theorem~\ref{thm:dec-leq-doec} implies that $\dec_{\gamma}(\Gcal, \Lambda) \leq \frac{D}{\gamma} \polylog(\frac{1}{\gamma})$, which enables $\iupboundlog{\sqrt{D T \ln|\Fcal|}}$ regret for \etwod~\citep{foster21thestatistical}. In Appendix~\ref{sec:squarecbf}, we present \squarecbf, an online oracle-efficient contextual bandit algorithm in light of this observation.
\squarecbf employs Algorithm~\ref{alg:extended-f-design} as a subroutine, and we speculate that it may sometimes enjoy better computational cost than the original \etwod algorithm.

\section{Conclusion}
We design a unified algorithm, \generalfalcon for contextual bandits with offline regression oracles. Our key algorithmic innovation is to use an exploitative F-design that generalizes prior works~\citep{simchi2022bypassing,xu2020upper} to achieve an exploration-exploitation tradeoff.
Central to \generalfalcon's performance guarantee is a new statistical complexity measure, Decision-Offline Estimation Coefficient (DOEC), which we show to be small in many examples. Our results not only recover existing works but also provide new guarantees, while at the same time demonstrating robustness to environmental changes. For future work, we are interested in investigating: Does the lower bound of DOEC also imply information-theoretic lower bounds for online contextual bandit learning? Do generalizations of DOEC exist that enable other guarantees, e.g., efficient first-order contextual bandits with offline regression oracles~\citep{foster2021efficient}?
Are there finer structural characterizations of DOEC, e.g., in terms of the value function star number~\citep{foster2020instance}?
Can we extend the \generalfalcon principle to other applications such as partial monitoring, reinforcement learning, and RLHF~\citep{wang2023rlhf,li2025provably}?

\paragraph{Acknowledgments.}
We thank the anonymous COLT reviewers for their constructive feedback.
We thank Advait Khopade for help with experiments in a companion project. CZ would like to thank Kyoungseok Jang and Kwang-Sung Jun for helpful discussions of the proof of~\cite{simchi2022bypassing}, and the coauthors of~\cite{krishnamurthy2020contextual} and Alekh Agarwal for helpful discussions of the optimization and statistical analysis of Mini-Monster~\citep{agarwal14taming}.
We also thank Chen-Yu Wei for bringing to our attention the per-context FTRL analysis of~\cite{neu2020efficient} for adversarial linear contextual bandits.
We would also like to thank Haipeng Luo for making his lecture notes on Mini-Monster available~\citep{luo2017minimonster}. This work is supported by National Science Foundation grant  IIS-2440266 (CAREER).

\bibliography{9.references}

@article{van2026improved,
  title={An Improved Algorithm for Adversarial Linear Contextual Bandits via Reduction},
  author={van Erven, Tim and Mayo, Jack and Olkhovskaya, Julia and Wei, Chen-Yu},
  journal={Advances in Neural Information Processing Systems},
  volume={38},
  pages={135838--135863},
  year={2026}
}

@inproceedings{neu2020efficient,
  title={Efficient and robust algorithms for adversarial linear contextual bandits},
  author={Neu, Gergely and Olkhovskaya, Julia},
  booktitle={Conference on Learning Theory},
  pages={3049--3068},
  year={2020},
  organization={PMLR}
}

@misc{luo2017minimonster,
  title={CSCI 699 Fall 2017 Lecture 21},
  author={Haipeng Luo},
  howpublished={Lecture notes, CSCI 699, University of Southern California},
  url={https://haipeng-luo.net/courses/CSCI699/lecture21.pdf},
  year={2017}
}

@article{yang1999information,
  title={Information-theoretic determination of minimax rates of convergence},
  author={Yang, Yuhong and Barron, Andrew},
  journal={Annals of Statistics},
  pages={1564--1599},
  year={1999},
  publisher={JSTOR}
}

@article{foster2022statistical,
  title={Statistical reinforcement learning and decision making: Course notes},
  author={Foster, Alexander Rakhlin Dylan J and Rakhlin, A},
  journal={MIT Lecture notes for course},
  volume={9},
  pages={S915},
  year={2022}
}

@inproceedings{blum1999beating,
  title={Beating the hold-out: Bounds for k-fold and progressive cross-validation},
  author={Blum, Avrim and Kalai, Adam and Langford, John},
  booktitle={Proceedings of the twelfth annual conference on Computational learning theory},
  pages={203--208},
  year={1999}
}

@inproceedings{abe1999associative,
	title        = {Associative reinforcement learning using linear probabilistic concepts},
	author       = {Abe, Naoki and Long, Philip M},
	year         = {1999},
	booktitle    = {ICML},
	pages        = {3--11}
}

@article{auer03nonstochastic,
	title        = {{The Nonstochastic Multiarmed Bandit Problem}},
	author       = {Auer, Peter and Cesa-Bianchi, Nicol{\`{o}} and Freund, Yoav and Schapire, Robert E},
	year         = {2003},
	month        = {jan},
	journal      = {SIAM J. Comput.},
	publisher    = {Society for Industrial and Applied Mathematics},
	address      = {Philadelphia, PA, USA},
	volume       = {32},
	number       = {1},
	pages        = {48--77},
	doi          = {10.1137/S0097539701398375},
	issn         = {0097-5397}
}

@book{cesa-bianchi06prediction,
	title        = {{Prediction, Learning, and Games}},
	author       = {Cesa-Bianchi, Nicolo and Lugosi, Gabor},
	year         = {2006},
	publisher    = {Cambridge University Press},
	isbn         = {0521841089}
}

@article{langford2007epoch,
	title        = {The epoch-greedy algorithm for contextual multi-armed bandits},
	author       = {Langford, John and Zhang, Tong},
	year         = {2007},
	journal      = {Advances in neural information processing systems},
	publisher    = {Curran Associates Red Hook, NY},
	volume       = {20},
	number       = {1},
	pages        = {96--1}
}

@inproceedings{dani08stochastic,
	title        = {{Stochastic Linear Optimization under Bandit Feedback.}},
	author       = {Dani, Varsha and Hayes, Thomas P and Kakade, Sham M},
	year         = {2008},
	booktitle    = {Proceedings of the Conference on Learning Theory (COLT)},
	pages        = {355--366}
}

@inproceedings{filippi10parametric,
	title        = {{Parametric Bandits: The Generalized Linear Case}},
	author       = {Filippi, Sarah and Cappe, Olivier and Garivier, Aur{\'{e}}lien and Szepesv{\'{a}}ri, Csaba},
	year         = {2010},
	booktitle    = {Advances in Neural Information Processing Systems (NeurIPS)},
	pages        = {586--594}
}

@inproceedings{li2010contextual,
	title        = {A contextual-bandit approach to personalized news article recommendation},
	author       = {Li, Lihong and Chu, Wei and Langford, John and Schapire, Robert E},
	year         = {2010},
	booktitle    = {Proceedings of the 19th international conference on World wide web},
	pages        = {661--670}
}

@article{dudik2011efficient,
	title        = {Efficient optimal learning for contextual bandits},
	author       = {Dudik, Miroslav and Hsu, Daniel and Kale, Satyen and Karampatziakis, Nikos and Langford, John and Reyzin, Lev and Zhang, Tong},
	year         = {2011},
	journal      = {arXiv preprint arXiv:1106.2369}
}

@inproceedings{agarwal12contextual,
	title        = {{Contextual bandit learning with predictable rewards}},
	author       = {Agarwal, Alekh and Dud$\backslash$'$\backslash$ik, Miroslav and Kale, Satyen and Langford, John and Schapire, Robert},
	year         = {2012},
	booktitle    = {Artificial Intelligence and Statistics},
	pages        = {19--26}
}

@inproceedings{russo13eluder,
	title        = {{Eluder dimension and the sample complexity of optimistic exploration}},
	author       = {Russo, Daniel and {Van Roy}, Benjamin},
	year         = {2013},
	booktitle    = {Advances in Neural Information Processing Systems (NeurIPS)},
	pages        = {2256--2264}
}

@inproceedings{agarwal14taming,
	title        = {{Taming the monster: A fast and simple algorithm for contextual bandits}},
	author       = {Agarwal, Alekh and Hsu, Daniel and Kale, Satyen and Langford, John and Li, Lihong and Schapire, Robert},
	year         = {2014},
	booktitle    = {Proceedings of the International Conference on Machine Learning (ICML)},
	pages        = {1638--1646}
}

@article{krishnamurthy2016contextual,
	title        = {Contextual semibandits via supervised learning oracles},
	author       = {Krishnamurthy, Akshay and Agarwal, Alekh and Dudik, Miro},
	year         = {2016},
	journal      = {Advances In Neural Information Processing Systems},
	volume       = {29}
}

@inproceedings{agarwal17corralling,
	title        = {{Corralling a Band of Bandit Algorithms}},
	author       = {Agarwal, Alekh and Luo, Haipeng and Neyshabur, Behnam and Schapire, Robert E},
	year         = {2017},
	booktitle    = {Proceedings of the Conference on Learning Theory (COLT)},
	volume       = {65},
	pages        = {12--38}
}

@article{rakhlin2017empirical,
  title={Empirical entropy, minimax regret and minimax risk},
  author={Rakhlin, A and Sridharan, K and Tsybakov, AB},
  journal={Bernoulli: a journal of mathematical statistics and probability},
  volume={23},
  number={2},
  pages={789--824},
  year={2017},
  publisher={International Statistical Institute}
}

@inproceedings{li2017provably,
	title        = {{Provably Optimal Algorithms for Generalized Linear Contextual Bandits}},
	author       = {Li, Lihong and Lu, Yu and Zhou, Dengyong},
	year         = {2017},
	booktitle    = {Proceedings of the International Conference on Machine Learning (ICML)},
	volume       = {70},
	pages        = {2071--2080}
}

@article{bietti18contextual,
	title        = {{A Contextual Bandit Bake-off}},
	author       = {Bietti, Alberto and Agarwal, Alekh and Langford, John},
	year         = {2018},
	journal      = {arXiv preprint arXiv:1802.04064},
	archiveprefix = {arXiv},
	arxivid      = {1802.04064},
	eprint       = {1802.04064}
}

@book{lattimore18bandit, 
    place={Cambridge}, 
    title={Bandit Algorithms}, 
    publisher={Cambridge University Press}, 
    author={Lattimore, Tor and Szepesvári, Csaba}, 
    year={2020}
}

@preamble{ " \newcommand{\noop}[1]{} " }

@article{foster2020adapting,
	title        = {Adapting to misspecification in contextual bandits},
	author       = {Foster, Dylan J and Gentile, Claudio and Mohri, Mehryar and Zimmert, Julian},
	year         = {2020},
	journal      = {Advances in Neural Information Processing Systems},
	volume       = {33},
	pages        = {11478--11489}
}

@article{foster2020instance,
	title        = {Instance-dependent complexity of contextual bandits and reinforcement learning: A disagreement-based perspective},
	author       = {Foster, Dylan J and Rakhlin, Alexander and Simchi-Levi, David and Xu, Yunzong},
	year         = {2020},
	journal      = {arXiv preprint arXiv:2010.03104}
}

@article{foster20beyond,
	title        = {{Beyond UCB: Optimal and efficient contextual bandits with regression oracles}},
	author       = {Foster, Dylan J and Rakhlin, Alexander},
	year         = {2020},
	journal      = {Proceedings of the International Conference on Machine Learning (ICML)}
}

@inproceedings{amin2011bandits,
  title={Bandits, query learning, and the haystack dimension},
  author={Amin, Kareem and Kearns, Michael and Syed, Umar},
  booktitle={Proceedings of the 24th Annual Conference on Learning Theory},
  pages={87--106},
  year={2011},
  organization={JMLR Workshop and Conference Proceedings}
}

@article{jun20crush,
	title        = {{Crush Optimism with Pessimism: Structured Bandits Beyond Asymptotic Optimality}},
	author       = {Jun, Kwang-Sung and Zhang, Chicheng},
	year         = {2020},
	journal      = {ICML Workshop on Theoretical Foundations of Reinforcement Learning (arXiv:2006.08754)}
}

@article{krishnamurthy2020contextual,
	title        = {Contextual bandits with continuous actions: Smoothing, zooming, and adapting},
	author       = {Krishnamurthy, Akshay and Langford, John and Slivkins, Aleksandrs and Zhang, Chicheng},
	year         = {2020},
	journal      = {Journal of Machine Learning Research},
	volume       = {21},
	number       = {137},
	pages        = {1--45}
}

@article{majzoubi2020efficient,
	title        = {Efficient contextual bandits with continuous actions},
	author       = {Majzoubi, Maryam and Zhang, Chicheng and Chari, Rajan and Krishnamurthy, Akshay and Langford, John and Slivkins, Aleksandrs},
	year         = {2020},
	journal      = {Advances in Neural Information Processing Systems},
	volume       = {33},
	pages        = {349--360}
}

@inproceedings{tirinzoni2020novel,
	title        = {{A Novel Confidence-Based Algorithm for Structured Bandits}},
	author       = {Tirinzoni, Andrea and Lazaric, Alessandro and Restelli, Marcello},
	year         = {2020},
	booktitle    = {Proceedings of the International Conference on Artificial Intelligence and Statistics (AISTATS)}
}

@article{xu2020upper,
	title        = {Upper counterfactual confidence bounds: a new optimism principle for contextual bandits},
	author       = {Xu, Yunbei and Zeevi, Assaf},
	year         = {2020},
	journal      = {arXiv preprint arXiv:2007.07876}
}

@article{foster2021efficient,
	title        = {Efficient first-order contextual bandits: Prediction, allocation, and triangular discrimination},
	author       = {Foster, Dylan J and Krishnamurthy, Akshay},
	year         = {2021},
	journal      = {Advances in Neural Information Processing Systems},
	volume       = {34},
	pages        = {18907--18919}
}

@article{foster2021statistical,
	title        = {The statistical complexity of interactive decision making},
	author       = {Foster, Dylan J and Kakade, Sham M and Qian, Jian and Rakhlin, Alexander},
	year         = {2021},
	journal      = {arXiv preprint arXiv:2112.13487}
}

@article{foster21thestatistical,
	title        = {{The Statistical Complexity of Interactive Decision Making}},
	author       = {Foster, Dylan J and Kakade, Sham M and Qian, Jian and Rakhlin, Alexander},
	year         = {2021},
	journal      = {CoRR},
	volume       = {abs/2112.1}
}

@article{ouyang2022training,
	title        = {Training language models to follow instructions with human feedback},
	author       = {Ouyang, Long and Wu, Jeffrey and Jiang, Xu and Almeida, Diogo and Wainwright, Carroll and Mishkin, Pamela and Zhang, Chong and Agarwal, Sandhini and Slama, Katarina and Ray, Alex and others},
	year         = {2022},
	journal      = {Advances in Neural Information Processing Systems},
	volume       = {35},
	pages        = {27730--27744}
}

@article{simchi2022bypassing,
	title        = {Bypassing the monster: A faster and simpler optimal algorithm for contextual bandits under realizability},
	author       = {Simchi-Levi, David and Xu, Yunzong},
	year         = {2022},
	journal      = {Mathematics of Operations Research},
	publisher    = {INFORMS},
	volume       = {47},
	number       = {3},
	pages        = {1904--1931}
}

@article{zhang2022feel,
	title        = {Feel-good thompson sampling for contextual bandits and reinforcement learning},
	author       = {Zhang, Tong},
	year         = {2022},
	journal      = {SIAM Journal on Mathematics of Data Science},
	publisher    = {SIAM},
	volume       = {4},
	number       = {2},
	pages        = {834--857}
}

@inproceedings{zhu2022contextuala,
	title        = {Contextual bandits with smooth regret: Efficient learning in continuous action spaces},
	author       = {Zhu, Yinglun and Mineiro, Paul},
	year         = {2022},
	booktitle    = {International Conference on Machine Learning},
	pages        = {27574--27590},
	organization = {PMLR}
}

@inproceedings{zhu2022contextualb,
	title        = {Contextual bandits with large action spaces: Made practical},
	author       = {Zhu, Yinglun and Foster, Dylan J and Langford, John and Mineiro, Paul},
	year         = {2022},
	booktitle    = {International Conference on Machine Learning},
	pages        = {27428--27453},
	organization = {PMLR}
}

@article{wang2023rlhf,
	title        = {Is RLHF More Difficult than Standard RL?},
	author       = {Wang, Yuanhao and Liu, Qinghua and Jin, Chi},
	year         = {2023},
	journal      = {arXiv preprint arXiv:2306.14111}
}

@inproceedings{krishnamurthy2021adapting,
  title={Adapting to misspecification in contextual bandits with offline regression oracles},
  author={Krishnamurthy, Sanath Kumar and Hadad, Vitor and Athey, Susan},
  booktitle={International Conference on Machine Learning},
  pages={5805--5814},
  year={2021},
  organization={PMLR}
}

@incollection{tewari2017ads,
  title={From ads to interventions: Contextual bandits in mobile health},
  author={Tewari, Ambuj and Murphy, Susan A},
  booktitle={Mobile health: sensors, analytic methods, and applications},
  pages={495--517},
  year={2017},
  publisher={Springer}
}

@article{paszke2019pytorch,
  title={Pytorch: An imperative style, high-performance deep learning library},
  author={Paszke, Adam and Gross, Sam and Massa, Francisco and Lerer, Adam and Bradbury, James and Chanan, Gregory and Killeen, Trevor and Lin, Zeming and Gimelshein, Natalia and Antiga, Luca and others},
  journal={Advances in neural information processing systems},
  volume={32},
  year={2019}
}

@article{varoquaux2015scikit,
  title={Scikit-learn: Machine learning without learning the machinery},
  author={Varoquaux, Ga{\"e}l and Buitinck, Lars and Louppe, Gilles and Grisel, Olivier and Pedregosa, Fabian and Mueller, Andreas},
  journal={GetMobile: Mobile Computing and Communications},
  volume={19},
  number={1},
  pages={29--33},
  year={2015},
  publisher={ACM New York, NY, USA}
}

@inproceedings{beygelzimer2005error,
  title={Error limiting reductions between classification tasks},
  author={Beygelzimer, Alina and Dani, Varsha and Hayes, Tom and Langford, John and Zadrozny, Bianca},
  booktitle={Proceedings of the 22nd international conference on Machine learning},
  pages={49--56},
  year={2005}
}

@article{foster2024online,
  title={Online estimation via offline estimation: An information-theoretic framework},
  author={Foster, Dylan J and Han, Yanjun and Qian, Jian and Rakhlin, Alexander},
  journal={Advances in Neural Information Processing Systems},
  volume={37},
  pages={42840--42898},
  year={2024}
}

@article{guruswami2009hardness,
  title={Hardness of learning halfspaces with noise},
  author={Guruswami, Venkatesan and Raghavendra, Prasad},
  journal={SIAM Journal on Computing},
  volume={39},
  number={2},
  pages={742--765},
  year={2009},
  publisher={SIAM}
}

@inproceedings{feldman2006new,
  title={New results for learning noisy parities and halfspaces},
  author={Feldman, Vitaly and Gopalan, Parikshit and Khot, Subhash and Ponnuswami, Ashok Kumar},
  booktitle={2006 47th Annual IEEE Symposium on Foundations of Computer Science (FOCS'06)},
  pages={563--574},
  year={2006},
  organization={IEEE}
}

@inproceedings{levy2023efficient,
  title={Efficient rate optimal regret for adversarial contextual mdps using online function approximation},
  author={Levy, Orin and Cohen, Alon and Cassel, Asaf and Mansour, Yishay},
  booktitle={International Conference on Machine Learning},
  pages={19287--19314},
  year={2023},
  organization={PMLR}
}

@article{jin2021bellman,
  title={Bellman eluder dimension: New rich classes of rl problems, and sample-efficient algorithms},
  author={Jin, Chi and Liu, Qinghua and Miryoosefi, Sobhan},
  journal={Advances in neural information processing systems},
  volume={34},
  pages={13406--13418},
  year={2021}
}

@article{li2022understanding,
  title={Understanding the eluder dimension},
  author={Li, Gene and Kamath, Pritish and Foster, Dylan J and Srebro, Nati},
  journal={Advances in Neural Information Processing Systems},
  volume={35},
  pages={23737--23750},
  year={2022}
}

@inproceedings{levy2025regret,
  title={Regret Bounds for Adversarial Contextual Bandits with General Function Approximation and Delayed Feedback},
  author={Levy, Orin and Erez, Liad and Cohen, Alon and Mansour, Yishay},
  booktitle={The Thirty-ninth Annual Conference on Neural Information Processing Systems},
  year={2025}
}

@article{li2025provably,
  title={Provably efficient RLHF pipeline: A unified view from contextual bandits},
  author={Li, Long-Fei and Qian, Yu-Yang and Zhao, Peng and Zhou, Zhi-Hua},
  journal={ArXiv preprint},
  volume={2502},
  year={2025}
}

@inproceedings{agarwal2024non,
  title={The Non-linear $ F $-Design and Applications to Interactive Learning},
  author={Agarwal, Alekh and Qian, Jian and Rakhlin, Alexander and Zhang, Tong},
  booktitle={Forty-first International Conference on Machine Learning},
  year={2024}
}

@article{demirer2019semi,
  title={Semi-parametric efficient policy learning with continuous actions},
  author={Demirer, Mert and Syrgkanis, Vasilis and Lewis, Greg and Chernozhukov, Victor},
  journal={arXiv preprint arXiv:1905.10116},
  year={2019}
}

@article{song2022hybrid,
  title={Hybrid rl: Using both offline and online data can make rl efficient},
  author={Song, Yuda and Zhou, Yifei and Sekhari, Ayush and Bagnell, J Andrew and Krishnamurthy, Akshay and Sun, Wen},
  journal={arXiv preprint arXiv:2210.06718},
  year={2022}
}

@article{xie2022role,
  title={The role of coverage in online reinforcement learning},
  author={Xie, Tengyang and Foster, Dylan J and Bai, Yu and Jiang, Nan and Kakade, Sham M},
  journal={arXiv preprint arXiv:2210.04157},
  year={2022}
}

@article{zhu2022efficient,
  title={Efficient active learning with abstention},
  author={Zhu, Yinglun and Nowak, Robert},
  journal={Advances in Neural Information Processing Systems},
  volume={35},
  pages={35379--35391},
  year={2022}
}

@inproceedings{chu2011contextual,
  title={Contextual bandits with linear payoff functions},
  author={Chu, Wei and Li, Lihong and Reyzin, Lev and Schapire, Robert},
  booktitle={Proceedings of the fourteenth international conference on artificial intelligence and statistics},
  pages={208--214},
  year={2011},
  organization={JMLR Workshop and Conference Proceedings}
}

@article{kiefer1960equivalence,
  title={The equivalence of two extremum problems},
  author={Kiefer, Jack and Wolfowitz, Jacob},
  journal={Canadian Journal of Mathematics},
  volume={12},
  pages={363--366},
  year={1960},
  publisher={Cambridge University Press}
}

@inproceedings{sahaefficient,
  title={Efficient and Near-Optimal Algorithm for Contextual Dueling Bandits with Offline Regression Oracles},
  author={Saha, Aadirupa and Schapire, Robert E},
  booktitle={The Thirty-ninth Annual Conference on Neural Information Processing Systems},
  year={2025}
}

@article{abbasi2011improved,
  title={Improved algorithms for linear stochastic bandits},
  author={Abbasi-Yadkori, Yasin and P{\'a}l, D{\'a}vid and Szepesv{\'a}ri, Csaba},
  journal={Advances in neural information processing systems},
  volume={24},
  year={2011}
}

@article{kapoor2019corruption,
  title={Corruption-tolerant bandit learning},
  author={Kapoor, Sayash and Patel, Kumar Kshitij and Kar, Purushottam},
  journal={Machine Learning},
  volume={108},
  number={4},
  pages={687--715},
  year={2019},
  publisher={Springer}
}

@article{brukhim2025hardness,
  title={On the Hardness of Bandit Learning},
  author={Brukhim, Nataly and Pacchiano, Aldo and Dudik, Miroslav and Schapire, Robert},
  journal={arXiv preprint arXiv:2506.14746},
  year={2025}
}

@article{lattimore2014bounded,
  title={Bounded regret for finite-armed structured bandits},
  author={Lattimore, Tor and Munos, R{\'e}mi},
  journal={Advances in neural information processing systems},
  volume={27},
  year={2014}
}

@article{jun2020crush,
  title={Crush optimism with pessimism: Structured bandits beyond asymptotic optimality},
  author={Jun, Kwang-Sung and Zhang, Chicheng},
  journal={Advances in Neural Information Processing Systems},
  volume={33},
  pages={6366--6376},
  year={2020}
}

@inproceedings{lattimore2020learning,
  title={Learning with good feature representations in bandits and in rl with a generative model},
  author={Lattimore, Tor and Szepesvari, Csaba and Weisz, Gellert},
  booktitle={International conference on machine learning},
  pages={5662--5670},
  year={2020},
  organization={PMLR}
}

\newpage
\appendix
\tableofcontents

\section{Additional Related Work}
\label{sec:related-work}

\paragraph{Contextual Bandits}
Contextual bandits is one of the most fundamental models for sequential decision making with side information~\citep{langford2007epoch,agarwal14taming,li2010contextual}.
Our work focuses on the stochastic setting, where the expected cost of taking action $a$ given context $x$ is governed by some fixed function $f^*$.
The adversarial contextual bandit problem has also been extensively studied~\cite[e.g.,][]{auer03nonstochastic,neu2020efficient,van2026improved}; our analysis has some similarities with~\cite{neu2020efficient}, in that we also first establish a per-context regret bound and conclude the final regret bound via averaging.
For stochastic contextual bandits, most of the early works study a global structure on the reward function class, such as linear~\citep{dani08stochastic,abbasi2011improved,chu2011contextual},
generalized linear models~\citep{filippi10parametric,li2017provably}, and general nonlinear~\citep{russo13eluder}.

There are two main lines of work for efficient stochastic contextual bandit algorithms with general function approximation: reduction to cost-sensitive classification~\citep{langford2007epoch,dudik2011efficient,agarwal14taming,krishnamurthy2016contextual,krishnamurthy2020contextual,majzoubi2020efficient} and reduction to regression~\citep{abe1999associative,agarwal12contextual,russo13eluder,foster20beyond,simchi2022bypassing,xu2020upper}.
\citet{langford2007epoch} first reduces the contextual bandit problem to cost-sensitive classification among a given policy class and achieves $\iupbound{T^{2/3}}$ regret.
\citet{dudik2011efficient} improves the regret to $\iupbound{\isqrt{T}}$ with $\iupbound{T^5}$ oracle calls per round.
\citet{agarwal14taming} further reduces the total oracle calls to $\iupboundlog{\isqrt{KT}}$.
The regression-based approach directly learns a reward model and uses it to guide exploration, which bypasses the computational hardness of agnostic classification in policy search-based approaches~\citep{guruswami2009hardness,feldman2006new}.

\paragraph{Contextual Bandits with Online Regression Oracles}
Online regression oracles have been widely used in contextual bandit algorithms since \citet{foster20beyond}, who proposes the \squarecb algorithm to reduce the contextual bandit problem to online regression by using inverse gap weighting (IGW)~\citep{abe1999associative} to guide exploration and achieves $\iupbound{\isqrt{KT \onlinereg}}$ regret, where $\onlinereg$ is the online regression oracle's estimation error. More recent works further improve the \squarecb algorithm in various aspects: \citet{foster2020adapting} analyzes the linear reward setting with misspecification error; \citet{foster2021efficient} extends to first-order regret guarantees via online log-loss regression; \citet{zhu2022contextuala,zhu2022contextualb} extend the \squarecb algorithm to large action space settings with
smooth regret and per-context linear reward structure, respectively; \citet{levy2025regret} extends to the delayed feedback setting.
\citet{foster2021statistical} gives a unified framework, \etwod, that reduces online learning to online estimation in various interactive learning problems including contextual bandits.

\paragraph{Contextual Bandits with Offline Regression Oracles}
Offline regression oracles are more practical in real-world applications (for example, Empirical Risk Minimization and its variants including regularized least squares or logistic regression).
Finding effective exploration strategies is a key challenge in designing contextual bandit algorithms with offline regression oracles.
IGW-based exploration is sufficient to construct an action distribution that satisfies both the low-regret (LR) and good-coverage (GC) conditions in the discrete action space setting~\citep{simchi2022bypassing} and linear reward setting~\citep[][Section 4]{xu2020upper}.
\citet{xu2020upper} analyzes the contextual bandit problem under a general reward function class and proposes the upper counterfactual confidence bound (UCCB) algorithm using
counterfactual action divergence to unify exploration strategies in various function classes. However, their algorithm requires $O(T)$ calls to the offline regression oracle.
\citet{sahaefficient} solves a linear contextual dueling bandit problem where they merge the LR and GC conditions into one constraint to design an efficient algorithm. This inspires our definition of DOEC, where any exploration distribution certifying an upper bound on DOEC satisfies LR and GC simultaneously.
\citet{foster2024online} provides \etwodoff that combines the \etwod reduction and a new reduction from online estimation to offline estimation; their number of calls to the offline oracle is $O(T)$, and when instantiated to contextual bandits, it has a suboptimal dependence of $\sqrt{\log|\Fcal|}$ in the regret guarantee.
We refer the reader to Table~\ref{tab:comp} for a comparison of the number of oracle calls required for existing oracle-efficient algorithms.

\begin{table*}[t]
    \centering
    \begin{tabular}{cccc}
    \toprule
    Algorithm & Regression  & Total \# oracle  & Assumptions \\
              & oracle      & calls            & \\
    \midrule
        \etwod~\citep{foster2021statistical} & Online  & $T$              & General function classes \\
        \falcon~\citep{simchi2022bypassing}& Offline & $\log(T)$\ *     &  Discrete action space\\

        \linearfalcon~\citep[][Sec. 4]{xu2020upper}& Offline & $\log(T)$\ *     &  Per-context linear reward \\
        UCCB~\citep{xu2020upper} & Offline & $T$                    & General function classes\\
        \etwodoff~\citep{foster2024online}
        &
        Offline&
        $T$
        &
        General function classes
        \\
        \midrule
        \generalfalcon (this paper) & Offline& $\log(T)$\ *      & General function classes \\

    \bottomrule
    \end{tabular}
    \caption{Performance comparison of $\sqrt{T}$-regret contextual bandit algorithms with regression oracles. *: If $T$ is known, the number of oracle calls can be reduced to $O(\log\log(T))$ with a specific schedule (details can be found in Appendix~\ref{app:regret-small-epoch-schedule}).
    }
    \label{tab:comp}
\end{table*}

\paragraph{Experimental Design}
Experimental design is a classical topic in statistics.
The classical result of~\citep{kiefer1960equivalence} shows that for linear regression,
G-optimal design, minimizing the
maximum predictive variance over the  space of covariates, is equivalent to D-optimal design, i.e., designing the data distribution such that the determinant of its covariance matrix is maximized.
\citet[Ch.~22]{lattimore18bandit} first uses G-optimal design in the fixed action set linear bandit setting.
Beyond the linear regime, \citet{agarwal2024non} propose nonlinear F-design, which generalizes G-optimal design to general function classes to guide exploration in many interactive learning settings.
These works inspire our exploitative F-design for exploration with general function approximation.

\section{Proofs for Section~\ref{sec:offline-oracle-efficient-algorithm} Part 1: Regret Analysis for \generalfalcon}
\label{app:regret-analysis-offline}

In this section, we provide the regret analysis for \generalfalcon(\Cref{alg:general-falcon}).
Before diving into the proof details, we first summarize the main notations used in this section.

\subsection{Basic Notations}
\begin{itemize}
    \item $p, q$: distributions over the action space $\Acal$.
    \item $m(t)$: the epoch index that time step $t$ belongs to.
    \item $M$: total number of epochs up to time $T$, i.e., $M = m(T)$.
    \item $p_t$: the distribution constructed at time step $t$ based on the estimated reward function $\hat{f}_{m(t)}$ and context $x_t$.
    \item $\pi_m$: policy executed in epoch $m$, which maps from context space $\Xcal$ to distributions $\Lambda$.
    \item $\tau_m$: the terminal time step of epoch $m$.
    \item $\Rcal(p \mid x) \coloneq \EE_{a \sim p}[f^*(x, a)]$: ground-truth expected reward of distribution $p$ under context $x$.
    \item $\HRcal_m(p \mid x) \coloneq \EE_{a \sim p}[\hat{f}_m(x, a)]$: estimated expected reward of distribution $p$ under context $x$ in epoch $m$ based on the estimated reward function $\hat{f}_m$.

    \item $\REG( p \mid x) \coloneq \max_{\lambda' \in \Lambda} \Rcal(\lambda' \mid x) - \Rcal(p \mid x)$, the regret of distribution $p$
    under context $x$.

    \item $\HREG_{m}(p \mid x) \coloneq \max_{\lambda' \in \Lambda}\HRcal_{m}(\lambda' \mid x) - \HRcal_{m}(p \mid x)$, the estimated regret of distribution $p$ under context $x$ in epoch $m$ based on the estimated reward function $\hat{f}_m$.
    To avoid clutter, with a slight abuse of notation, we will use notation $\HREG_{m}(\pi \mid x)$ to denote $\max_{\lambda' \in \Lambda}\HRcal_{m}(\lambda' \mid x) - \HRcal_{m}(\pi(\cdot \mid x) \mid x)$, the estimated regret of policy $\pi$ under context $x$ in epoch $m$ based on the estimated reward function $\hat{f}_m$.

    \item $\WL_m(f, x) \coloneq \EE_{a \sim \pi_{m}(\cdot \mid x)}\sbr{ \del{f(x, a) - f^*(x, a)}^2 }$:the expected estimation error of $f$ measured by policies executed in epoch $m$ under context $x$.

\end{itemize}

\subsection{Virtual Policy Viewpoint of \generalfalcon}

Our analysis will adopt an equivalent ``virtual policy''  view of Algorithm~\ref{alg:general-falcon}, that is, Algorithm~\ref{alg:general-falcon-conceptual}. Specifically, Algorithm~\ref{alg:general-falcon-conceptual} computes a policy $\pi_m$ at the beginning of epoch $m$ such that for every $x \in \Xcal$, the action distribution $\pi_m(\cdot|x)$ certifies a small relaxed DOEC, which further implies low regret and good coverage properties (Eqs.~\eqref{eqn:low-regret-x-tag} and\eqref{eqn:good-coverage-x-tag}).
We clarify that Algorithm~\ref{alg:general-falcon-conceptual} is only meant to be conceptual since it needs to compute a separate exploration distribution for each context; we introduce it only for the sake of analysis -- specifically, we will prove online performance guarantees of policy sequence $\icbr{\pi_{m(t)}}_{t=1}^T$ on every context $x$.

\begin{algorithm}[t]
    \LinesNumbered
    \SetAlgoLined
    \SetAlgoVlined
    \KwIn{learning parameter $\cbr{\gamma_m, \varepsilon_{m}}_{m=1}^M$, epoch schedule $0 = \tau_0 < \tau_1 < \cdots < \tau_M$, benchmark distributions $\Lambda$, reward function class $\Fcal$, offline regression oracle $\offlineoracle$, relaxed coverage $\overline{\coverage}_{\varepsilon}$.}

    \For{epoch $m = 1$ \KwTo $M$}{
        \If {$m =1$}
        {
            Compute $\pi_1(\cdot \mid x) \coloneq \argmin_{p \in \co(\Lambda)} \max_{\lambda \in \Lambda} \overline{\coverage}_{\varepsilon_m}(p, \lambda; \Fcal_x), \forall x \in \Xcal$.
        }
        \Else{
        Compute $\hat{f}_m \gets \offlineoracle(\Fcal)(\cbr{(x_i, a_i, r_i)}_{i=\tau_{m-2}+1}^{\tau_{m-1}})$.

        Define policy $\pi_m$ as follows:
        \begin{equation}
        \pi_m(\cdot | x) \coloneq
        \argmin_{p \in \co(\Lambda)}
            \max_{\lambda \in \Lambda }
            \rbr{
            \HRcal_m(\lambda \mid x)
            -
            \HRcal_m(p \mid x)
            +
            \frac{\overline{\coverage}_{\varepsilon_m}(p, \lambda; \Fcal_{x})}{\gamma_m}
            },
            \forall x \in \Xcal.
        \label{eqn:monster-policy}
        \end{equation}
        }
        \For{round $t = \tau_m + 1$ \KwTo $\tau_{m+1}$}{
            Observe context $x_t \in \Xcal$, sample action $a_t \sim \pi_m(\cdot \mid x_t)$, and observe reward $r_t$.
        }
    }
    \caption{\generalfalcon: conceptual version} \label{alg:general-falcon-conceptual}
\end{algorithm}

Here we prove a straightforward extension of Lemma~\ref{lem:lr_gc}, namely Lemma~\ref{lemma:lr_gc_x}, that shows the virtual policies $\pi_m$'s satisfy low regret and good coverage properties simultaneously, which will serve as the basis of our subsequent regret analysis.

\begin{lemma}
    \label{lemma:lr_gc_x}
    For context $x$, distribution $\pi_m(\cdot |x )$ satisfies the following two properties simultaneously for $m \geq 2$:
    \begin{itemize}
        \item Low Regret:
            \begin{align}
                \HREG_m(\pi_m \mid x)
                \leq \overline{\doec}_{\gamma_m, \varepsilon_m}(\Fcal_{x}, \Lambda)
                    \tag{LR$_x$}
                    \label{eqn:low-regret-x-tag}
            \end{align}

        \item Good Coverage:
            \begin{align}
                \coverage_{\varepsilon_m}\rbr{ \pi_m, \lambda; \Fcal_{x} } \leq \gamma_m \overline{\doec}_{\gamma_m, \varepsilon_m}(\Fcal_{x}, \Lambda) + \gamma_m \HREG_m(\lambda \mid x),
                \quad
                \forall \lambda \in \Lambda,
                \tag{GC$_x$}        \label{eqn:good-coverage-x-tag}
            \end{align}
        \end{itemize}
        And for $m=1$, we have
        \[
        \coverage_{\varepsilon_1}(\pi_1, \lambda; \Fcal_x) \leq \gamma_1 \overline{\doec}_{\gamma_1,\varepsilon_1}(\Fcal_x, \Lambda),
        \quad \forall \lambda \in \Lambda
        \]
        where we recall the convention that $\gamma_1 = 0$ and $\gamma_1 \overline{\doec}_{\gamma_1,\varepsilon_1}(\Fcal_x, \Lambda)$ is interpreted as $\overline{\Vcal}_{\varepsilon_1}^*(\Fcal_x, \Lambda)$.
\end{lemma}

\begin{proof}
For $m \geq 2$, since by Eq.~\eqref{eqn:monster-policy}, $\pi_m(\cdot | x)$ minimizes
\[
\HREG_m(p \mid x)
+
\max_{\lambda \in \Lambda} \rbr{ \tfrac{1}{\gamma_m} \overline{\coverage}_{\varepsilon_m}(p, \lambda; \Fcal_{x})-\HREG_m(\lambda \mid x)},
\]
whose optimal objective is $\overline{\doec}_{\gamma_m,\varepsilon_m}(\Fcal_x, \Lambda)$. Therefore,
\[
\HREG_m(\pi_m \mid x)
+
\max_{\lambda \in \Lambda} \rbr{ \tfrac{1}{\gamma_m} \overline{\coverage}_{\varepsilon_m}(\pi_m, \lambda; \Fcal_{x})-\HREG_m(\lambda \mid x)}
\leq
\overline{\doec}_{\gamma_m,\varepsilon_m}(\Fcal_x, \Lambda),
\]
where both terms on the left-hand side are nonnegative.
Focusing on the first term, we get
\[
\HREG_m(\pi_m \mid x) \leq \overline{\doec}_{\gamma_m, \varepsilon_m}(\Fcal_x, \Lambda),
\]
thus proving Eq.~\eqref{eqn:low-regret-x-tag}.
On the other hand, focusing on the second term, we get
\[
\max_{\lambda \in \Lambda} \rbr{ \tfrac{1}{\gamma_m} \overline{\coverage}_{\varepsilon_m}(\pi_m, \lambda; \Fcal_{x})-\HREG_m(\lambda \mid x)}
\leq
\overline{\doec}_{\gamma_m,\varepsilon_m}(\Fcal_x, \Lambda);
\]
since $\coverage_{\varepsilon_m}(\pi_m, \lambda; \Fcal_x) \leq \overline{\coverage}_{\varepsilon_m}(\pi_m, \lambda; \Fcal_x)$ pointwise, this proves Eq.~\eqref{eqn:good-coverage-x-tag} with the original coverage on the left-hand side.

For $m = 1$, by the definition of $\pi_1$ in \Cref{alg:general-falcon-conceptual}, we have that $p_1$ minimizes $\max_{\lambda \in \Lambda} \overline{\coverage}_{\varepsilon_1}(p, \lambda; \Fcal_x)$, whose optimal objective is $\overline{\Vcal}_{\varepsilon_1}^*(\Fcal_x, \Lambda)$, which is equal to $\lim_{\gamma \to 0} \gamma \overline{\doec}_{\gamma,\varepsilon_1}(\Fcal_x, \Lambda)$ by the convention; the above inequality again follows from $\coverage_{\varepsilon_1} \leq \overline{\coverage}_{\varepsilon_1}$ pointwise.
\end{proof}

\subsection{Favorable Event}

We define an event that the offline regression oracle returns a good estimated reward function for each epoch $m$:
\[
    E_m = \cbr{ \EE_{x \sim \Dcal_X, a \sim \pi_{m-1}(\cdot | x)} \sbr{ \hat{f}_m(x,a) - f^*(x,a) }^2 \leq \offlinereg(\Fcal, \tau_{m-1}-\tau_{m-2}, \delta_m)  }.
\]
By taking a union bound over all epochs $m \in [1, M]$, we have that with probability at least $1 - \sum_{m=1}^M \delta_m \geq 1 - \delta$, event $E = \cap_{m=1}^M E_m$ holds. Throughout the rest of the proof, we condition on event $E$ happening.

\subsection{Results of Per-Context Regret Analysis}
Next, we prove the concentration between the true reward and the estimated reward in a per context manner in each epoch $m$.
\begin{lemma}[Off-policy Evaluation Per Context]
    \label{lemma:concentrate-reward-in-coverage-per-context}
    For any $\delta \in (0, 1)$, reward function class $\Fcal$, with probability at least $1 - \delta$, for all $x \in \Xcal$, $m \in [2, M]$ and any $\lambda \in \Delta(\Acal)$, we have
    \begin{align*}
        \abs{\HRcal_m(\lambda \mid x) - \Rcal(\lambda \mid x)} \leq \sqrt{\coverage_{\varepsilon_{m-1}}(\pi_{m-1}, \lambda; \Fcal_x \mid x) \del{ \WL_{m-1}(\hat{f}_m, x) + \varepsilon_{m-1}}}
    \end{align*}
\end{lemma}

\begin{proof}[Proof of Lemma~\ref{lemma:concentrate-reward-in-coverage-per-context}]
    Starting from the squared LHS of the desired inequality, we have:
    \begin{align*}
        & \del{ \HRcal_m(\lambda \mid x) - \Rcal(\lambda \mid x) }^2
        = \del{ \EE_{a \sim \lambda}\sbr{\hat{f}_m(x, a) - f^*(x, a)} }^2
            \nonumber
        \\
        =&
        \frac{ \del{\EE_{a \sim \lambda}\sbr{\hat{f}_m(x, a) - f^*(x, a)}}^2 }{
            \EE_{a \sim \pi_{m-1}(\cdot \mid x)} \sbr{\hat{f}_m(x, a) - f^*(x, a)}^2 + \varepsilon_{m-1}
        }
        \cdot
        \del{\WL_{m-1}(\hat{f}_m, x) + \varepsilon_{m-1} }
            \nonumber
        \\
        \leq&
        \coverage_{\varepsilon_{m-1}}(\pi_{m-1}, \lambda; \Fcal_x \mid x) \cdot \del{\WL_{m-1}(\hat{f}_m, x) + \varepsilon_{m-1} }
            \tag{Definition of $\coverage(\cdot)$}
    \end{align*}

    By taking a square root on both sides of the above inequality, we get the desired inequality.
\end{proof}

Based on the above lemmas, we are ready to prove a central lemma that controls the true regret and estimated regret of any policy for every $x$ in each epoch $m \in [2, M]$ with each other by a threshold $G_m(x)$.
Before we state the lemma, we first define $G_m(x)$:

    \begin{align}
        G_m(x) \coloneq \frac{1}{\gamma_m} \sum_{s=2}^m \rbr{ 2 \gamma_{s-1} \overline{\doec}_{\gamma_{s-1}, \varepsilon_{s-1}}(\Fcal_x, \Lambda) + 9 \gamma_{s}^2 \del{\WL_{s-1}(\hat{f}_{s}, x) + \varepsilon_{s-1}} }.
            \label{eqn:def-G-m}
    \end{align}
    Here we recall the convention that $\gamma_1 = 0$ and $\gamma_{1} \overline{\doec}_{\gamma_{1}, \varepsilon_{1}} = \overline{\Vcal}_{\varepsilon_1}^*(\Fcal_x, \Lambda)$.

\begin{lemma}[Regret Concentration in All Epochs per Context]
    \label{lemma:regret-concentration} For any $\delta \in (0, 1)$, with probability at least $1 - \delta$, for all $x \in \Xcal$, $m \in [2, M]$ and $\lambda \in \Lambda$, we have
    \begin{align}
        \REG(\lambda \mid x)    & \leq 2 \HREG_{m}(\lambda \mid x) + G_m(x)
            \label{eqn:control-true-regret-by-estimated}  \\
        \HREG_{m}(\lambda \mid x) & \leq 2 \REG(\lambda \mid x) + G_m(x)
            \label{eqn:control-estimated-regret-by-true}
    \end{align}
    and consequently, the above two inequalities also hold for all $\lambda \in \co(\Lambda)$.
\end{lemma}

\begin{proof}
    To avoid clutter, we use
    \begin{itemize}
        \item $\WL_{m-1} \leftarrow \WL_{m-1}(\hat{f}_m, x)$
        \item $\cov_{m-1}(\cdot) \leftarrow \coverage_{\varepsilon_{m-1}}(\pi_{m-1}(\cdot|x), \cdot; \Fcal_x)$
        \item $\overline{\doec}_{m-1} \leftarrow \overline{\doec}_{\gamma_{m-1}, \varepsilon_{m-1}}(\Fcal_x, \Lambda)$
        \item $\widehat{\lambda}_m \coloneq \argmax_{\lambda \in \Lambda} \HREG_{m}(\lambda \mid x)$
        \item $\lambda^* \coloneq \argmax_{\lambda \in \Lambda} \REG(\lambda \mid x)$
    \end{itemize}
    for notational simplicity in the following proof.

For $m \geq 2$, denote by
\[
    H_m(x) :=
    \gamma_m G_m(x)
    =
    \sum_{s=2}^m \rbr{ 2 \gamma_{s-1} \overline{\doec}_{s-1} + 9 \gamma_s^2 (\WL_{s-1} + \varepsilon_{s-1}) }.
\]
We will subsequently use the property that
\begin{equation}
H_m(x) - H_{m-1}(x) = 2 \gamma_{m-1} \overline{\doec}_{m-1} + 9 \gamma_m^2 (\WL_{m-1} + \varepsilon_{m-1})
\label{eqn:h-m-diff}
\end{equation}

\paragraph{Base case:} For $m=2$, Lemma~\ref{lemma:concentrate-reward-in-coverage-per-context} implies that for any $\lambda \in \Lambda$,
\[
|\widehat{\Rcal}_2(\lambda \mid x) - \Rcal(\lambda \mid x) |
\leq
\sqrt{ \cov_{1}(\lambda) (\WL_{1} + \varepsilon_{1}) }
\leq
\sqrt{\gamma_{1} \overline{\doec}_{\gamma_1, \varepsilon_1} \cdot (\WL_{1} + \varepsilon_{1}) },
\]
where the second inequality is because Lemma~\ref{lemma:lr_gc_x} guarantees that for all $\lambda \in \Lambda$,
$\coverage_{\varepsilon_{1}}(\pi_{1}(\cdot|x), \lambda; \Fcal_x \mid x) \leq \overline{\Vcal}_{\varepsilon_1}^*(\Fcal_x, \Lambda) = \gamma_1 \overline{\doec}_{\gamma_1, \varepsilon_1}(\Fcal_x, \Lambda)$.

Thus, Lemma~\ref{lem:reg-general} (given at the end of the proof) implies that
\begin{align*}
\abs{\REG(\lambda \mid x) - \HREG_2(\lambda \mid x)}
\leq &
 2\sqrt{ \gamma_1 \overline{\doec}_{\gamma_1, \varepsilon_1}(\Fcal_x, \Lambda) \cdot (\WL_{1} + \varepsilon_{1}) } \\
\leq &
\frac{\gamma_1}{\gamma_2} \overline{\doec}_{\gamma_1, \varepsilon_1}(\Fcal_x, \Lambda) + \gamma_2 (\WL_{1} + \varepsilon_{1})
\\
\leq & G_2(x),
\end{align*}
where the second inequality is from AM-GM inequality, and the third inequality is by the definition of $G_m(x)$.
This establishes Eqs.~\eqref{eqn:control-true-regret-by-estimated} and~\eqref{eqn:control-estimated-regret-by-true} for $m=2$.

\paragraph{Inductive step:} Assume that Eqs.~\eqref{eqn:control-true-regret-by-estimated} and~\eqref{eqn:control-estimated-regret-by-true} hold for epoch $m-1$. We will show that the same inequalities hold for epoch $m$.

\paragraph{Proof of Eq.~\eqref{eqn:control-true-regret-by-estimated}:}
We will start by upper bounding $\REG(\lambda \mid x) - \HREG_{m}(\lambda \mid x)$:
\begin{align}
    & \REG(\lambda \mid x) - \HREG_{m}(\lambda \mid x)
    \nonumber
    \\
    \leq& \abs{ \Rcal(\lambda^* \mid x) - \HRcal_m(\lambda^* \mid x) }
    + \abs{ \Rcal(\lambda \mid x) - \HRcal_m(\lambda \mid x) }
        \tag{Lemma~\ref{lem:reg-general}}
    \\
    \leq& \frac1{\gamma_m}
    \rbr{\sqrt{\cov_{m-1}(\lambda^*) \gamma_m^2(\WL_{m-1} + \varepsilon_{m-1}) }
        + \sqrt{\cov_{m-1}(\lambda) \gamma_m^2(\WL_{m-1} + \varepsilon_{m-1})}}
        \tag{Applying Lemma~\ref{lemma:concentrate-reward-in-coverage-per-context} and algebra}
    \\
    \leq & \frac{1}{\gamma_m} \rbr{  3\gamma_m^2(\WL_{m-1} + \varepsilon_{m-1}) + \frac{1}{6} \cov_{m-1}(\lambda^*)
    +
    \frac{1}{6} \cov_{m-1}(\lambda)
    }
        \label{eqn:interm-control-true-regret}
    \end{align}
    where the last inequality is by AM-GM inequality.

    Continuing Eq.~\eqref{eqn:interm-control-true-regret}, we now upper bound $\cov_{m-1}(\lambda^*)$ and $\cov_{m-1}(\lambda)$ by the good coverage property of $\pi_{m-1}$ as well as inductive hypothesis:

    \begin{align*}
    \eqref{eqn:interm-control-true-regret}
    \leq& \frac{1}{\gamma_m} \rbr{ \twolines{ 3\gamma_m^2 (\WL_{m-1} + \varepsilon_{m-1}) + \frac{\gamma_{m-1}}{3} \overline{\doec}_{m-1} }{ + \frac{\gamma_{m-1}}{6}\HREG_{m-1}(\lambda^* \mid x)
    + \frac{\gamma_{m-1}}{6}\HREG_{m-1}(\lambda \mid x)}}
        \tag{Good Coverage property for $\pi_{m-1}$, Lemma~\ref{lemma:lr_gc_x}}
    \\
    \leq&
    \frac{1}{\gamma_m} \rbr{ \twolines{ 3\gamma_m^2 (\WL_{m-1} + \varepsilon_{m-1}) + \frac{\gamma_{m-1}}{3} \overline{\doec}_{m-1} }{ + \frac{\gamma_{m-1}}{3}\REG(\lambda^* \mid x)
    + \frac{\gamma_{m-1}}{3}\REG(\lambda \mid x)
    + \frac{\gamma_{m-1}G_{m-1}(x)}{3}
    }}
    \tag{Inductive hypothesis and algebra}
    \\
    \leq&
    \frac{1}{\gamma_m} \rbr{  3\gamma_m^2 (\WL_{m-1} + \varepsilon_{m-1}) + \frac{\gamma_{m-1}}{3} \overline{\doec}_{m-1} +
    \frac{H_{m-1}(x)}{3}
    }
    + \frac{1}{3}\REG(\lambda^* \mid x)
    + \frac{1}{3}\REG(\lambda \mid x)
    \tag{$\gamma_m \geq \gamma_{m-1}$, $H_{m-1}(x) = \gamma_{m-1}G_{m-1}(x)$}
    \\
    \leq& \frac{H_m(x)}{3\gamma_m} + \frac{1}{3} \REG(\lambda \mid x)
    \tag{$\REG(\lambda^* \mid x) = 0$ and Eq.~\eqref{eqn:h-m-diff}}
    \\
    = &
    \frac13 G_m(x) + \frac{1}{3} \REG(\lambda \mid x)
    \tag{Definition of $H_m(x) = \gamma_m G_m(x)$}
\end{align*}
Solving the inequality for $\REG(\lambda \mid x)$, we get:
\[
    \frac{2}{3}\REG(\lambda \mid x) \leq \HREG_m(\lambda \mid x) + \frac13 G_m(x)
    \implies
    \REG(\lambda \mid x) \leq \frac32 \HREG_m(\lambda \mid x) + \frac{1}{2} G_m(x),
\]
thus establishing Eq.~\eqref{eqn:control-true-regret-by-estimated}.

\paragraph{Proof of Eq.~\eqref{eqn:control-estimated-regret-by-true}:}
We will start by upper bounding $\HREG_{m}(\lambda \mid x) - \REG(\lambda \mid x)$:
\begin{align}
    & \HREG_{m}(\lambda \mid x)
    - \REG(\lambda \mid x)
    \nonumber
    \\
    \leq& \abs{ \Rcal(\widehat{\lambda}_m \mid x) - \HRcal_m(\widehat{\lambda}_m \mid x) }
    + \abs{ \Rcal(\lambda \mid x) - \HRcal_m(\lambda \mid x) }
        \tag{Lemma~\ref{lem:reg-general}}
    \\
    \leq& \frac1{\gamma_m}
    \rbr{ \sqrt{\cov_{m-1}(\widehat{\lambda}_m) \gamma_m^2(\WL_{m-1} + \varepsilon_{m-1}) }
    + \sqrt{\cov_{m-1}(\lambda) \gamma_m^2(\WL_{m-1} + \varepsilon_{m-1})}
    }
        \tag{Applying Lemma~\ref{lemma:concentrate-reward-in-coverage-per-context} and algebra}
    \\
    \leq & \frac{1}{\gamma_m} \rbr{  3\gamma_m^2(\WL_{m-1} + \varepsilon_{m-1}) + \frac{1}{6} \cov_{m-1}(\widehat{\lambda}_m)
    +
    \frac{1}{6} \cov_{m-1}(\lambda)
    }
        \label{eqn:interm-control-est-regret}
    \end{align}
    where the last inequality is by AM-GM inequality.

    Continuing Eq.~\eqref{eqn:interm-control-est-regret}, we now upper bound $\cov_{m-1}(\widehat{\lambda}_m)$ and $\cov_{m-1}(\lambda)$ by the good coverage property of $\pi_{m-1}$ as well as inductive hypothesis:

    \begin{align*}
    \eqref{eqn:interm-control-est-regret}
    \leq& \frac{1}{\gamma_m} \rbr{ \twolines{ 3\gamma_m^2 (\WL_{m-1} + \varepsilon_{m-1}) + \frac{\gamma_{m-1}}{3} \overline{\doec}_{m-1} }{ + \frac{\gamma_{m-1}}{6}\HREG_{m-1}(\widehat{\lambda}_m \mid x)
    + \frac{\gamma_{m-1}}{6}\HREG_{m-1}(\lambda \mid x)}}
        \tag{Good Coverage property for $\pi_{m-1}$, Lemma~\ref{lemma:lr_gc_x}}
    \\
    \leq&
    \frac{1}{\gamma_m} \rbr{ \twolines{ 3\gamma_m^2 (\WL_{m-1} + \varepsilon_{m-1}) + \frac{\gamma_{m-1}}{3} \overline{\doec}_{m-1} }{ + \frac{\gamma_{m-1}}{3}\REG(\widehat{\lambda}_m \mid x)
    + \frac{\gamma_{m-1}}{3}\REG(\lambda \mid x)
    + \frac{\gamma_{m-1}G_{m-1}(x)}{3}
    }}
    \tag{Inductive hypothesis and algebra}
        \\
    \leq&
    \frac{1}{\gamma_m} \rbr{ \twolines{ 3\gamma_m^2 (\WL_{m-1} + \varepsilon_{m-1}) + \frac{\gamma_{m-1}}{3} \overline{\doec}_{m-1} }{ + \frac{\gamma_{m-1} G_m(x)}{3}
    + \frac{\gamma_{m-1}}{3}\REG(\lambda \mid x)
    + \frac{\gamma_{m-1}G_{m-1}(x)}{3}
    }}
    \tag{Eq.~\eqref{eqn:control-true-regret-by-estimated} for epoch $m$}
    \\
    \leq&
    \frac{1}{\gamma_m} \rbr{  3\gamma_m^2 (\WL_{m-1} + \varepsilon_{m-1}) + \frac{\gamma_{m-1}}{3} \overline{\doec}_{m-1} +
    \frac{H_{m-1}(x)}{3}
    }
    + \frac{1}{3}G_m(x)
    + \frac{1}{3}\REG(\lambda \mid x)
    \tag{$\gamma_m \geq \gamma_{m-1}$, $H_{m-1}(x) = \gamma_{m-1}G_{m-1}(x)$}
    \\
    \leq& \frac{H_m(x)}{3\gamma_m} + \frac13 G_m(x) + \frac{1}{3} \REG(\lambda \mid x)
    \tag{Eq.~\eqref{eqn:h-m-diff}}
    \\
    = &
    \frac23 G_m(x) + \frac{1}{3} \REG(\lambda \mid x)
    \tag{Definition of $H_m(x) = \gamma_m G_m(x)$ and algebra}
\end{align*}

Solving the inequality for $\HREG_m(\pi \mid x)$, we get:
\[
    \HREG_m(\lambda \mid x) \leq \frac{4}{3}\REG(\pi \mid x) + \frac23 G_m(x),
\]
thus establishing Eq.~\eqref{eqn:control-estimated-regret-by-true}.

This completes the inductive step. And thus, Eqs.~\eqref{eqn:control-true-regret-by-estimated} and~\eqref{eqn:control-estimated-regret-by-true} hold for all $m =2,\ldots,M$.

Finally, since Eqs.~\eqref{eqn:control-true-regret-by-estimated} and~\eqref{eqn:control-estimated-regret-by-true} are concerned with $\REG(\lambda)$ and $\HREG_m(\lambda)$, both of which are linear in $\lambda$, they must also hold for all $\lambda \in \co(\Lambda)$ as well.
\end{proof}

\begin{lemma}
For any $m \geq 2$ and any $\lambda$ in $\Delta(\Acal)$,
\begin{equation}
\REG(\lambda \mid x ) - \widehat{\REG}_{m}(\lambda \mid x )
\leq \abs{ \Rcal(\lambda^* \mid x) - \HRcal_m(\lambda^* \mid x) }
    + \abs{ \Rcal(\lambda \mid x) - \HRcal_m(\lambda \mid x) }
\label{eqn:reg-hreg-general}
\end{equation}
\begin{equation}
\widehat{\REG}_{m}(\lambda \mid x )
-\REG(\lambda \mid x )
\leq \abs{ \Rcal(\widehat{\lambda}_m \mid x) - \HRcal_m(\widehat{\lambda}_m \mid x) }
    + \abs{ \Rcal(\lambda \mid x) - \HRcal_m(\lambda \mid x) }
\label{eqn:hreg-reg-general}
\end{equation}
\label{lem:reg-general}
\end{lemma}
\begin{proof}
We first show Eq.~\eqref{eqn:reg-hreg-general}:
\begin{align*}
& \REG(\lambda \mid x ) - \widehat{\REG}_{m}(\lambda \mid x ) \\
= &  \del{ \Rcal(\lambda^* \mid x) - \Rcal(\lambda \mid x) }
    - \del{ \HRcal_m (\widehat{\lambda}_{m} \mid x) - \HRcal_m(\lambda \mid x) }
        \nonumber
    \\
    \leq& \del{ \Rcal(\lambda^* \mid x) - \Rcal(\lambda \mid x) }
    - \del{ \HRcal_m(\lambda^* \mid x) - \HRcal_m(\lambda \mid x)}
        \tag{Since $\HRcal_m(\widehat{\lambda}_{m} \mid x) \geq \HRcal_m(\lambda^*\mid x)$}
    \\
    \leq& \abs{ \Rcal(\lambda^* \mid x) - \HRcal_m(\lambda^* \mid x) }
    + \abs{ \Rcal(\lambda \mid x) - \HRcal_m(\lambda \mid x) }
        \tag{Regrouping}
\end{align*}
Eq.~\eqref{eqn:hreg-reg-general} follows from a similar calculation, as we detail below:
\begin{align*}
& \widehat{\REG}_{m}(\lambda \mid x ) - \REG(\lambda \mid x ) \\
= &  \del{ \HRcal_m (\widehat{\lambda}_{m} \mid x) - \HRcal_m(\lambda \mid x) } - \del{ \Rcal(\lambda^* \mid x) - \Rcal(\lambda \mid x) }
\nonumber
    \\
    \leq& \del{ \HRcal_m (\widehat{\lambda}_{m} \mid x) - \HRcal_m(\lambda \mid x) } - \del{ \Rcal(\widehat{\lambda}_m \mid x) - \Rcal(\lambda \mid x) }
        \tag{Since $\Rcal(\lambda^* \mid x) \geq \Rcal( \widehat{\lambda}_{m} \mid x)$}
    \\
    \leq& \abs{ \Rcal(\widehat{\lambda}_m \mid x) - \HRcal_m(\widehat{\lambda}_m \mid x) }
    + \abs{ \Rcal(\lambda \mid x) - \HRcal_m(\lambda \mid x) }
        \tag{Regrouping}
\end{align*}
\end{proof}
    We have obtained concentration between the true regret and the estimated regret in a per context manner.
    Next, we are ready to present an important intermediate lemma that controls the true regret at each time step in each epoch $m$ for each context $x$.

    \begin{lemma}[$\Lambda$-Regret Bound per Context for \generalfalcon in Each Epoch]
        \label{lemma:regret-bound-per-epoch}
        For any $\delta \in (0, 1)$, with probability at least $1 - \delta$, for all $x \in \Xcal$, $m \in [2, T]$,
        we have that
        \begin{align}
            \REG(\pi_{m} \mid x)
            \leq
            \frac{1}{\gamma_m}
            \rbr{ \sum_{s=1 }^m 3 \gamma_s \overline{\doec}_{\gamma_s, \varepsilon_s}(\Fcal_x, \Lambda) +
            \sum_{s=2}^m
            9 \gamma_{s}^2 \del{ \WL_{s-1}(\hat{f}_s, x) + \varepsilon_{s-1}}
            }.
                \label{eqn:regret-bound-per-epoch}
        \end{align}
    \end{lemma}

    The upper bound on the per-context regret in \Cref{eqn:regret-bound-per-epoch} is a weighted sum among all epochs $s \leq m$.
    Each epoch $s$ contributes to the upper bound in two ways: the first term $\overline{\doec}_{\gamma_s, \varepsilon_s}(\Fcal_x, \Lambda)$ captures the exploration difficulty for context $x$ in epoch $s$; and the second term $\gamma_{s}^2 \del{ \WL_{s-1}(\hat{f}_s, x) + \varepsilon_{s-1} }$ quantifies the squared loss of the regression function $\hat{f}_s$ with respect to the actions taken by the policy in epoch $s-1$, weighted by the exploration parameter $\gamma_s$ squared.

    \begin{proof}[Proof of Lemma~\ref{lemma:regret-bound-per-epoch}]
        Starting from Lemma~\ref{lemma:regret-concentration}, we have
        \begin{align*}
            \REG(\pi_m \mid x)
            & \leq 2 \HREG_{m}(\pi_m \mid x) + G_m(x)
                \tag{Applying \Cref{eqn:control-true-regret-by-estimated} in Lemma~\ref{lemma:regret-concentration}}
            \\
            & \leq 2 \overline{\doec}_{\gamma_m, \varepsilon_m}(\Fcal_x, \Lambda) + G_m(x)
                \tag{Applying LR of $\pi_m$ (\Cref{eqn:low-regret-x-tag})}
        \end{align*}
    The lemma follows from the definition of $G_m(x)$.
    \end{proof}

\subsection{Concluding the Regret Analysis}
We are ready to prove the main theorem for \generalfalcon.

\percontextub*

\begin{proof}[Proof of Lemma~\ref{lemma:general-falcon-per-context-regret-bound}]
    In Lemma~\ref{lemma:regret-bound-per-epoch}, we have shown that for any context $x \in \Xcal$, the regret of the in-epoch policy $\pi_m$ is bounded. We then bound the total regret across all time steps as follows:
    \begin{align*}
        & \quad \sum_{t=1}^{T} \REG(\pi_{m(t)} \mid x)
        \leq \tau_1 + \sum_{m=2}^{M} (\tau_{m} - \tau_{m-1}) \REG(\pi_m \mid x)
        \\
        & \leq \tau_1 + \sum_{m=2}^{M} \frac{9 \tau_{m}}{\gamma_m} \rbr{
        \sum_{s=1}^M \gamma_s \overline{\doec}_{\gamma_s, \varepsilon_s}(\Fcal_x, \Lambda)
        +
        \sum_{s=2}^M \gamma_{s}^2 \del{ \WL_{s-1}(\hat{f}_s, x) + \varepsilon_{s-1}}
        }
            \tag{Applying Lemma~\ref{lemma:regret-bound-per-epoch} and relaxing $s \le m$ to $s \le M$}
        \\
        &\leq
        \upbound{
        \tau_1 + M \cdot \max_{m \in \cbr{2,\ldots,M}} \frac{\tau_{m}}{\gamma_m} \cdot \del{ %
        \sum_{s=1}^M \gamma_s \overline{\doec}_{\gamma_s, \varepsilon_s}(\Fcal_x, \Lambda)
        +
        \sum_{s=2}^M \gamma_{s}^2 \del{ \WL_{s-1}(\hat{f}_s, x) + \varepsilon_{s-1}}}
        }
            \tag{Relaxing $\tau_m / \gamma_m$ to the maximum}
    \end{align*}
\end{proof}

We are now ready to prove Theorem~\ref{thm:main-regret} by taking the average of the per-context regret  Lemma~\ref{lemma:general-falcon-per-context-regret-bound} over all contexts $x \sim \Dcal_X$:

\mainregret*

\begin{proof}[Proof of \Cref{thm:main-regret}]
    Since in Lemma~\ref{lemma:general-falcon-per-context-regret-bound}, we have a per-context total regret bound, we can directly use it here to prove this theorem by taking expectation over $x \sim \Dcal_X$, and applying the regression guarantee.
    Therefore, we have
    \begin{align*}
        &\quad \sum_{t=1}^{T} \EE_{x_t \sim \Dcal_X} \sbr{\REG(\pi_{m(t)} \mid x_t)}
        \\
        & \leq
        \upboundlog{ \tau_1 + \max_{m \in \cbr{2,\ldots,M}} \frac{\tau_{m}}{\gamma_m} \cdot \del{ \twolines{ \sum_{s=1}^M \gamma_s \EE_{x \sim \Dcal_X}\sbr{\,\overline{\doec}_{\gamma_s, \varepsilon_s}(\Fcal_x, \Lambda)} }{ +
        \sum_{m=2}^M \gamma_{s}^2 \del{ \EE_{x \sim \Dcal_X}\sbr{\WL_{s-1}(\hat{f}_s, x)} + \varepsilon_{s-1} } }}
        }
        \\
        & \leq
        \upboundlog{ \tau_1 + \max_{m \in \cbr{2,\ldots,M}} \frac{\tau_{m}}{\gamma_m} \cdot \del{ \twolines{ \sum_{s=1}^M \gamma_s \EE_{x \sim \Dcal_X}\sbr{\,\overline{\doec}_{\gamma_s, \varepsilon_s}(\Fcal_x, \Lambda)} }{ +
        \sum_{s=2}^M \gamma_{s}^2 \del{ \offlinereg(\Fcal, \tau_{s-1}/2, \delta) + \varepsilon_{s-1}} }}
        }
        \\
        & \leq
        \upboundlog{ \tau_1 + \max_{m \in \cbr{2,\ldots,M}} \frac{\tau_{m}}{\gamma_m} \cdot \del{ \twolines{ \max_{s \in [M]} \gamma_s \EE_{x \sim \Dcal_X}\sbr{\,\overline{\doec}_{\gamma_s, \varepsilon_s}(\Fcal_x, \Lambda)} }{ +
        \max_{s \in \cbr{2,\ldots,M}} \gamma_{s}^2 \del{ \offlinereg(\Fcal, \tau_{s-1}/2, \delta) + \varepsilon_{s-1}} }}
        }
    \end{align*}
    In the first inequality, we apply Lemma~\ref{lemma:general-falcon-per-context-regret-bound} and take expectation over $x \sim \Dcal_X$; in the second inequality, we apply the regression guarantee that $\EE_{x \sim \Dcal_X}\sbr{\WL_{s-1}(\hat{f}_s, x)}
    \leq \offlinereg(\Fcal, \tau_{s-1} - \tau_{s-2}, \delta_s) \leq
    \offlinereg(\Fcal, \tau_{s-1}/2, \delta_s)$ because  event $E$ happens and  $\offlinereg$ is monotonically decreasing in the sample size;
    in the last inequality, we relax $\sum_{s=1}^M$ to $M \cdot \max_{s \in [M]}$ and  $\sum_{s=2}^M$ to $M \cdot \max_{s \in \cbr{2,\ldots,M}}$.
\end{proof}

\subsection{Regret Guarantees under Two Epoch Schedules}
\label{app:regret-schedules}

In this subsection, we instantiate Theorem~\ref{thm:main-regret} under two concrete epoch schedules: a doubling schedule and a small epoch schedule, which in particular verify the claimed regret bounds of \generalfalcon in the three running examples.
We first present Assumption~\ref{assum:bounded-doec}, which is satisfied by the relaxed coverages (and their cushioned versions) in all three running examples with $D = A$, $d$ and $1/h$ respectively, as well as when the Eluder dimension of $\Fcal_x$ is polylogarithmic in the scale factor $\varepsilon$ for all $x$ (Proposition~\ref{prop:sec-eluder}, under the trivial relaxation $\overline{\coverage} = \coverage$);
and the offline oracle implements ERM with a finite class $\Fcal$~\citep[e.g.,][]{agarwal12contextual}.

\begin{assum} The following hold:
    \label{assum:bounded-doec}
    \begin{itemize}
    \item There exists a constant $D > 0$, such that for any context $x \in \Xcal$, we have $
        \max_{x \in \Xcal} \overline{\doec}_{\gamma, \varepsilon}(\Fcal_x, \Lambda) \leq \frac{D}{\gamma} \mathrm{polylog}\del{\frac{1}{\varepsilon}}.
    $

    \item The offline regression oracle ensures that $\offlinereg(\Fcal, T, \delta) \leq \frac{\log (|\Fcal|/\delta)}{T}$.
    \end{itemize}
\end{assum}

\subsubsection{Regret Bound with Doubling Schedule}
\begin{definition}[Doubling Schedule]
    \label{def:doubling}
    A schedule of $\tau_m$, $\gamma_m, \varepsilon_m$, $m = 1,\ldots,M$ is called a \emph{doubling schedule}, if
    \[
    M = \log(T), \tau_m = 2^m, \gamma_1 = 0, \gamma_m = \sqrt{\frac{D}{\offlinereg(\Fcal, \tau_{m-1}/2, \delta)}}, \varepsilon_m = \frac1T.
    \]
\end{definition}

\begin{corollary}
    \label{corol:regret-doubling-schedule}
    If \Cref{assum:bounded-doec} holds, by using the doubling schedule (Defined in \Cref{def:doubling}), \generalfalcon achieves a regret of $\iupboundlog{\sqrt{ D T \log |\Fcal|}}$.
\end{corollary}

\begin{proof}
    By the setting of the schedule we have
    \begin{align*}
        \max_{m \in \cbr{2,\ldots,M}} \frac{\tau_m}{\gamma_m} = \max_{m \in \cbr{2,\ldots,M}} \tau_m \sqrt{ \frac{\offlinereg(\Fcal, \tau_{m-1}/2, \delta)}{D}}
        =
        \upbound{\sqrt{\frac{T\log(|\Fcal|/\delta)}{ D }}}
    \end{align*}
    Then according to \Cref{thm:main-regret}, we have

    \begin{align*}
        & \quad \sum_{t=1}^{T} \EE_{x_t \sim \Dcal_X} \sbr{\REG(\pi_{m(t)} \mid x_t)}
        \\
        & \leq
        \upboundlog{\tau_1 + \max_{m \in \cbr{2,\ldots,M}} \frac{\tau_{m}}{\gamma_m} \cdot \del{ \twolines{ \max_{m \in [M]} \gamma_m \EE_{x \sim \Dcal_X}\sbr{\,\overline{\doec}_{\gamma_m, \varepsilon_m}(\Fcal_x, \Lambda)} }{ +
        \max_{m \in \cbr{2,\ldots,M}} \gamma_{m}^2 \del{ \offlinereg(\Fcal, \tau_{m-1}/2, \delta) + \varepsilon_{m-1}} }}
            \tag{Applying \Cref{thm:main-regret}}
        }
        \\
        & \leq
        \upboundlog{ \max_{m \in \cbr{2,\ldots,M}} \frac{\tau_{m}}{\gamma_m} \cdot \del{ \max_{m \in [M]} D \mathrm{polylog}(\frac{1}{\varepsilon_m}) + D }
            \tag{By the assumption on $\overline{\doec}$ and the choice of $\gamma_m$}
        }
        \\
        & \leq
        \upboundlog{\sqrt{\frac{T\log(|\Fcal|/\delta)}{ D }} \cdot D \mathrm{polylog}(T)}
            \tag{By Algebra}
        \\
        & =
        \upboundlog{\sqrt{ T D \log(|\Fcal| / \delta)}}
    \end{align*}
\end{proof}

\subsubsection{Regret Bound with Small Epoch Schedule}
\label{app:regret-small-epoch-schedule}
If the total time horizon $T$ is known in advance, we can further reduce the number of epochs to $O(\log\log(T))$, for a discrete action space setting, by using the following small epoch schedule.

\begin{definition}[Small Epoch Schedule]
    \label{def:small-epoch}
    A schedule of $\tau_m$, $\gamma_m, \varepsilon_m$, $m = 1,\ldots,M$
    is called a \emph{small epoch schedule}, s.t.
    \[
    M = \log\log(T),
    \tau_m = \floor{2T^{1 - 2^{-m}}}, \gamma_1 = 1,
    \gamma_m = \sqrt{\frac{ D}{\offlinereg(\Fcal, \tau_{m-1} - \tau_{m-2}, \delta)}},
    \varepsilon_m = \frac1T.
    \]
\end{definition}

\begin{corollary}
    \label{corol:regret-small-epoch-schedule}
    If \Cref{assum:bounded-doec} holds, by using the small epoch schedule (defined in \Cref{def:small-epoch}), \generalfalcon achieves a regret of $\iupboundlog{\sqrt{ D T \log |\Fcal|}}$.
\end{corollary}

\begin{proof}
    By the small epoch schedule, we have
    \begin{align*}
        \max_{m \in \cbr{2,\ldots,M}} \frac{\tau_m}{\gamma_m} = \max_{m \in [M]} \sqrt{\frac{\tau_m^2 \log(|\Fcal|/\delta)}{ D(\tau_{m-1} - \tau_{m-2}) }}
        \leq
        \sqrt{\frac{T\log(\abs{\Fcal}/\delta)}{D}}
    \end{align*}

    Then according to \Cref{thm:main-regret}, we have
    \begin{align*}
        & \quad \sum_{t=1}^{T} \EE_{x_t \sim \Dcal_X} \sbr{\REG(\pi_{m(t)} \mid x_t)}
        \\
        & \leq
        \upboundlog{ \tau_1 + \max_{m \in \cbr{2,\ldots,M}} \frac{\tau_{m}}{\gamma_m} \cdot \del{ \twolines{ \max_{m \in [M]} \gamma_m \EE_{x \sim \Dcal_X}\sbr{\,\overline{\doec}_{\gamma_m, \varepsilon_m}(\Fcal_x, \Lambda)} }{ +
        \max_{m \in \cbr{2,\ldots,M}} \gamma_{m}^2 \del{ \offlinereg(\Fcal, \tau_{m-1} - \tau_{m-2}, \delta) + \varepsilon_{m-1}} }}
            \tag{Applying \Cref{thm:main-regret}}
        }
        \\
        & \leq
        \upboundlog{ \max_{m \in \cbr{2,\ldots,M}} \frac{\tau_{m}}{\gamma_m} \cdot \del{ \max_{m \in [M]} D \mathrm{polylog}(\frac{1}{\varepsilon}) +
        D }
            \tag{By the assumption on $\overline{\doec}$ and the choice of $\gamma_m$}
        }
        \\
        & \leq
        \upboundlog{\sqrt{\frac{T\log(\abs{\Fcal}/\delta)}{D}}  \cdot D  \mathrm{polylog}(T)}
            \tag{By the maximum of $\tau_m/\gamma_m$}
        \\
        & =
        \upboundlog{\sqrt{ T D \log(|\Fcal| / \delta)}}
    \end{align*}
\end{proof}

\section{Proofs for Section~\ref{sec:offline-oracle-efficient-algorithm} Part 2: Relaxations Justification and Regret Guarantees}

\subsection{Computational Necessity of Relaxation}
\label{app:computational-hardness}

We show the following propositions on the computational hardness of calculating coverage, which is inspired by a computational hardness result of bandit learning~\citep{brukhim2025hardness}.
Proposition~\ref{prop:coverage-hardness} shows that evaluating the original coverage $\coverage_\varepsilon$ is \NP-hard for general succinctly represented classes, and Proposition~\ref{prop:doec-hardness} shows that solving the original exploitative F-design (Eq.~\eqref{eqn:monster-distn} with $\overline{\coverage} = \coverage_\varepsilon$) is \NP-hard, in contrast to the relaxed exploitative F-design in our running examples, whose closed-form coverage makes it a polynomial-time convex program (Lemma~\ref{lem:relaxed-doec-bound}).

\begin{proposition}[Evaluating the original coverage is \NP-hard]
\label{prop:coverage-hardness}
There is a polynomial-time mapping from $3$-CNF formulas $\phi$ (with $m$ clauses over $n$ variables) to instances $(\Acal_\phi, \Gcal_\phi, p, \varepsilon)$ with a designated action $a_0 \in \Acal_\phi$, where $\Acal_\phi$ has only two actions, $p$ is a distribution over $\Acal_\phi$, $\varepsilon \in [0, \frac12)$, and $\Gcal_\phi \subseteq [0,1]^{\Acal_\phi}$ is a finite, succinctly represented reward class---each $g \in \Gcal_\phi$ is specified by an index from which $g(a)$ is computable in $\poly(n,m)$ time for every action $a$. The mapping guarantees
\[
    \coverage_\varepsilon(p, \delta_{a_0}; \Gcal_\phi)
    \begin{cases}
        > 2, & \phi \text{ satisfiable},\\[2pt]
        \leq 1, & \phi \text{ unsatisfiable}.
    \end{cases}
\]
Consequently, unless $\mathsf{P} = \NP$, no polynomial-time algorithm evaluates the original coverage $\coverage_\varepsilon(\cdot,\cdot;\Gcal)$ with precision $\frac12$ for general succinctly represented classes $\Gcal$.
\end{proposition}

\begin{proof}
Given $\phi$ over variables $x_1, \ldots, x_n$, we set the action space $\Acal_\phi = \cbr{a_0, a_1}$ and set the reward function class
\[
    \Gcal_\phi = \cbr{\mathbf{0}} \cup \cbr{g_x : x \in \cbr{0,1}^n}, \quad\text{where}\quad
    g_x(a_0) = 1, \qquad g_x(a_1) = \mathbf{1}\sbr{x \text{ does not satisfy } \phi} .
\]
Each $g_x$ takes values in $\cbr{0,1} \subseteq [0,1]$, and given the index $x$ both entries are computable in time $\poly(n,m)$ (testing whether $x$ satisfies $\phi$); the class is finite and fixed for any $\varepsilon \in [0, \frac12)$. %

Consider $p = \delta_{a_1}$ and $\lambda = \delta_{a_0}$, then $\coverage_\varepsilon(\delta_{a_1}, \delta_{a_0}; \Gcal_\phi) = \sup_{g, g' \in \Gcal_\phi} \frac{(g(a_0) - g'(a_0))^2}{\varepsilon + (g(a_1) - g'(a_1))^2}$
For the expression inside the supremum to be nonzero, it is necessary to let $g' = \mathbf{0}$ and $g = g_x$ for some $x$ (or the other way around).

Therefore,
\[
    \coverage_\varepsilon(p, \delta_{a_0}; \Gcal_\phi)
    = \sup_{x \in \cbr{0,1}^n} \frac{1}{\varepsilon + \mathbf{1}\sbr{x \text{ does not satisfy } \phi}}
    = \frac{1}{\varepsilon + \min_{x} \mathbf{1}\sbr{x \text{ does not satisfy } \phi}} .
\]
Finally, $\min_x \mathbf{1}\sbr{x \text{ does not satisfy } \phi} = 0$ iff some $x$ satisfies $\phi$, i.e.\ iff $\phi$ is satisfiable, in which case the coverage equals $\tfrac1\varepsilon > 2$; otherwise the minimum is $1$ and the coverage equals $\tfrac{1}{\varepsilon + 1} \leq 1$. A polynomial-time evaluator of the coverage with precision $\frac12$ would thus decide $3$-SAT, giving $\mathsf{P} = \NP$.
\end{proof}

\begin{proposition}[Solving the original exploitative F-design is hard unless $\mathsf{NP} = \mathsf{RP}$]
\label{prop:doec-hardness}
There is a polynomial-time mapping from 3-CNF formulas $\phi$ (with $m$ clauses over $n$ variables) to instances $(\Acal, \Gcal_\phi, \Delta(\Acal), \varepsilon)$, with $\varepsilon \in [0, \frac14)$ and $\Gcal_\phi \subseteq [0,1]^{\Acal}$ a finite, succinctly represented reward class---each $g \in \Gcal_\phi$ is specified by an index from which $g(a)$ is computable in $\poly(n,m)$ time for every action $a$. The mapping guarantees that whenever $\phi$ is satisfiable, any $p^* \in \Delta(\Acal)$ that solves the original F-design (Eq.~\eqref{eqn:monster-distn} with $\overline{\coverage} = \coverage$ and $\gamma \to 0$) up to precision $\alpha \leq \frac{1}{4}$ satisfies
\[
    \Pr_{a \sim p^*}\bigl[a \text{ is a satisfying assignment of } \phi\bigr] \geq 0.55.
\]
Consequently, a polynomial-time solver for the original exploitative F-design with constant precision could decide $3$-SAT in randomized polynomial time: run the solver on the mapped instance, sample $a \sim p^*$, and verify that $a$ satisfies $\phi$. Unless $\mathsf{NP} = \mathsf{RP}$, no polynomial-time algorithm solves the original exploitative F-design with constant precision for general succinctly represented classes $\Gcal$.
\end{proposition}
\begin{proof}
We let the action set $\Acal = \cbr{0,1}^n$ be the set of all assignments to the $n$ variables, and let $\Gcal_\phi := \cbr{\mathbf{0}, g_\phi}$, where $\mathbf{0}$ is the zero function and $g_\phi(x) = \mathbf{1}\sbr{x \text{ satisfies } \phi}$.
Since $\Gcal_\phi$ has only two functions, with $\hat{g} = \mathbf{0}$ the exploitative F-design simplifies to
\[
\Vcal_\varepsilon^*\!\del{\Gcal_\phi, \Delta(\Acal)}
= \min_{p \in \Delta(\Acal)} \max_{x \in \Acal}
\frac{g_\phi(x)^2}{\varepsilon + \EE_{y \sim p}\sbr{g_\phi(y)^2}}.
\]
As a shorthand, denote $S(p) := \EE_{y \sim p}\sbr{g_\phi(y)^2} = \Pr_{y \sim p}\sbr{y \text{ satisfies } \phi}$.
Now consider any satisfiable $\phi$. We have $\max_x g_\phi(x)^2 = 1$, and thus our F-design objective function is:
\[
\max_x \frac{g_\phi(x)^2}{\varepsilon + S(p)}
=
\frac{1}{\varepsilon + S(p)}.
\]
Minimizing this over $p$ is equivalent to maximizing $S(p)$.
The maximum $S(p) = 1$ is achieved by concentrating $p$ on satisfying assignments, giving $\Vcal_\varepsilon^*\!\del{\Gcal_\phi, \Delta(\Acal)} = \frac{1}{1 + \varepsilon}$.

Now suppose $p^*$ achieves precision $\alpha \leq \frac{1}{4}$, i.e.,
\[
\frac{1}{\varepsilon + S(p^*)} \leq \frac{1}{1+\varepsilon} + \alpha \leq \frac{5}{4}
\quad\Longrightarrow\quad
\Pr_{a \sim p^*}\sbr{a \text{ satisfies } \phi} = S(p^*) \geq \frac{4}{5} - \varepsilon \geq 0.55.
\]

\textbf{Reduction to $\mathsf{RP}$.}
Given $\phi$, construct the instance in polynomial time, run the poly-time F-design solver to obtain $p^*$, and sample $a \sim p^*$.
If $a$ satisfies $\phi$, output ``satisfiable.''
If $\phi$ is unsatisfiable, $g_\phi \equiv \mathbf{0}$ so no $a$ satisfies $\phi$; the procedure never erroneously accepts.
If $\phi$ is satisfiable, the above shows the procedure accepts with probability at least $0.55 > \frac{1}{2}$.
Hence a polynomial-time F-design solver would place $3$-SAT in $\mathsf{RP}$, and unless $\mathsf{NP} = \mathsf{RP}$, no such solver can exist.
\end{proof}

\subsection{Validity of the Relaxed Coverages}
\label{app:relaxed-coverage}

Recall from Section~\ref{sec:offline-oracle-efficient-algorithm} that \generalfalcon computes its action sampling distribution by solving the relaxed exploitative F-design (Eq.~\eqref{eqn:monster-distn}), in which the original coverage $\coverage_{\varepsilon}$ (Eq.~\eqref{eqn:coverage}) is replaced by a relaxed coverage $\overline{\coverage}_\varepsilon$. The property of $\overline{\coverage}_\varepsilon$ that our regret analysis hinges on is that it is a \emph{pointwise upper bound} of $\coverage_{\varepsilon}$: any good coverage guarantee with respect to $\overline{\coverage}_\varepsilon$ then implies the same guarantee with $\coverage_{\varepsilon}$. In this subsection, we verify this property for the relaxed coverages defined in our running examples (Lemma~\ref{lem:relaxed-coverage-valid}). Notably, none of these relaxations depends on the ``cushion parameter'' $\varepsilon$: each of them upper bounds $\coverage_{\varepsilon}$ for every $\varepsilon \geq 0$, and we accordingly write them without the subscript $\varepsilon$.

Throughout this subsection, we fix a context $x$ and abbreviate $\Gcal = \Fcal_x$ and $\phi(a) = \phi(x,a)$; in the (generalized) linear settings, we write $\Theta_x := \cbr{\theta(x) : \theta \in \Theta}$ for the per-context parameter set. Since the relaxed coverages below may have zero denominators, we adopt the conventions $\frac{0}{0} := 0$ and $\frac{c}{0} \coloneq +\infty$ for $c > 0$; in the (generalized) linear settings,
we define:
\[
\tr( A^{-1} C ) :=
\begin{cases}
+\infty, & \mathrm{range}(A) \subsetneq \mathrm{range}(C) \\
\tr( A^\dagger C ), & \mathrm{range}(A) \subseteq \mathrm{range}(C)
\end{cases}
\]

\begin{lemma}[Relaxed coverages upper bound the original coverage]
\label{lem:relaxed-coverage-valid}
Let $p$ be a nonnegative measure over $\Acal$, $\lambda \in \Lambda$, and $\varepsilon \geq 0$. Then
$\coverage_{\varepsilon}(p, \lambda; \Gcal) \leq \overline{\coverage}_{\varepsilon}(p, \lambda; \Gcal)$
holds in each of the following settings:
\begin{enumerate}
    \item \textbf{(Discrete action space)} $|\Acal| < \infty$, $\Gcal \subseteq [0,1]^{\Acal}$, $\Lambda = \cbr{\delta_a : a \in \Acal}$, and
    $\overline{\coverage}_{\varepsilon}(p, \lambda; \Gcal) = \sum_{a \in \Acal} \frac{\lambda(a)}{p(a)}$.
    \item \textbf{(Per-context generalized linear reward)} $\Gcal = \cbr{a \mapsto \sigma(\phi(a)^\top \theta) : \theta \in \Theta_x}$, $\Lambda = \cbr{\delta_a : a \in \Acal}$,  with link function $\sigma$ satisfying $0 < \underline{L} \leq \sigma' \leq \overline{L}$, and
    $\overline{\coverage}_{\varepsilon}(p, \lambda; \Gcal) = \kappa^2 \tr\del{\Sigma_p^{-1} \Sigma_\lambda}$ with $\kappa := \overline{L}/\underline{L}$.
    \item \textbf{($h$-smoothed regret)} $\Gcal \subseteq [0,1]^{\Acal}$, $\lambda \in \Delta_h^\mu(\Acal)$, $p$ admits a density with respect to $\mu$, and
    $\overline{\coverage}_{\varepsilon}(p, \lambda; \Gcal) = \frac{1}{h} \EE_{a \sim \mu}\sbr{\frac{\lambda(a)}{p(a)}}$, where $\lambda(a)$ and $p(a)$ denote the densities of $\lambda$ and $p$ with respect to $\mu$.
\end{enumerate}
\end{lemma}

\begin{proof}
Fix $g, g' \in \Gcal$ and let $\delta \coloneqq g - g'$
denote their difference. By the definition of $\coverage_{\varepsilon}$ (Eq.~\eqref{eqn:coverage}) and since $\varepsilon \geq 0$, it suffices to show
\[
\del{\EE_{a \sim \lambda}[\delta(a)]}^2 \leq \overline{\coverage}_{\varepsilon}(p, \lambda; \Gcal) \cdot \EE_{a \sim p}[\delta(a)^2]
\]
in each setting. We may assume $\overline{\coverage}_{\varepsilon}(p, \lambda; \Gcal) < \infty$, as the claim is trivial otherwise.

\textbf{Setting 1.} Finiteness of coverage implies $p(a) > 0$ whenever $\lambda(a) > 0$. By the Cauchy--Schwarz inequality (with sums restricted to the support of $\lambda$),
\[
\del{\sum_{a \in \Acal} \lambda(a) \delta(a)}^2
\leq \del{\sum_{a \in \Acal} \frac{\lambda(a)^2}{p(a)}} \del{\sum_{a \in \Acal} p(a) \delta(a)^2}
\leq \del{\sum_{a \in \Acal} \frac{\lambda(a)}{p(a)}} \del{\sum_{a \in \Acal} p(a) \delta(a)^2},
\]
where the last inequality uses $\lambda(a) \leq 1$.

\textbf{Setting 2.}
By the mean value theorem, $\delta(a) = \sigma'(\xi_a)\, \phi(a)^\top u$ for some $\xi_a$, where $u := \theta - \theta'$; since $\underline{L} \leq \sigma' \leq \overline{L}$, Jensen's inequality gives
\[
\del{\EE_{a \sim \lambda}[\delta(a)]}^2 \leq \EE_{a \sim \lambda}\sbr{\delta(a)^2} \leq \overline{L}^2\, u^\top \Sigma_\lambda u,
\qquad
\EE_{a \sim p}\sbr{\delta(a)^2} \geq \underline{L}^2\, u^\top \Sigma_p u .
\]
Then by applying \Cref{lem:trace-quadratic-bound} with $C = \Sigma_{\lambda}$ and $A = \Sigma_p$ we have
$
u^\top \Sigma_\lambda u
\leq \tr\del{\Sigma_p^{-1} \Sigma_\lambda}\, u^\top \Sigma_p u.
$
Combining the three inequalities gives $\del{\EE_{a \sim \lambda}[\delta(a)]}^2 \leq \kappa^2 \tr\del{\Sigma_p^{-1} \Sigma_\lambda}\, \EE_{a \sim p}\sbr{\delta(a)^2}$. The per-context linear reward is the special case where $\sigma$ is the identity function.

\textbf{Setting 3.} Finiteness of coverage implies $p(a) > 0$ for $\mu$-almost every $a$ with $\lambda(a) > 0$. By the Cauchy--Schwarz inequality with respect to $\mu$,
\[
\del{\EE_{a \sim \mu}\sbr{\lambda(a) \delta(a)}}^2
\leq \EE_{a \sim \mu}\sbr{\frac{\lambda(a)^2}{p(a)}} \cdot \EE_{a \sim \mu}\sbr{p(a)\, \delta(a)^2}
\leq \frac{1}{h} \EE_{a \sim \mu}\sbr{\frac{\lambda(a)}{p(a)}} \cdot \EE_{a \sim \mu}\sbr{p(a) \delta(a)^2},
\]
where the last inequality uses $\lambda(a) \leq \frac{1}{h}$ (as $\lambda \in \Delta_h^\mu(\Acal)$).
\end{proof}

We record an elementary linear-algebra inequality, used in the generalized linear case of the validity argument below.

\begin{lemma}[Trace bound for a quadratic form]
\label{lem:trace-quadratic-bound}
For any $A \succeq 0$, $C \succeq 0$, and vector $u$,
\[
    u^\top C u \leq \tr\del{A^{-1} C} \cdot u^\top A u,
\]
where $\tr\del{A^{-1} C}$ follows the convention from the beginning of this subsection.
\end{lemma}
\begin{proof}
We split on whether $\mathrm{range}(C) \subseteq \mathrm{range}(A)$.
    \begin{itemize}
        \item If $\mathrm{range}(C) \not\subseteq \mathrm{range}(A)$, then $\tr\del{A^{-1} C} = +\infty$ by convention and the bound is immediate: the right-hand side is $+\infty$ when $u^\top A u > 0$, and is $0 \cdot (+\infty) = +\infty$ when $u^\top A u = 0$.
        \item If $\mathrm{range}(C) \subseteq \mathrm{range}(A)$, denote $M := A^{\dagger/2} C A^{\dagger/2} \succeq 0$ with $A^{\dagger/2} := (A^\dagger)^{1/2}$. Then $u^\top C u = (A^{1/2} u)^\top M (A^{1/2} u) \leq \norm{M}_{\mathrm{op}}\, u^\top A u$, and $\norm{M}_{\mathrm{op}} = \lambda_{\max}(M) \leq \tr(M) = \tr\del{A^{-1} C}$, where the first equality uses $A^{1/2} M A^{1/2} = C$ (valid since $\mathrm{range}(C) \subseteq \mathrm{range}(A)$), the inequality holds because $M \succeq 0$ (its largest eigenvalue is at most the sum of all eigenvalues), and the last equality is the cyclic property of the trace. \qedhere
    \end{itemize}
\end{proof}

\subsection{Relaxed DOEC Bounds for the Running Examples}
\label{sec:reg-running-examples}

Having verified that the relaxed coverages are valid (Lemma~\ref{lem:relaxed-coverage-valid}), we now bound the induced relaxed DOEC in each running example, and in doing so justify that \generalfalcon can implement its per-step exploitative F-design (line~\ref{line:doec}, Eq.~\eqref{eqn:monster-distn}) by solving a convex program. The key structural fact is that each relaxed coverage is the directional derivative of a concave \emph{barrier}, which both turns Eq.~\eqref{eqn:monster-distn} into a barrier-regularized reward maximization and produces the certified bound $\overline{V}$ through a first-order optimality argument.

\begin{lemma}[Relaxed DOEC bounds for the running examples]
\label{lem:relaxed-doec-bound}
Fix a context $x$, abbreviate $\Gcal = \Fcal_x$, and let $\hat{g} \in \Gcal$, $\gamma > 0$, and $\varepsilon \geq 0$. In each of the three running examples of Lemma~\ref{lem:relaxed-coverage-valid}, the relaxed exploitative F-design (Eq.~\eqref{eqn:monster-distn}) is solved by the same $p^*$ as the concave maximization
\[
p^* \in \argmax_{p \in \co(\Lambda)} \cbr{\EE_{a \sim p}\sbr{\hat{g}(a)} + \tfrac{1}{\gamma} B(p)},
\qquad
B(p) =
\begin{cases}
\sum_{a \in \Acal} \log p(a), & \text{discrete},\\[2pt]
\kappa^2 \log\det \Sigma_p, & \text{generalized linear},\\[2pt]
\tfrac{1}{h}\EE_{a \sim \mu}\sbr{\log p(a)}, & h\text{-smoothed},
\end{cases}
\]
where $B$ is the concave barrier whose directional derivative recovers the relaxed coverage, $\overline{\coverage}_\varepsilon(p, \lambda; \Gcal) = \inner{\nabla B(p)}{\lambda}$.

Furthermore, $\overline{\doec}_{\gamma, \varepsilon}(\hat{g}, \Gcal, \Lambda) = \overline{V}$,
with $\overline{V} = \frac{|\Acal|}{\gamma},\ \frac{\kappa^2 d}{\gamma},\ \frac{1}{\gamma h}$ in the discrete, generalized linear, and $h$-smoothed settings, respectively. See Table~\ref{tab:relaxed-coverage} for a summary.
\end{lemma}

\begin{proof}
    For the purpose of uniformity, we assume that $\Acal$ is finite in the $h$-smoothed setting.
Write $\Phi(p, \lambda) := \EE_{a \sim \lambda}\sbr{\hat{g}(a)} - \EE_{a \sim p}\sbr{\hat{g}(a)} + \frac{1}{\gamma}\overline{\coverage}_\varepsilon(p, \lambda; \Gcal)$ in the discrete and generalized linear settings, and $\Phi(p, \lambda) := \EE_{a \sim \mu}\sbr{\lambda(a) \hat{g}(a)} - \EE_{a \sim \mu}\sbr{p(a) \hat{g}(a)} + \frac{1}{\gamma}\overline{\coverage}_\varepsilon(p, \lambda; \Gcal)$ in the $h$-smoothed setting, so that, by Definition~\ref{def:doec}, $\overline{\doec}_{\gamma, \varepsilon}(\hat{g}, \Gcal, \Lambda) = \min_{p \in \co(\Lambda)} \max_{\lambda \in \Lambda} \Phi(p, \lambda)$. It therefore suffices to find one $p^* \in \co(\Lambda)$ with $\Phi(p^*, \lambda) \leq \overline{V}$ for all $\lambda \in \Lambda$.

\paragraph{A unified barrier.} In each setting the relaxed coverage is the directional derivative of a concave barrier $B$, in the sense that
\[
\overline{\coverage}_\varepsilon(p, \lambda; \Gcal) = \inner{\nabla B(p)}{\lambda},
\qquad\text{and } \inner{\nabla B(p)}{p} = \gamma \overline{V} \text{ is constant in } p,
\]
where $\inner{\cdot}{\cdot}$ is the Euclidean inner product over $\Acal$.

Recall that we define $p^*$ to be the solution of $\max_{p \in \Lambda} G(p)$ with $G(p) := \EE_{a \sim p}\sbr{\hat{g}(a)} + \frac{1}{\gamma} B(p)$ in discrete action and generalized linear settings, and $G(p) := \EE_{a \sim \mu}\sbr{p(a) \hat{g}(a)} + \frac{1}{\gamma} B(p)$ in the $h$-smoothed setting. Since $B$ is concave, $G$ is concave, and thus $p^*$ is a global maximizer of $G$ over $\co(\Lambda)$.

\paragraph{Certification by first-order optimality.}
For the discrete action and generalized linear settings, by the first-order optimality condition of the above convex optimization problem, $\inner{\nabla G(p^*)}{\lambda - p^*} \leq 0$ for all $\lambda \in \Lambda$. Substituting $\nabla G(p^*) = \hat{g} + \frac{1}{\gamma}\nabla B(p^*)$
and rearranging,
\[
\Phi(p^*, \lambda)
= \inner{\hat{g}}{\lambda - p^*} + \frac{1}{\gamma}\inner{\nabla B(p^*)}{\lambda}
\leq \frac{1}{\gamma}\inner{\nabla B(p^*)}{p^*}
= \overline{V}
\qquad\text{for all } \lambda \in \Lambda,
\]
using $\overline{\coverage}_\varepsilon(p^*, \lambda; \Gcal) = \inner{\nabla B(p^*)}{\lambda}$ in the first equality and the constant diagonal value in the last. Thus, $\overline{\doec}_{\gamma, \varepsilon}(\hat{g}, \Gcal, \Lambda) = \min_{p} \max_{\lambda} \Phi(p, \lambda) \leq \max_{\lambda} \Phi(p^*, \lambda) \leq \overline{V}$.

Furthermore, for all $p$, note that $\Phi(p, \lambda)$ is linear in $\lambda$, therefore,
\[
\max_{\lambda \in \Lambda} \Phi(p, \lambda)
=
\max_{\lambda \in \co(\Lambda)} \Phi(p, \lambda)
\geq \Phi(p, p) = \frac{1}{\gamma}\inner{\nabla B(p)}{p} = \overline{V}.
\]
This implies that
$\overline{\doec}_{\gamma, \varepsilon}(\hat{g}, \Gcal, \Lambda) \geq \overline{V}$.

For the $h$-smoothed setting, the same argument applies, except that $\hat{g}$ is replaced by $\hat{g}\mu$ (element-wise product) and $\nabla B(p^*) = \hat{g}\mu + \frac{1}{\gamma}\nabla B(p^*)$. Therefore, $\Phi(p^*, \lambda) = \inner{\hat{g}\mu}{\lambda - p^*} + \frac{1}{\gamma}\inner{\nabla B(p^*)}{\lambda} \leq \frac{1}{\gamma}\inner{\nabla B(p^*)}{p^*} = \overline{V}$ for all $\lambda \in \Lambda$ and $\max_{\lambda \in \Lambda} \Phi(p, \lambda) \geq \Phi(p, p) = \frac{1}{\gamma}\inner{\nabla B(p)}{p} = \overline{V}$ for all $p$, so the same conclusion holds.

In summary,
\[
\overline{\doec}_{\gamma, \varepsilon}(\hat{g}, \Gcal, \Lambda) = \overline{V}.
\]

\paragraph{Instantiating the barrier.} The three settings differ only in $B$ and the constant $\inner{\nabla B(p)}{p}$:
\begin{itemize}
    \item \textbf{Discrete} ($\Lambda = \Delta(\Acal)$): $B(p) = \sum_{a} \log p(a)$, so $\inner{\nabla B(p)}{\lambda} = \sum_{a} \frac{\lambda(a)}{p(a)}$ and $\inner{\nabla B(p)}{p} = \sum_{a} \frac{p(a)}{p(a)} = |\Acal|$.
    \item \textbf{Generalized linear} ($\Lambda = \Delta(\Acal)$): $B(p) = \kappa^2 \log\det \Sigma_p$, so $\inner{\nabla B(p)}{\lambda} = \kappa^2 \tr\del{\Sigma_p^{-1} \Sigma_\lambda}$ (using $\frac{\partial}{\partial p(a)} \log\det \Sigma_p = \phi(a)^\top \Sigma_p^{-1} \phi(a)$) and $\inner{\nabla B(p)}{p} = \kappa^2 \tr\del{\Sigma_p^{-1} \Sigma_p} = \kappa^2 d$.
    \item \textbf{$h$-smoothed} ($\Lambda = \Delta_h^\mu(\Acal)$): $B(p) = \frac{1}{h}\EE_{a \sim \mu}\sbr{\log p(a)}$, so $\inner{\nabla B(p)}{\lambda} = \frac{1}{h}\EE_{a \sim \mu}\sbr{\frac{\lambda(a)}{p(a)}}$ and $\inner{\nabla B(p)}{p} = \frac{1}{h}\EE_{a \sim \mu}\sbr{\frac{p(a)}{p(a)}} = \frac{1}{h}$.
\end{itemize}
This gives $ \overline{V} = \frac{|\Acal|}{\gamma},\ \frac{\kappa^2 d}{\gamma},\ \frac{1}{\gamma h}$, i.e.\ the stated bounds.
\end{proof}

Based on the above lemma, we can derive the closed-form solutions of the relaxed exploitative F-design in the discrete and $h$-smoothed settings.
In the discrete action setting, this recovers a version of the inverse gap weighting~\citep{abe1999associative,foster20beyond,simchi2022bypassing} exploration rule.
Interestingly, in the $h$-smoothed regret setting, we obtain a new exploration rule which is different from previous work~\citep{zhu2022contextuala}.
Our exploration rule works for reduction to both offline and online regression, while~\citet{zhu2022contextuala}'s exploration rule is only known to work for online regression.

\begin{corollary}[Closed-form F-design via inverse gap weighting]
\label{corol:igw-closed-form}
In the discrete and $h$-smoothed settings of Lemma~\ref{lem:relaxed-doec-bound}, the certifying distribution $p^*$ admits a closed form. Writing $\Delta_a := \max_{a' \in \Acal}\hat{g}(a') - \hat{g}(a)$ for the estimated suboptimality gap of action $a$,
\[
p^*(a) = \frac{1}{\nu + \gamma \Delta_a} \quad (\text{discrete}),
\qquad
p^*(a) = \frac{1}{h \vee (\nu + \gamma h \Delta_a)} \quad (h\text{-smoothed}),
\]
where in each case the multiplier $\nu$ is chosen so that $p^*$ is normalized.
\end{corollary}

\begin{proof}[Proof sketch]
By Lemma~\ref{lem:relaxed-doec-bound}, $p^*$ maximizes the concave objective $G(p) = \EE_{a \sim p}\sbr{\hat{g}(a)} + \frac{1}{\gamma} B(p)$ over $\Lambda$, with barrier $B(p) = \sum_{a} \log p(a)$ in the discrete setting and $B(p) = \frac{1}{h}\EE_{a \sim \mu}\sbr{\log p(a)}$ in the $h$-smoothed setting. In the discrete setting, stationarity of $G$ on the simplex gives a multiplier $\nu$ for the constraint $\sum_{a} p(a) = 1$ with $\hat{g}(a) + \frac{1}{\gamma p^*(a)} = \nu$, equivalently $p^*(a) = \frac{1}{\nu + \gamma \Delta_a}$ after absorbing $\max_{a'}\hat{g}(a')$ into $\nu$. In the $h$-smoothed setting the additional box constraint $p(a) \leq \frac{1}{h}$ of $\Delta_h^\mu(\Acal)$ enters the KKT conditions and caps the solution at $p^*(a) = \frac{1}{h \vee (\nu + \gamma h \Delta_a)}$; in both cases $\nu$ is set by normalization.
\end{proof}

\begin{table}[t]
\centering
\footnotesize
\renewcommand{\arraystretch}{1.4}
\begin{tabular}{@{}lp{0.8\textwidth}@{}}
\toprule
Discrete
&
\begin{tabular}[t]{@{}p{0.26\textwidth}p{0.52\textwidth}@{}}
Relaxed coverage & $\overline{\coverage}_{\varepsilon}(p, \lambda; \Gcal) = \sum_{a \in \Acal} \frac{\lambda(a)}{p(a)}$ \\
Equivalent convex program & $\max_{p \in \Delta(\Acal)}\ \EE_{a \sim p}\sbr{\hat{g}(a)} + \frac{1}{\gamma} \sum_{a \in \Acal} \log p(a)$ \\
Solution & $p^*(a) = \frac{1}{\nu + \gamma \Delta_a}$, $\nu$ is such that $\sum_{a \in \Acal} p^*(a) = 1$ \\
Certified DOEC bound $\overline{V}$ & $\frac{|\Acal|}{\gamma}$
\end{tabular}
\\
\midrule
Generalized linear
&
\begin{tabular}[t]{@{}p{0.26\textwidth}p{0.52\textwidth}@{}}
Relaxed coverage & $\overline{\coverage}_{\varepsilon}(p, \lambda; \Gcal) = \kappa^2 \tr\del{\Sigma_p^{-1} \Sigma_\lambda}$ \\
Equivalent convex program & $\max_{p \in \Delta(\Acal)}\ \EE_{a \sim p}\sbr{\hat{g}(a)} + \frac{\kappa^2}{\gamma} \log\det\del{\Sigma_p}$ \\
Solution & no closed form (solution of the convex program above) \\
Certified DOEC bound $\overline{V}$ & $\frac{\kappa^2 d}{\gamma}$
\end{tabular}
\\
\midrule
$h$-smoothed
&
\begin{tabular}[t]{@{}p{0.26\textwidth}p{0.52\textwidth}@{}}
Relaxed coverage & $\overline{\coverage}_{\varepsilon}(p, \lambda; \Gcal) = \frac{1}{h} \EE_{a \sim \mu}\sbr{\frac{\lambda(a)}{p(a)}}$ \\
Equivalent convex program & $\max_{p \in \Delta_h^\mu(\Acal)}\ \EE_{a \sim \mu}\sbr{p(a) \hat{g}(a)} + \frac{1}{\gamma h} \EE_{a \sim \mu}\sbr{\log p(a)}$ \\
Solution & $p^*(a) = \frac{1}{h \vee (\nu + \gamma h \Delta_a)}$, $\nu$ is such that $\EE_{a \sim \mu}\sbr{p^*(a)} = 1$ \\
Certified DOEC bound $\overline{V}$ & $\frac{1}{\gamma h}$
\end{tabular}
\\
\bottomrule
\end{tabular}
\caption{Relaxed coverages for the running examples and the induced relaxed exploitative F-design subproblems; none of them depends on the cushion parameter $\varepsilon$. Here, $\hat{g}$ is the reward estimate, $\Delta_a := \max_{a' \in \Acal} \hat{g}(a') - \hat{g}(a)$ is the estimated suboptimality gap of action $a$, $\Sigma_p := \sum_{a \in \Acal} p(a) \phi(a) \phi(a)^\top$, and $\nu$ ensures that the solution distributions are properly normalized. In the $h$-smoothed setting, $p(a)$ and $\lambda(a)$ denote densities with respect to the base measure $\mu$. The settings and constants $\underline{L}, \overline{L}, \kappa$ are as in Lemma~\ref{lem:relaxed-coverage-valid}.}
\label{tab:relaxed-coverage}
\end{table}

\section{Proofs for Section~\ref{sec:offline-oracle-efficient-algorithm} Part 3: Extensions}
\label{sec:extensions}

In this section, we discuss several extensions of our main regret analysis when using ERM as the offline regression oracle in \Cref{alg:general-falcon}.
We will demonstrate how to instantiate our regret bound to handle model misspecification, corruption-robustness, and context-distribution shift settings.
The key idea is to modify the offline regression oracle guarantees used in our per-context regret analysis to accommodate these different settings.
Specifically, in the realizable, finite $\Fcal$ setting, the expected squared error $\EE_{x \sim \Dcal_X}[\WL_{m-1}(\hat{f}_m, x)]$ is bounded by $\lesssim \frac{\log(|\Fcal|T/ \delta)}{\tau_{m-1}-\tau_{m-2}}$ using standard concentration inequalities.
However, when the \emph{realizability} assumption does not hold, for example, in the model misspecification setting, the ERM will not give us a logarithmic squared error bound since the true model $f^*$ may not be in the function class $\Fcal$.
Thankfully, we use a standard lemma~\citep[e.g.,][Lemma 3]{zhu2022efficient} to bound the expected squared error by empirical excess square error plus an additional logarithmic term $\iupbound{\log(\abs{\Fcal}) + \log(T/\delta)}$.

\subsection{Model Misspecification}
\label{app:misspecification}

In practice, our function class $\Fcal$ may not perfectly capture the true reward function.
To quantify the impact of such model misspecification on our regret analysis, we adopt the concept of uniform misspecification level~\citep{lattimore2020learning}, which measures the worst-case deviation between the best in-class approximation and the true reward function over all context-action pairs:

\begin{assum}[Universal misspecification level]
    \label{assum:misspecification-induced-by-policies}
    There exist a function $\widetilde{f} \in \Fcal$ and a constant $B \geq 0$ such that
    \[
        \sup_{x \in \Xcal,\, a \in \Acal} \abs{ \widetilde{f}(x, a) - f^*(x, a) } \leq \sqrt{B}.
    \]
\end{assum}

\Cref{assum:misspecification-induced-by-policies} states that some in-class function $\widetilde{f} \in \Fcal$ uniformly approximates $f^*$ to within $\sqrt{B}$ over all contexts and actions. The special case of $B = 0$ boils down to the realizability assumption.
Based on this assumption, we extend our main regret analysis to account for the misspecification error. The only place realizability enters our analysis is the per-context off-policy evaluation (OPE) lemma (\Cref{lemma:concentrate-reward-in-coverage-per-context}), whose proof pairs $\hat{f}_m$ with the then-in-class $f^*$ inside the coverage supremum. The following misspecified version restores it by pairing $\hat{f}_m$ with the in-class surrogate $\widetilde{f}$ and paying the misspecification level $\sqrt{B}$.

\begin{lemma}[Per-context OPE under misspecification]
    \label{lemma:concentrate-reward-misspecified}
    Under \Cref{assum:misspecification-induced-by-policies}, on the same probability-$(1-\delta)$ event as \Cref{lemma:concentrate-reward-in-coverage-per-context}, for all $x \in \Xcal$, $m \in [2, M]$, and $\lambda \in \Delta(\Acal)$,
    \[
        \abs{\HRcal_m(\lambda \mid x) - \Rcal_{\widetilde{f}}(\lambda \mid x)}
        \leq \sqrt{\coverage_{\varepsilon_{m-1}}(\pi_{m-1}, \lambda; \Fcal_x \mid x)\del{2\WL_{m-1}(\hat{f}_m, x) + 2B + \varepsilon_{m-1}}} + \sqrt{B} .
    \]
\end{lemma}

\begin{proof}
    Let $\widetilde{f} \in \Fcal$ be the surrogate of \Cref{assum:misspecification-induced-by-policies} and write $\Rcal_{\widetilde{f}}(\lambda \mid x) := \EE_{a \sim \lambda}\sbr{\widetilde{f}(x, a)}$. Since $\hat{f}_m, \widetilde{f} \in \Fcal$, the pair $(\hat{f}_m, \widetilde{f})$ appears in the supremum defining $\coverage$, so repeating the proof of \Cref{lemma:concentrate-reward-in-coverage-per-context} with $\widetilde{f}$ in place of $f^*$ gives
    \[
        \abs{\HRcal_m(\lambda \mid x) - \Rcal_{\widetilde{f}}(\lambda \mid x)}
        \leq \sqrt{\coverage_{\varepsilon_{m-1}}(\pi_{m-1}(\cdot \mid x), \lambda; \Fcal \mid x)\del{\EE_{a \sim \pi_{m-1}}\sbr{(\hat{f}_m(x, a) - \widetilde{f}(x, a))^2} + \varepsilon_{m-1}}} .
    \]
    By $(a+b)^2 \leq 2a^2 + 2b^2$ and the misspecification assumption, we have
    \begin{align*}
        \EE_{a \sim \pi_{m-1}}\sbr{(\hat{f}_m(x, a) - \widetilde{f}(x, a))^2}
        &\leq 2\,\EE_{a \sim \pi_{m-1}}\sbr{(\hat{f}_m(x, a) - f^*(x, a))^2} + 2\,\EE_{a \sim \pi_{m-1}}\sbr{(\widetilde{f}(x, a) - f^*(x, a))^2}
        \\
        &\leq 2\WL_{m-1}(\hat{f}_m, x) + 2B.
    \end{align*}
    Finally, the bias from replacing $\widetilde{f}$ by $f^*$ is controlled by the universal level: $\abs{\Rcal_{\widetilde{f}}(\lambda \mid x) - \Rcal(\lambda \mid x)} = \abs{\EE_{a \sim \lambda}\sbr{\widetilde{f} - f^*}} \leq \sqrt{\EE_{a \sim \lambda}\sbr{(\widetilde{f} - f^*)^2}} \leq \sqrt{B}$. The claim follows by the triangle inequality $\abs{\HRcal_m - \Rcal} \leq \abs{\HRcal_m - \Rcal_{\widetilde{f}}} + \abs{\Rcal_{\widetilde{f}} - \Rcal}$.
\end{proof}

Using Lemma~\ref{lemma:concentrate-reward-misspecified} in place of Lemma~\ref{lemma:concentrate-reward-in-coverage-per-context}, the per-context regret bound (\Cref{lemma:general-falcon-per-context-regret-bound}) goes through verbatim with $\WL_{m-1}(\hat{f}_m, x)$ replaced by $2\WL_{m-1}(\hat{f}_m, x) + 2B$ and an additional additive $\sqrt{B}$ per step;
both only inflate constants and the $B$-dependent term, which is absorbed into the bound below.

\begin{corollary}[Regret Bound under Misspecification]
    \label{corol:miss-total-regret-bound}
    If \Cref{assum:misspecification-induced-by-policies,assum:bounded-doec} hold, with probability at least $1 - \delta$, under the doubling schedule (\Cref{def:doubling}) but changing $\gamma_m = \sqrt{\frac{ D}{ B + \log(|\Fcal|T/ \delta)/\tau_{m-1}}}$, the total regret of \Cref{alg:general-falcon} satisfies
    \[
        \sum_{t=1}^T \EE_{x_t \sim \Dcal_X} \sbr{ \REG(\pi_{m(t)} \mid x_t) }
        =
        \upboundlog{\sqrt{T D \del{\log(\abs{\Fcal}/\delta)}} + T\sqrt{B D}}
    \]
\end{corollary}

The above corollary is directly derived from our per-context regret bound in Lemma~\ref{lemma:general-falcon-per-context-regret-bound} by incorporating the misspecification error term, which handles the estimation error term $\WL$ in a manner similar to the realizable case but has an additional error term $B$ to account for the misspecification error.

\begin{proof}[Proof of Corollary~\ref{corol:miss-total-regret-bound}]
    First, we want to show how the misspecification affects the model estimation error.
    Denote $\hat{f}_m$ as the best ERM predictor returned by the offline regression oracle trained on the dataset collected in epoch $m-1$, and let $\widetilde{f} \in \Fcal$ be the universal surrogate of \Cref{assum:misspecification-induced-by-policies}, i.e.\ $\sup_{x, a}\abs{\widetilde{f}(x, a) - f^*(x, a)} \leq \sqrt{B}$.

    \begin{align*}
        &\quad \sum_{t=\tau_{m-2}+1}^{\tau_{m-1}}
            \EE_{x_t \sim \Dcal_X, a_t \sim \pi_{m-1}}[\WL_{m-1}(\hat{f}_m, x_t)]
        \\
        &\lesssim
        \sum_{t=\tau_{m-2}+1}^{\tau_{m-1}} \del{\hat{f}_m(x_t, a_t) - r_t}^2 - \del{f^*(x_t, a_t) - r_t}^2 + \log(|\Fcal|T/ \delta)
            \tag{By the second inequality of~\citet[][Lemma 3]{zhu2022efficient}
            }
        \\
        &\lesssim
        \sum_{t=\tau_{m-2}+1}^{\tau_{m-1}} \del{\widetilde{f}(x_t, a_t) - r_t}^2 - \del{f^*(x_t, a_t) - r_t}^2 + \log(|\Fcal|T/ \delta)
            \tag{$\hat{f}_m$ is the ERM predictor}
        \\
        &\lesssim
        \sum_{t=\tau_{m-2}+1}^{\tau_{m-1}}
            \EE_{x_t \sim \Dcal_X, a_t \sim \pi_{m-1}}\sbr{\del{\widetilde{f}(x_t, a_t) - f^*(x_t, a_t)}^2} + 2 \log(|\Fcal|T/ \delta)
            \tag{By the first inequality of~\citet[][Lemma 3]{zhu2022efficient}
            }
        \\
        &\lesssim
        (\tau_{m-1} - \tau_{m-2}) B + \log(|\Fcal|T/ \delta)
            \tag{By the \Cref{assum:misspecification-induced-by-policies}}
        \\
            \label{eqn:miss-reg-bound}
    \end{align*}
    In the first inequality, we apply the second inequality of~\citet[][Lemma 3]{zhu2022efficient} to bound the expected squared error $\EE_{x \sim \Dcal_X}[\WL_{m-1}(\hat{f}_m, x)]$. We note that although the original lemma was stated for the realizable case, it continues to hold in the misspecified setting.
    In the last inequality, we apply the misspecification assumption.

    Next, we convert the regret in the ground-truth reward function $f^*$ to the best-in-class reward function $\widetilde{f}$ that is close to $f^*$ by sacrificing $2\sqrt{B}$. For any $x \in \Xcal$, $p \in \co(\Lambda)$
    \begin{align*}
        \REG(p \mid x) &= R(\lambda^* \mid x) - R(p \mid x)
        \\
        &=
        \del{R_{\widetilde{f}}(\lambda^* \mid x) - R_{\widetilde{f}}(p \mid x)}
        + \del{R_{\widetilde{f}}(p \mid x) - R(p \mid x)} + \del{R(\lambda^* \mid x) - R_{\widetilde{f}}(\lambda^* \mid x)}
        \\
        &\leq
        \del{R_{\widetilde{f}}(\lambda^* \mid x) - R_{\widetilde{f}}(p \mid x)} + 2\sqrt{B},
            \tag{By the misspecification assumption}
        \\
        &\leq
        \del{R_{\widetilde{f}}(\widetilde{\lambda} \mid x) - R_{\widetilde{f}}(p \mid x)} + 2\sqrt{B},
            \tag{$\widetilde{\lambda}:= \argmax_{\lambda \in \Lambda}R_{\widetilde{f}}(\lambda \mid x)$}
    \end{align*}
    and we denote $R_{\widetilde{f}}(\widetilde{\lambda} \mid x) - R_{\widetilde{f}}(p \mid x) =: \widetilde{\REG}(p \mid x)$.
    Then, with the same reasoning as in Lemma~\ref{lemma:general-falcon-per-context-regret-bound}, we can bound $\sum_{t=1}^T \EE_{x_t \sim \Dcal_X} \sbr{ \widetilde{\REG}(\pi_{m(t)} \mid x_t) }$ similarly to get the total regret bound as follows:

    \begin{align*}
        &\quad \sum_{t=1}^T \EE_{x_t \sim \Dcal_X} \sbr{ \REG(\pi_{m(t)} \mid x_t) }
        \leq
        \sum_{t=1}^T \EE_{x_t \sim \Dcal_X} \sbr{ \widetilde{\REG}(\pi_{m(t)} \mid x_t) } + 2T\sqrt{B}
        \\
        &\leq
        \upboundlog{\tau_1 +
        \max_{m \in \cbr{2,\ldots,M}} \frac{\tau_{m}}{\gamma_{m}}\del{D + \sum_{s=2}^{M} \gamma_{s}^2 \EE_{x \sim \Dcal_X}\sbr{ \WL_{s-1}(\hat{f}_s, x) }}} + 2T\sqrt{B}
            \tag{From Lemma~\ref{lemma:general-falcon-per-context-regret-bound}}
        \\
        &\leq
        \upboundlog{\frac{\tau_{M}}{\gamma_{M}} \del{D\text{polylog}(\frac{1}{\varepsilon}) + D }} + 2T\sqrt{B}
            \tag{$\gamma_m$ and $\tau_m/\gamma_m$ is non-decreasing in $m$}
        \\
        &\leq
        \upboundlog{\sqrt{T D \del{\log(\abs{\Fcal}/\delta)}} + T\sqrt{B D}}
            \tag{$\tau_M/\gamma_M = \iupboundlog{T\sqrt{B/D} + \sqrt{T ( \log(|\Fcal|T/ \delta))/D}}$}
    \end{align*}
\end{proof}

\subsection{Corruption-Robustness}
\label{app:corruption}
In practice, the observed rewards may be corrupted by an adversary, which can significantly affect the performance of contextual bandit algorithms.
We extend our per-context regret analysis to a corruption-robust setting, where the observed rewards may be corrupted by an adversary.
Specifically, we consider the following corruption model:

\begin{assum}[Corruption Model]
    \label{assum:corruption-model}
    The generative process of $(x_t,a_t,r_t)$'s is:
    first, the adversary specifies rounds for corruptions $\Ccal \subset [T]$, where $|\Ccal| \leq C$.
    Then for the sequence $(\tilde{x}_t, (\tilde{R}_t(a))_{a\in\Acal}) \sim D$, when $t \in \Ccal$, the $(x_t, (R_t(a))_{a\in\Acal})$ can be arbitrarily chosen by the adversary; otherwise, it must equal $(\tilde{x}_t, (\tilde{R}_t(a))_{a\in\Acal})$.
    Then the learner chooses $a_t$ and observes $r_t = R_t(a_t)$.
\end{assum}

Under \Cref{assum:corruption-model}, we extend our per-context regret analysis to account for the adversarial corruption in the observed rewards.

\begin{corollary}[Regret Bound under Corruption]
    \label{corol:corruption-total-regret-bound}
    If \Cref{assum:corruption-model,assum:bounded-doec} hold, then with probability at least $1 - \delta$, under the doubling schedule (\Cref{def:doubling}) but changing $\gamma_m = \sqrt{\frac{D \tau_{m-1}}{C + \log(|\Fcal|T / \delta)}}$, the total regret of \Cref{alg:general-falcon} satisfies
    \[
        \sum_{t=1}^T \EE_{x_t \sim \Dcal_X} \sbr{ \REG(\pi_{m(t)} \mid x_t) }
        =
        \upboundlog{\sqrt{T D \del{C + \log(\abs{\Fcal})}}}
    \]
\end{corollary}

\begin{proof}[Proof of Corollary~\ref{corol:corruption-total-regret-bound}]
    We follow the proof of Lemma~\ref{lemma:general-falcon-per-context-regret-bound} but modify the application of the offline regression oracle guarantee to account for the corrupted feedback.
    Denote $\hat{f}_m$ as the best ERM predictor returned by the offline regression oracle at the beginning of epoch $m$ over the corrupted data. Denote $\widetilde{f}_m$ as the best ERM predictor returned over the uncorrupted data.
    Denote the dataset collected in epoch $m-1$ as $\Scal_{m-1} = \cbr{(x_t, a_t, r_t)}_{t=\tau_{m-2}+1}^{\tau_{m-1}}$, and the uncorrupted subset as $\widetilde{\Scal}_{m-1} = \cbr{(\tilde{x}_t, \tilde{a}_t, \tilde{r}_t)}_{t=\tau_{m-2}+1}^{\tau_{m-1}}$.
    For every tuple $(x_t, a_t, r_t) \in \Scal_{m-1}$, there exists a corresponding uncorrupted tuple $(\tilde{x}_t, \tilde{a}_t, \tilde{r}_t)$. For every function $f \in \Fcal$, we have
    \[
        \sum_{t=\tau_{m-2}+1}^{\tau_{m-1}} \del{f(x_t, a_t) - r_t}^2
        \leq
        \sum_{t=\tau_{m-2}+1}^{\tau_{m-1}} \del{f(\tilde{x}_t, \tilde{a}_t) - \tilde{r}_t}^2 + C.
    \]
    Then, we bound the expected squared error under the corrupted data distribution as follows:

    \begin{align*}
        &\quad \sum_{t=\tau_{m-2}+1}^{\tau_{m-1}}
            \EE_{x_t \sim \Dcal_X, a_t \sim \pi_m}[\WL_{m-1}(\hat{f}_m, x_t)]
        \\
        &\leq
        2 \sum_{t=\tau_{m-2}+1}^{\tau_{m-1}} \del{\hat{f}_m(x_t, a_t) - r_t}^2 - \del{f^*(x_t, a_t) - r_t}^2 + \log(|\Fcal|T/ \delta)
            \tag{By the second inequality of \citet[][Lemma 3]{zhu2022efficient}}
        \\
        &\leq
        2 \sum_{t=\tau_{m-2}+1}^{\tau_{m-1}} \del{\widetilde{f}_m(x_t, a_t) - r_t}^2 - \del{f^*(x_t, a_t) - r_t}^2 + \log(|\Fcal|T/ \delta)
            \tag{$\hat{f}_m$ is the ERM predictor on corrupted data}
        \\
        &\leq
        2 \sum_{t=\tau_{m-2}+1}^{\tau_{m-1}} \del{\widetilde{f}_m(\tilde{x}_t, \tilde{a}_t) - \tilde{r}_t}^2 - \del{f^*(\tilde{x}_t, \tilde{a}_t) - \tilde{r}_t}^2 + 4C + \log(|\Fcal|T/ \delta)
            \tag{By corruption definition}
        \\
        &=
        4C + \log(|\Fcal|T/ \delta)
            \tag{$\widetilde{f}_m$ is the ERM predictor over uncorrupted data in the realizable setting}
    \end{align*}
    Then, we plug this bound into the per-context regret bound in Lemma~\ref{lemma:general-falcon-per-context-regret-bound} to get the total regret bound as follows:
    \begin{align*}
        &\quad \sum_{t=1}^T \EE_{x_t \sim \Dcal_X} \sbr{ \REG(\pi_{m(t) \mid x_t}) }
        \leq
        \upboundlog{
        \tau_1 +
        \max_{m \in \cbr{2,\ldots,M}} \frac{\tau_{m}}{\gamma_{m}}\del{D + \sum_{s=2}^{M} \gamma_{s}^2 \EE_{x \sim \Dcal_X}\sbr{ \WL_{s-1}(\hat{f}_s, x) }}}
            \tag{From Lemma~\ref{lemma:general-falcon-per-context-regret-bound}}
        \\
        &\leq
        \upboundlog{\frac{\tau_{M}}{\gamma_{M}} \del{D \text{polylog}(\frac{1}{\varepsilon}) + \tau_{m-1} D}}
            \tag{$\gamma_m$ and $\tau_m/\gamma_m$ is non-decreasing in $m$}
        \\
        &\leq
        \upboundlog{\sqrt{T D \del{C + \log(\abs{\Fcal})}}}
            \tag{$\tau_M/\gamma_M = \iupboundlog{\sqrt{T (C + \log(|\Fcal|T/ \delta))/D}}$}
    \end{align*}

\end{proof}

\subsection{Context-Distribution Shift}
\label{app:context-dist-shift}

Another practical challenge in contextual bandit problems is the potential shift in the context distribution over time.
Specifically, we made the following assumption:
\begin{assum}[Context-Distribution Shift]
    \label{assum:context-distribution-shift}
    Let context $x_t$ be drawn from context distributions $\Dcal_t$ at each time step $t$.
    There exists a constant $A \geq 1$ and a context distribution $\Dcal^*$, such that for all $t \in [T]$ and $x \in \Xcal$,
    \[
        A^{-1} \Dcal^*(x) \leq \Dcal_t(x) \leq A \Dcal^*(x).
    \]
\end{assum}

Under \Cref{assum:context-distribution-shift}, we extend our per-context regret analysis to accommodate the changing context distributions.
The key idea is to adjust the regression guarantees to account for the distribution shift, ensuring that our algorithm remains robust to these changes.
We give the following corollary:

\begin{corollary}[Regret Bound under Context Distribution Shift]
    \label{corol:shift-total-regret-bound}
    Under \Cref{assum:context-distribution-shift,assum:bounded-doec}, with probability at least $1 - \delta$, under the doubling schedule (\Cref{def:doubling}) but changing $\gamma_m = \isqrt{\frac{\tau_{m-1} AD}{\log(\abs{\Fcal}/\delta)}}$ %
    , the total regret of \Cref{alg:general-falcon} satisfies
    \[
        \sum_{t=1}^T \EE_{x_t \sim \Dcal_t} \sbr{ \REG(\pi_{m(t)} \mid x_t) }
        =
        \upboundlog{ \sqrt{A^3 T D \log\del{\frac{\abs{\Fcal}}{\delta}}}}
    \]
\end{corollary}

\begin{proof}[Proof of Corollary~\ref{corol:shift-total-regret-bound}]
    We follow the proof of Lemma~\ref{lemma:general-falcon-per-context-regret-bound} but modify the application of the offline regression oracle guarantee to account for the context-distribution shift.
    For each epoch $m$, denote $\hat{f}_m$ as the ERM predictor returned by the offline regression oracle trained on the dataset collected in epoch $m-1$ over the context distribution $\Dcal^*$.
    Denote $\widetilde{f}_m$ as the best ERM predictor trained on the dataset collected in epoch $m-1$ but under the original context distribution $\Dcal_t$.
    Then, by the context-distribution shift definition, we have
    \begin{align*}
        & \sum_{t=\tau_{m-2}+1}^{\tau_{m-1}}
            \EE_{x_t \sim \Dcal^*, a_t \sim \pi_m}[\WL_{m-1}(\hat{f}_m, x_t)]
        \leq
        A \sum_{t=\tau_{m-2}+1}^{\tau_{m-1}}
            \EE_{x_t \sim \Dcal_t, a_t \sim \pi_m}[\WL_{m-1}(\widetilde{f}_m, x_t)]
        \\
        \leq&
        A \log(|\Fcal|T/ \delta)
            \tag{By the offline regression oracle guarantee in the realizable setting}
    \end{align*}

    Therefore:

    \begin{align*}
        &\quad \sum_{t=1}^T \EE_{x_t \sim \Dcal^*} \sbr{ \REG(\pi_{m(t) \mid x_t}) }
        \leq
        \upboundlog{\tau_1 + \max_{m \in \cbr{2,\ldots,M}} \frac{\tau_{m}}{\gamma_{m}}\del{D + \sum_{s=2}^{M} \gamma_{s}^2 \EE_{x \sim \Dcal^*}\sbr{ \WL_{s-1}(\hat{f}_s, x) }}}
            \tag{From Lemma~\ref{lemma:general-falcon-per-context-regret-bound}}
        \\
        &\leq
        \upboundlog{\frac{\tau_{M} D}{\gamma_{M}} + \max_{m \in \cbr{2,\ldots,M}} \frac{\tau_{m}}{\gamma_{m}} \EE_{x \sim \Dcal^*}\sbr{ \sum_{s=2}^{M} \gamma_s^2 \WL_{s-1}(\hat{f}_s, x) }}
            \tag{$\gamma_m$ and $\tau_m/\gamma_m$ is non-decreasing in $m$
            }
        \\
        &\leq
        \upboundlog{\frac{\tau_{M} D}{\gamma_{M}} + A \gamma_{M} \log\del{\frac{\abs{\Fcal}}{\delta}}}
        \\
        &\leq
        \upboundlog{\sqrt{A T D\log\del{\frac{\abs{\Fcal}}{\delta}}}}
    \end{align*}

    Then, following the above derivation, we have
    \[
    \sum_{t=1}^T \EE_{x_t \sim \Dcal_t} \sbr{ \REG(\pi_{m(t) \mid x_t}) }
    \leq
    A \sum_{t=1}^T \EE_{x_t \sim \Dcal^*} \sbr{ \REG(\pi_{m(t) \mid x_t}) }
    \leq
     \upboundlog{\sqrt{A^3 T D\log\del{\frac{\abs{\Fcal}}{\delta}}}}
    \]

\end{proof}

\subsection{Context-dependent Benchmark Distribution Space}
\label{app:context-dependent-benchmark}
For some applications, each context $x$ may be associated with its own benchmark space of distributions $\Lambda_x$. For example, in LLM alignment applications~\citep{ouyang2022training}, we have a base policy $\pi_{\text{base}}$ and it is suitable to define $\Lambda_x^h = \cbr{\lambda: \frac{d \lambda(\cdot)}{d \pi_{\text{base}}(\cdot | x)} \leq \frac1h}$.
In this setting, our expected regret definition will change to
\[
\Regret_\Lambda(T, \alg)
=
\sum_{t=1}^T
\sbr{ \max_{\lambda \in \Lambda_{x_t}} \EE_{a \sim \lambda}\sbr{f^*(x_t, a)} - \EE_{a \sim p_t}\sbr{f^*(x_t, a)} }
\]

Our per-context regret analysis can be easily adapted to handle this context-dependent benchmark space setting.
The idea is to modify our per-context regret result to account for the context-dependent benchmark space.
We only need to replace the per-context regret lemma~\ref{lemma:general-falcon-per-context-regret-bound}
with the following lemma:
\begin{lemma}
    For any $\delta \in (0, 1)$, $\Lambda \subseteq \Delta(\Acal)$, with probability at least $1 - \delta$, given accessibility to any offline regression oracle, the $\Lambda$-Regret of \Cref{alg:general-falcon} for every context $x$ is bounded as
    \begin{align*}
        &\quad
        \sum_{t=1}^T \REG_{\Lambda_{x}}(\pi_{m(t)} \mid x)
        \\
        &\leq
        \upbound{ \tau_1 +
        \max_{m \in \cbr{2,\ldots,M}} \frac{\tau_{m}}{\gamma_m}
        \cdot
        \rbr{ \sum_{s=1}^{M} \gamma_s \doec_{\gamma_s, \varepsilon_s}(\Fcal_x, \Lambda_x) + \sum_{s=2}^M
        \gamma_{s}^2 \del{ \WL_{s-1}(\hat{f}_s, x) + \varepsilon_{s-1}
        }
        }},
    \end{align*}
\end{lemma}
The proof will be identical to that of Lemma~\ref{lemma:general-falcon-per-context-regret-bound}, with the only difference being that we replace the use of the fixed benchmark space $\Lambda$ with the context-dependent benchmark space $\Lambda_x$ throughout the proof.
Therefore, following similar steps as in Appendix~\ref{app:regret-analysis-offline}, we derive the total regret bound under the context-dependent benchmark space setting of \generalfalcon:
\begin{align*}
    &\quad \EE \sbr{\Regret(T,\generalfalcon)} \\
    &\leq
    \upboundlog{ \tau_1+
    \max_{m \in \cbr{2,\ldots,M}} \frac{\tau_{m}}{\gamma_m}
    \cdot
    \rbr{ \twolines{ \max_{m \in [M]} \gamma_m \EE\sbr{\doec_{\gamma_m, \varepsilon_m}(\Fcal_x, \Lambda_x)
    }}{
    +
    \max_{m \in \cbr{2,\ldots,M}} \gamma_{m}^2 \del{
    \offlinereg(\Fcal, \tau_{m-1}/2, \delta_m)
    + \varepsilon_{m-1}
    }}}
    }.
\end{align*}

\section{Proofs for Section~\ref{sec:structure-doec}}
\label{sec:coord-desc-converge}

\label{app:existence-of-extended-f-design}

In this section, we prove Theorem~\ref{thm:doec-leq-sec}: for any relaxed coverage satisfying Assumption~\ref{assum:admissible-relaxed-coverage} with step-size threshold $\bar{\Delta}$, \Cref{alg:extended-f-design} terminates within $\lfloor 1/\bar{\Delta} \rfloor$ iterations, and its output distribution $p^*$ certifies:
\begin{align}
    \overline{\doec}_{\gamma, \varepsilon}(\hat{g}, \Gcal, \Lambda)
    \leq
    \frac{10}{{\gamma}}\overline{\SEC}_{\varepsilon}(\Gcal, \Lambda)
        \label{eqn:small-doec}
\end{align}

This section is organized as follows. We first state the assumption that an admissible relaxed coverage should satisfy, then prove Theorem~\ref{thm:doec-leq-sec} via a potential-function argument, relying on two properties of the potential: a constant per-iteration decrease, proved in Appendix~\ref{sec:potential-function-constant-decreasing} using the dilution-stability of Assumption~\ref{assum:admissible-relaxed-coverage}, and a uniform lower bound, proved in Appendix~\ref{sec:potential-function-lower-boundedness} via the weighted sequential extrapolation bound Eq.~\eqref{eqn:weighted-sec-bound}, which follows from Lemma~\ref{lem:weighted-from-monotone}. Appendices~\ref{app:finite-eluder} and~\ref{app:smoothed-regret} bound the $\overline{\SEC}_{\varepsilon}$ in the finite Eluder dimension with trivial relaxation (Proposition~\ref{prop:sec-eluder}) and under the three running examples (Lemma~\ref{lem:relaxed-sec-running-examples}); Appendix~\ref{app:relaxed-sec-examples} verifies the admissible assumption for the trivial relaxation (Lemma~\ref{lem:trivial-relaxation-admissible}) and for the relaxed coverages with cushion parameter $\varepsilon$ of the running examples (Lemma~\ref{lem:relaxed-coverage-admissible}); Appendix~\ref{app:large-step-size} shows that a more aggressive step size yields faster termination when $\Lambda = \cbr{\delta_a : a \in \Acal}$ (Proposition~\ref{prop:cd-fast-converge}); and Appendix~\ref{sec:appendix-doec-sec-loose} proves Proposition~\ref{prop:doec-small-sec-large}, showing that the DOEC can be far smaller than the SEC. Throughout, we write $\Dcal\Gcal := \cbr{g - g' : g, g' \in \Gcal}$.

\subsubsection{Admissible relaxed coverage}

We state the admissibility assumption required of the relaxed coverage $\overline{\coverage}_{\varepsilon}$, distilled from the convergence analysis of \Cref{alg:extended-f-design} below.

\begin{assum}[Admissible relaxed coverage]
\label{assum:admissible-relaxed-coverage}
The relaxed coverage $\overline{\coverage}_{\varepsilon}$ (extended to take an unnormalized nonnegative measure as its first argument) satisfies the following four properties for every $\lambda \in \Lambda$ and $\varepsilon > 0$:
\begin{enumerate}[label=(\roman*), leftmargin=*]
    \item \textbf{(Monotonicity)} $\overline{\coverage}_{\varepsilon}(p, \lambda; \Gcal)$ is non-increasing in $\varepsilon$, and if $p \succeq q$, $\overline{\coverage}_{\varepsilon}(p, \lambda; \Gcal) \leq \overline{\coverage}_{\varepsilon}(q, \lambda; \Gcal)$;
    \item \textbf{(Homogeneity)} For any $c > 0$, $\overline{\coverage}_{c\varepsilon}(c\, p, \lambda; \Gcal) = \frac{1}{c}\, \overline{\coverage}_{\varepsilon}(p, \lambda; \Gcal)$;
    \item \textbf{(Continuity)} $\overline{\coverage}_{\varepsilon}(p, \lambda; \Gcal)$ is continuous in the covering measure $p$ with respect to the total variation distance;
    \item \textbf{(Dilution stability)} there exists a step-size threshold $\bar{\Delta} \in (0, 1]$ (that can depend on $\varepsilon$), for all nonnegative measures $p$ and $\lambda$, such that $\overline{\coverage}_{\varepsilon}(p + \bar{\Delta} \lambda, \lambda; \Gcal) \geq \frac{1}{2}\, \overline{\coverage}_{\varepsilon}(p, \lambda; \Gcal)$.
\end{enumerate}
\end{assum}

All four properties describe the behavior of the relaxed coverage: Monotonicity (i), homogeneity (ii), and continuity (iii) require the relaxed coverage to be well-behaved in the covering measure $p$ and the cushion parameter $\varepsilon$.
Dilution stability (iv) controls how fast the coverage of a direction $\lambda$ can decrease as mass is added to the covering measure. By monotonicity (i), adding $p$ with mass $\bar{\Delta}\lambda$ only \emph{dilutes} (i.e., decreases) the coverage $\overline{\coverage}_{\varepsilon}(p, \lambda; \Gcal)$ of that same direction; the property requires this dilution to be \emph{stable}, in the sense that one step of size $\bar{\Delta}$ shrinks the coverage by at most a factor of $2$. This is exactly the per-step guarantee that the coordinate descent step of \Cref{alg:extended-f-design} consumes when it adds mass $\Delta_t \lambda_t$ ($\Delta_t \leq \bar{\Delta}$) along a newly selected direction $\lambda_t$.

\subsection{Proof of Theorem~\ref{thm:doec-leq-sec}}
\doecleqsec*

\begin{proof}[Proof of \Cref{thm:doec-leq-sec}]
    The theorem statement consists of two parts: the termination of \Cref{alg:extended-f-design} and $p^*$ certifies Eq.~\eqref{eqn:small-doec}.

    \textbf{Termination of \Cref{alg:extended-f-design}:} We will use the following potential function in our analysis:

    \[
    \Phi_t = -
    \underbrace{\sum_{s=1}^t \Delta_s \overline{\coverage}_\varepsilon(p_{s},  \lambda_s; \Gcal)}_{F_t} + \gamma \underbrace{\sum_{s=1}^t \fr12 \Delta_s \widehat{R}(\lambda_s)}_{G_t \geq 0}
    \]

    We give three essential properties of the potential function $\Phi_t$ that help us analyze the termination.
    In the proof below, we also extend the definition of $\widehat{R}$ to any nonnegative measures that are not necessarily normalized. Specifically,
    \begin{equation}
    \widehat{R}(p) = \| p \|_1 \max_{\lambda' \in \Lambda}\EE_{a \sim \lambda'}[\hat{g}(a)] - \sum_{a \in \Acal} p(a) \hat{g}(a).
    \label{eqn:extended-r}
    \end{equation}
   With this notation, $G_t = \frac12 \widehat{R}(p_t)$, since $p_t = \sum_{s=1}^t \Delta_s \lambda_s$.
   We also recall that Algorithm~\ref{alg:extended-f-design} uses an extended definition of relaxed coverage that does not require the covering measure $p_s$ to be normalized.

    \begin{enumerate}
        \item Zero-Initialization property: $\Phi_0 = 0$. This holds since $F_0 = 0$ and $G_0 = 0$ by definition.
        \item Constant-Decreasing (Proposition~\ref{prop:potential-function-decrease}): for every coordinate descent step from $p_{t-1}$ to $p_t$ in \Cref{alg:extended-f-design}, since $\Delta_t = \bar{\Delta}$, we have
            \[
                \Phi_{t-1} - \Phi_t \geq 4\bar{\Delta} \cdot \overline{\SEC}_{\varepsilon}(\Gcal, \Lambda).
            \]
        \item Lower-Boundedness (Proposition~\ref{prop:wsec-sec-relationship}): for every $t > 0$,
            \[
                \Phi_t \geq -F_t \geq - \overline{\SEC}_{\varepsilon/\opnorm{p_t}{1}}(\Gcal, \Lambda).
            \]
    \end{enumerate}

    We now take these properties as given and defer the proofs of Propositions~\ref{prop:potential-function-decrease} and~\ref{prop:wsec-sec-relationship} to \Cref{sec:potential-function-constant-decreasing,sec:potential-function-lower-boundedness}, respectively.

    We now claim that the end of iteration $t = \floor{1/\bar{\Delta}}$ is never reached.
    Otherwise, at that iteration, $\opnorm{p_t}{1} = \bar{\Delta} t = \bar{\Delta} \cdot \floor{1/\bar{\Delta}} \leq 1$, and thus $\overline{\SEC}_{\varepsilon/\opnorm{p_t}{1}}(\Gcal, \Lambda) \leq \overline{\SEC}_{\varepsilon}(\Gcal, \Lambda)$ since $\varepsilon \mapsto \overline{\SEC}_{\varepsilon}(\Gcal, \Lambda)$ is monotonically decreasing by the monotonicity property (i) of Assumption~\ref{assum:admissible-relaxed-coverage}. Therefore,
    according to the Lower-Boundedness property, we have
    \[
        \Phi_t \geq -  \overline{\SEC}_{\varepsilon/\opnorm{p_t}{1}}(\Gcal, \Lambda) \geq -\overline{\SEC}_{\varepsilon}(\Gcal, \Lambda).
    \]

    On the other hand, since $\bar{\Delta} \leq 1$, we have $\floor{1/\bar{\Delta}} \geq \frac{1}{2\bar{\Delta}}$.
    According to the Constant-Decreasing property, we have
    \[
        \Phi_t
        \leq
        -4 \overline{\SEC}_{\varepsilon}(\Gcal, \Lambda) \cdot \bar{\Delta} t < - 2 \overline{\SEC}_{\varepsilon}(\Gcal, \Lambda).
    \]
    Thus, we reach a contradiction. Therefore, \Cref{alg:extended-f-design} terminates in $t_0 \leq \floor{1/\bar{\Delta}}$ iterations,
    and $\| p_{t_0} \|_1 \leq t_0 \bar{\Delta} \leq 1$. In addition, $p_{t_0}$ is a nonnegative combination of $\lambda_t$, which are elements of $\Lambda$.
    Therefore,
    \[
    p^* = p_{t_0} + (1-\| p_{t_0} \|) \widehat{\lambda}
    \]
    is also a nonnegative combination of elements of $\Lambda$, and since $\| p^* \|_1 = 1$, and all elements in $\Lambda$ are valid probability distributions, the
    combination is also a convex combination, and thus $p^* \in \co(\Lambda)$.

    We additionally show that the output distribution $p^*$ satisfies the Low Regret (LR) property. Indeed, since $p^*$ is a mixture of $p_{t_0}$ and $\widehat{\lambda}$, we have
    \begin{align*}
        &\quad \widehat{R}(p^*)
        =
        \widehat{R}(p_{t_0})
        = \frac{2}{\gamma} (F_{t_0} + \Phi_{t_0})
        \leq
        \frac{2}{\gamma} F_{t_0}
        \leq
        \frac{2}{\gamma}\overline{\SEC}_{\varepsilon}(\Gcal, \Lambda)
            \tag{By the lower Boundedness Property}
    \end{align*}
    where the first equality is because $\widehat{R}(\widehat{\lambda}) = 0$ and we use the extended definition of $\widehat{R}$ (Eq.~\eqref{eqn:extended-r}), the second equality is by the definition of $\Phi_{t_0}$, and the first inequality is by the fact that $\Phi_{t_0} \leq 0$ by the Constant-Decreasing property. The second inequality is by the Lower-Boundedness property.

    Next, we show that the output distribution $p^*$ satisfies Eq.~\eqref{eqn:monster-g}.
    Since the algorithm terminates at iteration $t_0$, $p_{t_0}$ satisfies that
    $
    \max_{\lambda \in \Lambda} \rbr{ \frac{1}{\gamma}\overline{\coverage}_\varepsilon(p_{t_0}, \lambda) - \widehat{R}(\lambda) }
    \leq
    \frac{8}{\gamma} \overline{\SEC}_\varepsilon(\Gcal, \Lambda)$; since $p^* \succeq p_{t_0}$, we also have
    $\overline{\coverage}_\varepsilon(p^*, \lambda) \leq
    \overline{\coverage}_\varepsilon(p_{t_0}, \lambda)$ by the monotonicity property of Assumption~\ref{assum:admissible-relaxed-coverage}, and thus,
    \begin{equation}
     \max_{\lambda \in \Lambda} \rbr{ \frac{1}{\gamma}\overline{\coverage}_\varepsilon(p^*, \lambda) - \widehat{R}(\lambda) }
    \leq
    \frac{8}{\gamma} \overline{\SEC}_\varepsilon(\Gcal, \Lambda).
    \end{equation}
    Combining this with the fact that $\widehat{R}(p^*) \leq \frac{2}{\gamma} \overline{\SEC}_\varepsilon(\Gcal, \Lambda)$, we have
    \[
    \widehat{R}(p^*)
    +
    \max_{\lambda \in \Lambda} \rbr{ \frac{1}{\gamma}\overline{\coverage}_\varepsilon(p^*, \lambda) - \widehat{R}(\lambda) }
    \leq
    \frac{10}{\gamma} \overline{\SEC}_\varepsilon(\Gcal, \Lambda),
    \]

    i.e., $p^*$ certifies that $\overline{\doec}_{\gamma, \varepsilon}(\hat{g}, \Gcal, \Lambda) \leq \frac{10}{\gamma} \overline{\SEC}_{\varepsilon}(\Gcal, \Lambda)$.
\end{proof}

\subsubsection{Property of Potential Function: Constant Decreasing}
\label{sec:potential-function-constant-decreasing}

\begin{proposition}[Constant-Decreasing Proposition of Potential Function]
    Suppose the relaxed coverage $\overline{\coverage}$ satisfies Assumption~\ref{assum:admissible-relaxed-coverage} with step-size threshold $\bar{\Delta}$. For each iteration $t$ in \Cref{alg:extended-f-design} such that the condition in line~\ref{line:check-violation} is satisfied, if $\Delta_t \leq \bar{\Delta}$, we have
    \[
        \Phi_{t-1} - \Phi_t \geq 4 \Delta_t \cdot \overline{\SEC}_{\varepsilon}(\Gcal, \Lambda)
    \]
    \label{prop:potential-function-decrease}
\end{proposition}

    \begin{proof}[Proof of Proposition~\ref{prop:potential-function-decrease}]
    We first examine $\Phi_{t-1} - \Phi_t$, which by its definition, is equal to
    \begin{equation}
    \Phi_{t-1} - \Phi_t = \Delta_t \overline{\coverage}_\varepsilon(p_t, \lambda_t) - \frac{\gamma}{2} \Delta_t \widehat{R}(\lambda_t)
    \label{eqn:pot-decr-lb}
    \end{equation}
    Since at every iteration $t$, the condition in line~\ref{line:check-violation} is satisfied,
    we also have
    \begin{equation}
    \overline{\coverage}_\varepsilon(p_{t-1}, \lambda_t) \geq \gamma \widehat{R}(\lambda_t) + 8 \overline{\SEC}_\varepsilon(\Gcal, \Lambda)
    \label{eqn:continue-iter}
    \end{equation}

    We now use that with $\Delta_t \leq \bar{\Delta}$, by the dilution stability assumption in Assumption~\ref{assum:admissible-relaxed-coverage},
    $\overline{\coverage}_\varepsilon(p_{t}, \lambda_t) \geq \frac12 \overline{\coverage}_\varepsilon(p_{t-1}, \lambda_t)$.

    Therefore, Eq.~\eqref{eqn:continue-iter} implies that
    \begin{equation}
    \overline{\coverage}_\varepsilon(p_{t}, \lambda_t) \geq 4 \overline{\SEC}_\varepsilon(\Gcal, \Lambda) + \frac{\gamma}{2} \widehat{R}(\lambda_t)
    \end{equation}
    Plugging this back to Eq.~\eqref{eqn:pot-decr-lb}, we have
    \[
    \Phi_{t-1} - \Phi_t
    \geq 4 \Delta_t \overline{\SEC}_\varepsilon(\Gcal, \Lambda).
    \]

\end{proof}

\subsubsection{Property of Potential Function: Lower Boundedness}
\label{sec:potential-function-lower-boundedness}

We establish a \emph{weighted sequential extrapolation bound} (Lemma~\ref{lem:weighted-from-monotone} below) for any relaxed coverage satisfying properties (i)--(iii) of Assumption~\ref{assum:admissible-relaxed-coverage}: monotonicity gives the bound for integer weights via a ``repetition'' argument; homogeneity and continuity extend it to real weights via a limiting argument.
Proposition~\ref{prop:wsec-sec-relationship} then follows by the definition of the potential function $\Phi_t$.

\begin{lemma}[Weighted relaxed SEC bound]
\label{lem:weighted-from-monotone}
Suppose the relaxed coverage $\overline{\coverage}$ is admissible (Assumption~\ref{assum:admissible-relaxed-coverage}). Then the following \emph{weighted sequential extrapolation bound} holds: for every $N \in \NN$, $\lambda_1, \ldots, \lambda_N \in \Lambda$, and weights $m_1, \ldots, m_N \geq 0$ with $M := \sum_{i=1}^N m_i > 0$, and every cushion parameter $\delta > 0$,
\begin{equation}
    \sum_{i=1}^N m_i\, \overline{\coverage}_{M\delta}\del{\textstyle\sum_{j=1}^i m_j \lambda_j,\ \lambda_i;\ \Gcal} \leq \overline{\SEC}_{\delta}(\Gcal, \Lambda).
    \label{eqn:weighted-sec-bound}
\end{equation}
\end{lemma}

\begin{proof}
We first prove the claim for integer weights $m_1, \ldots, m_N \in \NN$. Consider the length-$M$ sequence of elements of $\Lambda$ that repeats $\lambda_1$ for $m_1$ times, then $\lambda_2$ for $m_2$ times, and so on. By monotonicity of admissible relaxed coverage (property (i)), for every $i \in [N]$ and $k \in [m_i]$,
\[
\overline{\coverage}_{M\delta}\del{\textstyle\sum_{j=1}^i m_j \lambda_j,\ \lambda_i;\ \Gcal}
\leq
\overline{\coverage}_{M\delta}\del{\textstyle\sum_{j=1}^{i-1} m_j \lambda_j + k \lambda_i,\ \lambda_i;\ \Gcal},
\]
since $\sum_{j=1}^{i-1} m_j \lambda_j + k \lambda_i \preceq \sum_{j=1}^{i} m_j \lambda_j$. Summing the inequality over $k \in [m_i]$ and $i \in [N]$ gives
\[
\sum_{i=1}^N m_i\, \overline{\coverage}_{M\delta}\idel{\sum_{j=1}^i m_j \lambda_j, \lambda_i; \Gcal} \leq \overline{\SEC}_{\delta}(\Gcal, \Lambda)
\]
evaluated on the repeated sequence; the integer-weight claim follows.

For real weights $m_1, \ldots, m_N \geq 0$, fix $Z > 0$ and apply the integer-weight claim to the weights $w_i := \lfloor Z m_i \rfloor$ with $W := \sum_{i=1}^N w_i$:
\[
\sum_{i=1}^N w_i\, \overline{\coverage}_{W \delta}\del{\textstyle\sum_{j=1}^i w_j \lambda_j,\ \lambda_i;\ \Gcal} \leq \overline{\SEC}_{\delta}(\Gcal, \Lambda).
\]
Since $W \leq ZM$ and $\overline{\coverage}$ is non-increasing in the cushion parameter (Monotonicity), the inequality continues to hold with $W \delta$ replaced by $ZM\delta$. By homogeneity with $c = Z$,
\begin{align*}
    \sum_{i=1}^N w_i\, \overline{\coverage}_{ZM\delta}\del{\textstyle\sum_{j=1}^i w_j \lambda_j,\ \lambda_i;\ \Gcal}
    &=
    \sum_{i=1}^N \frac{w_i}{Z}\, \overline{\coverage}_{M\delta}\del{\textstyle\sum_{j=1}^i \frac{w_j}{Z} \lambda_j,\ \lambda_i;\ \Gcal}
    \\
    &\leq \overline{\SEC}_{\delta}(\Gcal, \Lambda).
\end{align*}

Letting $Z \to \infty$, so that $\frac{w_i}{Z} \to m_i$ for every $i$ and hence $\sum_{j=1}^i \frac{w_j}{Z} \lambda_j \to \sum_{j=1}^i m_j \lambda_j$ in total variation distance, and using continuity of the coverage in its first argument (property (iii) of Assumption~\ref{assum:admissible-relaxed-coverage}), the left-hand side converges to $\sum_{i=1}^N m_i\, \overline{\coverage}_{M\delta}\del{\sum_{j=1}^i m_j \lambda_j, \lambda_i; \Gcal}$, which proves the claim.
\end{proof}

\begin{proposition}[Lower Boundedness Property of Potential Function]
    \label{prop:wsec-sec-relationship}
    For every iteration $t > 0$ of \Cref{alg:extended-f-design}, the following inequality holds:
    \begin{align}
        \Phi_t
        \geq
        -F_t \geq
        - \overline{\SEC}_{\varepsilon / \opnorm{p_t}{1}}(\Gcal, \Lambda).
    \end{align}
\end{proposition}
\begin{proof}[Proof of Proposition~\ref{prop:wsec-sec-relationship}]
    Starting from the definition of $\Phi_t$, we have $\Phi_t = -F_t + \gamma G_t \geq -F_t$ since $G_t \geq 0$ by definition. Therefore, it suffices to show that $F_t \leq \overline{\SEC}_{\varepsilon / \opnorm{p_t}{1}}(\Gcal, \Lambda)$.
    To this end, we use the properties of the relaxed coverage. Since $p_s = \sum_{j=1}^s \Delta_j \lambda_j$ for every $s \leq t$, the sum
    \[
    F_t = \sum_{s=1}^t \Delta_s\, \overline{\coverage}_{\varepsilon}(p_s, \lambda_s; \Gcal)
    \]
    is a real-weighted sequential extrapolation sum with weights $m_s = \Delta_s$ and total weight $M = \sum_{s=1}^t \Delta_s = \opnorm{p_t}{1}$; by the admissibility of $\overline{\coverage}$ (Assumption~\ref{assum:admissible-relaxed-coverage}), Lemma~\ref{lem:weighted-from-monotone} above asserts that it obeys the weighted sequential extrapolation bound Eq.~\eqref{eqn:weighted-sec-bound}.
    Applying Eq.~\eqref{eqn:weighted-sec-bound} by setting the cushion parameter $\delta$ therein to be $\varepsilon / \opnorm{p_t}{1}$, we have $F_t \leq \overline{\SEC}_{\varepsilon / \opnorm{p_t}{1}}(\Gcal, \Lambda)$.
\end{proof}

\subsection{Relating \texorpdfstring{$\overline{\SEC}_{\varepsilon}$}{relaxed-SEC} to Other Complexity Measures
}

In this section, we prove upper bounds on the relaxed $\varepsilon$-SEC for our settings of interest: the trivial relaxation under finite Eluder dimension (Proposition~\ref{prop:sec-eluder}), and the three running examples: discrete action space, per-context (generalized) linear reward, and $h$-smoothed regret, under their cushioned relaxed coverages (Lemma~\ref{lem:relaxed-sec-running-examples}). These bounds, together with the guarantee of \Cref{alg:extended-f-design} in Theorem~\ref{thm:doec-leq-sec}, lead to regret guarantees of \generalfalcon under these settings.

\subsubsection{Case 1: \texorpdfstring{$\SEC_{\varepsilon}$}{relaxed-SEC} in Finite Eluder Dimension Case}
\label{app:finite-eluder}

Recall that, for a function class that has a finite Eluder dimension and $\Lambda$ is a set of Dirac measures over the action space, we show that the original $\varepsilon$-SEC is bounded.
We first recall the definition of the Eluder dimension:

\begin{definition}[Eluder Dimension~\citep{russo13eluder}]
    \label{def:eluder}
    Given a function class $\Gcal$ mapping from a domain $\Zcal$ to $\RR$, a point $z \in \Zcal$ is said to be $\varepsilon$-independent of a set of points $\{z_1, z_2, \ldots, z_n\} \subseteq \Zcal$ with respect to $\Gcal$ if there exist two functions $f, f' \in \Gcal$ such that $\sqrt{\sum_{i=1}^n (f(z_i) - f'(z_i))^2} \leq \varepsilon$ and $|f(z) - f'(z)| > \varepsilon$.
    The eluder dimension $\Edim(\Gcal, \varepsilon)$ is defined as the length of the longest sequence of points $z_1, z_2, \ldots, z_m \in \Zcal$ such that each point $z_i$ is $\varepsilon'$-independent of its predecessors $\{z_1, z_2, \ldots, z_{i-1}\}$ with respect to $\Gcal$ for some $\varepsilon' \geq \varepsilon$.
\end{definition}

Next, we are going to prove Proposition~\ref{prop:sec-eluder}, which states that $\SEC_\varepsilon(\Gcal, \Lambda)$ is upper bounded by the eluder dimension upto some logarithmic factors.
We note that this property was already given by \citet[Lemma E.3]{agarwal2024non}. However, their proof contains a small typo that relies on an incorrect lemma~\citep[Lemma E.2]{agarwal2024non} in their peeling argument.
Specifically, we need to improve their $\log T$ factors therein to $\log(1/\varepsilon)$.
Hence, we provide a corrected proof with a peeling argument and highlight the difference between the proof of Lemma~\ref{lemma:bound-sec-with-peeling} and \citet[Lemma E.2]{agarwal2024non}.
The readers are welcome to refer to \citet[Lemma E.2 and E.3]{agarwal2024non} for more context.

\begin{proposition}
    \label{prop:sec-eluder}
    For $\Lambda = \cbr{\delta_a: a \in \Acal}$,
    $\SEC_\varepsilon(\Gcal, \Lambda) \lesssim \Edim(\Gcal, \sqrt{\varepsilon}) \log^2(\frac{1}{\varepsilon})$, and thus
    \[
        \doec_{\gamma,\varepsilon}(\Gcal, \Lambda) \lesssim \frac{1}{\gamma}\Edim(\Gcal, \sqrt{\varepsilon}) \log^2(\frac{1}{\varepsilon})
        ,
    \]
    where $\Edim$ denotes the Eluder dimension.
\end{proposition}

Aside from our running Example 2,
Proposition~\ref{prop:sec-eluder} enables a provable regret guarantee of \generalfalcon under the per-context generalized linear reward model~\citep[][Example 2]{xu2020upper}. In this setting, for each $x$, $\Edim(\Fcal_x, \sqrt{\varepsilon}) \leq \iupboundlog{d \log\frac{1}{\varepsilon}}$~\citep{russo13eluder}, and $\doec_{\gamma,\varepsilon}(\Fcal_x, \Lambda) \leq \frac{d}{\gamma} \log^3(\frac1\varepsilon)$, and thus \generalfalcon has a regret bound of $\iupboundlog{\isqrt{d T \log|\Fcal|}}$, which matches~\citet{xu2020upper} and improves the number of offline regression oracle calls from $O(T)$ to $O(\log(T))$.
We also note that small Eluder dimension can capture reward classes beyond generalized linear models~\citep{li2022understanding}.

\begin{proof}[Proof of Proposition~\ref{prop:sec-eluder}]
    For any sequence $\cbr{\lambda_i = \delta_{a_i}, h_i}_{i=1}^N$ such that $h_i \in \Dcal \Gcal$, we will split the interval $(0, 1]$ into multiple small intervals with length increasing exponentially.
    Specifically, we split $(0, 1]$ into disjoint intervals $(2^{k-1}\sqrt{\varepsilon}, 2^k \sqrt{\varepsilon}]$ for $k = 1, 2, \ldots, K$, where $K = \lceil \log_2(1/\sqrt{\varepsilon}) \rceil$.
    Then, we can bound the weighted smooth Eluder summation as
    \begin{align*}
        \sum_{i=1}^N \frac{ h_i(a_i)^2 }{N \varepsilon + \sum_{j=1}^i h_i(a_j)^2 }
        &=
        \sum_{k=1}^K \sum_{
            \substack{i \in [1, N]: \\
            h_i(a_i) \in (2^{k-1}\sqrt{\varepsilon}, 2^k \sqrt{\varepsilon}]} }
            \frac{ h_i(a_i)^2 }{N \varepsilon + \sum_{j=1}^i h_i(a_j)^2 }
        \\
        &\leq
        \sum_{k=1}^K
        4\Edim(\Gcal, 2^{k-1}\sqrt{\varepsilon}) \log\del{1 + \frac{(2^{k-1}\sqrt{\varepsilon})^2}{\varepsilon} }
            \tag{apply Corrected Lemma~\ref{lemma:bound-sec-with-peeling} to each small interval}
        \\
        &\leq
        4\Edim(\Gcal, \sqrt{\varepsilon}) \sum_{k=1}^K \log\del{1 + 4^{k-1} }
            \tag{$\Edim(\Gcal, 2^{k-1}\sqrt{\varepsilon}) \leq \Edim(\Gcal, \sqrt{\varepsilon})$ for all $k$}
        \\
        &\leq
        4\Edim(\Gcal, \sqrt{\varepsilon}) K \log\del{\frac{1}{K} \sum_{k=1}^K \del{1 + 4^{k-1}}}
        \tag{Jensen's inequality}
        \\
        &\leq
        16 \Edim(\Gcal, \sqrt{\varepsilon}) K^2
    \end{align*}
    Since the above inequality holds for every $N \in \NN$, $h_1, \ldots, h_N \in \Dcal \Gcal$ and $\lambda_1, \ldots, \lambda_N \in \Lambda$ and every decomposition of $q$ into weighted distributions in $\Lambda$, we take the supremum over them to have Proposition~\ref{prop:sec-eluder} hold.
\end{proof}

\begin{lemma}[Bounding $\varepsilon$-SEC in Finite Eluder Dimension Case with Peeling Argument]
    \label{lemma:bound-sec-with-peeling}
    For any function class $\Gcal: \Acal \to \RR$, $\varepsilon > 0$, and any sequence $\cbr{a_i, h_i}_{i=1}^N$ such that $h_i \in \Dcal \Gcal$ and $h_i(a) \in (\theta, 2\theta]$ for some $\theta > 0$. Then, we have the following inequality holds:
    \begin{align}
        \sum_{i=1}^N \frac{ h_i(a_i)^2 }{N \varepsilon + \sum_{j=1}^i h_i(a_j)^2 }
        \leq
        \Edim(\Gcal, \sqrt{\varepsilon}) \log\del{1 + \frac{\theta^2}{ \varepsilon}}
    \end{align}
\end{lemma}

\begin{proof}[Proof of Lemma~\ref{lemma:bound-sec-with-peeling}]
    The proof is similar to the proof of~\citet[][Lemma E.2]{agarwal2024non}.
    \begin{align*}
        &\quad \sum_{i=1}^N \frac{ h_i(a_i)^2 }{N \varepsilon + \sum_{j=1}^i h_i(a_j)^2 }
        =
        \sum_{i=1}^N \frac{ h_i(a_i)^2 }{N \varepsilon + \sum_{j=1}^{i-1} h_i(a_j)^2 + h_i(a_i)^2 }
        \\
        &\leq
        4\Edim(\Gcal, \sqrt{\varepsilon}) \sum_{n=1}^{N / \Edim(\Gcal, \sqrt{\varepsilon})} \frac{\theta^2}{N \varepsilon + n \theta^2}
        \\
        &\leq
        4\Edim(\Gcal, \sqrt{\varepsilon}) \int_{0}^{N / \Edim(\Gcal, \sqrt{\varepsilon})} \frac{1}{N \varepsilon/\theta^2 + x} \ddx
        \\
        & = 4\Edim(\Gcal, \sqrt{\varepsilon}) \del{ \log\del{\frac{N\varepsilon}{\theta^2} + \frac{N}{\Edim(\Gcal, \sqrt{\varepsilon})}} - \log\del{\frac{N\varepsilon}{\theta^2}} }
        \\
        &\leq
        4\Edim(\Gcal, \sqrt{\varepsilon})\log\del{1 + \frac{\theta^2}{\varepsilon}}
            \tag{$\Edim(\Gcal, \varepsilon)$ is at least 1}
    \end{align*}
    The first inequality holds by filling the summation with the maximum possible allocations in the constructed buckets -- see \citet[][Lemma E.2]{agarwal2024non} for the relevant definitions. For each bucket, the maximum number of elements is the eluder dimension $\Edim(\Gcal, \varepsilon)$ due to the construction.
    The second inequality holds because the summation can be upper bounded by an integral.
    The third inequality holds by relaxing $\Edim$ in the denominator to $1$.
\end{proof}

\subsubsection{Case 2: \texorpdfstring{$\overline{\SEC}_{\varepsilon}$}{relaxed-SEC} in the Running Examples}
\label{app:smoothed-regret}

In this section, we bound the relaxed SEC, $\overline{\SEC}_\varepsilon(\Gcal, \Lambda)$, in the three running examples, using the cushioned relaxed coverages of Section~\ref{sec:structure-doec}.
Throughout, $\Sigma_p := \sum_{a \in \Acal} p(a)\phi(a)\phi(a)^\top$, and in the $h$-smoothed setting $\lambda(a), p(a)$ denote densities with respect to a base measure $\mu \in \Delta(\Acal)$, and recall that $\Delta_h^\mu(\Acal) = \cbr{\lambda \in \Delta(\Acal) : \frac{\diff \lambda}{\diff \mu}(a) \leq \frac{1}{h}\ \forall a \in \Acal}$.

\begin{lemma}[Relaxed SEC bounds for the running examples]
\label{lem:relaxed-sec-running-examples}
With the cushioned relaxed coverages below,  $\overline{\SEC}_{\varepsilon}(\Gcal, \Lambda)$ is bounded as follows:
\begin{enumerate}
    \item \textbf{Discrete action space} ($\Gcal \subseteq [0,1]^{\Acal}$, $\Lambda = \cbr{\delta_a : a \in \Acal}$), with $\overline{\coverage}_{\varepsilon}(p, \lambda; \Gcal) = \sum_{a \in \Acal} \frac{\lambda(a)}{p(a) + \varepsilon/|\Acal|}$:
    \[
        \overline{\SEC}_{\varepsilon}(\Gcal, \Lambda) \leq |\Acal| \log\del{1 + \frac{1}{\varepsilon}};
    \]
    \item \textbf{Per-context (generalized) linear reward} ($\Gcal = \cbr{a \mapsto \sigma(\phi(a)^\top \theta)}$, link $\sigma$ with $\kappa := \overline{L}/\underline{L}$ and parameter diameter $B$), with $\varepsilon' := \frac{\varepsilon}{\underline{L}^2 B^2}$ and $\overline{\coverage}_{\varepsilon}(p, \lambda; \Gcal) = \kappa^2 \tr\del{\del{\Sigma_p + \varepsilon' I}^{-1} \Sigma_\lambda}$:
    \[
        \overline{\SEC}_{\varepsilon}(\Gcal, \Lambda) \leq \kappa^2 d \log\del{1 + \frac{1}{\varepsilon'}}
    \]
    (taking $\sigma = \mathrm{id}$, so $\kappa = 1$, covers the per-context linear reward);
    \item \textbf{$h$-smoothed regret} ($\Lambda = \Delta_h^{\mu}(\Acal)$), with $\overline{\coverage}_{\varepsilon}(p, \lambda; \Gcal) = \frac{1}{h} \EE_{a \sim \mu}\sbr{\frac{\lambda(a)}{p(a) + \varepsilon}}$:
    \[
        \overline{\SEC}_{\varepsilon}(\Gcal, \Lambda) \leq \frac{1}{h} \log\del{1 + \frac{1}{\varepsilon}}.
    \]
\end{enumerate}
\end{lemma}

\begin{proof}
Fix $N \in \NN$ and $\lambda_1, \ldots, \lambda_N \in \Lambda$, and write $\lambda_{1:i} := \sum_{j=1}^i \lambda_j$. Recall that, in the relaxed $\varepsilon$-SEC (Definition~\ref{def:delta-sec}), the cushion parameter of a length-$N$ sequence is $N\varepsilon$.
We will repeatedly use the elliptical potential lemma~\citep[see, e.g.,][]{abbasi2011improved}
\[
\sum_{i=1}^N \tr( (\varepsilon_0 I + A_{1:i})^{-1} A_i )
\leq
\log \frac{ \det( \varepsilon_0 I + A_{1:N} )  }{ \det( \varepsilon_0 I ) }
\]
for any symmetric positive semidefinite $A_1, \dots, A_N \succeq 0$ and any $\varepsilon_0 > 0$, with $A_{1:i} := \sum_{j=1}^i A_j$. We will also use the lemma's scalar version
\begin{equation}
\sum_{i=1}^N \frac{a_i}{\varepsilon_0 + a_{1:i}} \leq \log\frac{\varepsilon_0 + a_{1:N}}{\varepsilon_0}
\qquad \text{for } a_1, \ldots, a_N \geq 0,\ \varepsilon_0 > 0,
\label{eqn:scalar-log-bound}
\end{equation}

\emph{Discrete:} we have
\[
\sum_{i=1}^N \sum_{a \in \Acal} \frac{\lambda_i(a)}{\lambda_{1:i}(a) + \frac{N\varepsilon}{|\Acal|}}
\leq \sum_{a \in \Acal} \log\del{1 + \frac{|\Acal|\, \lambda_{1:N}(a)}{N\varepsilon}}
\leq |\Acal| \log\del{1 + \frac{1}{\varepsilon}},
\]
where in the first inequality we are applying Eq.~\eqref{eqn:scalar-log-bound} per action with $a_i = \lambda_i(a)$, $\varepsilon_0 = \frac{N\varepsilon}{|\Acal|}$; in the second inequality we are applying Jensen's inequality to the concave function $\log(1+x)$, and using $\sum_{a \in \Acal} \lambda_{1:N}(a) = N$.
Taking the supremum over $N$ and the sequence gives the claimed bound.

\emph{Per-context generalized linear:} we have
\[
\kappa^2 \sum_{i=1}^N \tr\del{\del{\Sigma_{\lambda_{1:i}} + N\varepsilon' I}^{-1} \Sigma_{\lambda_i}}
\leq \kappa^2 \log\frac{\det\del{N\varepsilon' I + \Sigma_{\lambda_{1:N}}}}{\det\del{N\varepsilon' I}}
\leq \kappa^2 d \log\del{1 + \frac{1}{d \varepsilon'}},
\]
where in the first inequality we use the elliptical potential lemma (Eq.~\eqref{eqn:scalar-log-bound}) with $A_i = \Sigma_{\lambda_i}$ and $\varepsilon_0 = N\varepsilon'$ to control the sum of traces in the relaxed coverage; in the second inequality, we use
$
\log\frac{\det(N\varepsilon'I+\Sigma_{\lambda_{1:N}})}{\det(N\varepsilon'I)}
=
\log\det\del{I+\frac{1}{N\varepsilon'}\Sigma_{\lambda_{1:N}}}
\le
d\log\del{1+\frac{\tr(\Sigma_{\lambda_{1:N}})}{dN\varepsilon'}}
\le
d\log\del{1+\frac{1}{d\varepsilon'}}
$,
and $\lambda_{\max}\del{\Sigma_{\lambda_{1:N}}} \leq \tr\del{\Sigma_{\lambda_{1:N}}} \leq N$ (as $\norm{\phi(a)}_2 \leq 1$).

\emph{$h$-smoothed:} we have
\[
\frac{1}{h} \sum_{i=1}^N \EE_{a \sim \mu}\sbr{\frac{\lambda_i(a)}{\lambda_{1:i}(a) + N\varepsilon}}
\leq \frac{1}{h}\, \EE_{a \sim \mu}\sbr{\log\frac{N\varepsilon + \lambda_{1:N}(a)}{N\varepsilon}}
\leq \frac{1}{h} \log\del{1 + \frac{1}{\varepsilon}},
\]
where in the first inequality we are applying Eq.~\eqref{eqn:scalar-log-bound} per action under $\mu$ with $a_i = \lambda_i(a)$ and $\varepsilon_0 = N\varepsilon$; in the second inequality we are applying Jensen's inequality to the concave function $\log(1+x)$ under $\mu$ with $\EE_{a \sim \mu}\sbr{\lambda_{1:N}(a)} = N$.
In each case, taking the supremum over $N$ and $\lambda_1, \ldots, \lambda_N \in \Lambda$ gives the stated bound on $\overline{\SEC}_{\varepsilon}(\Gcal, \Lambda)$; the relaxed DOEC bounds then follow from Theorem~\ref{thm:doec-leq-sec}.
\end{proof}

\subsection{Admissibility of the Cushioned Relaxed Coverages in the Running Examples}
\label{app:relaxed-sec-examples}

In this section, we verify Assumption~\ref{assum:admissible-relaxed-coverage} in our settings of interest: first for the trivial relaxation $\overline{\coverage}_{\varepsilon} = \coverage_{\varepsilon}$ (Lemma~\ref{lem:trivial-relaxation-admissible}), and then for the $\varepsilon$-cushioned relaxed coverages of three running examples (Lemma~\ref{lem:relaxed-coverage-admissible}), whose step-size thresholds are quoted in Section~\ref{sec:structure-doec}; The cushioned relaxations are needed only by \Cref{alg:extended-f-design} and its analysis; Section~\ref{sec:offline-oracle-efficient-algorithm} and its supporting Appendix~\ref{app:relaxed-coverage} work with the cushion-free relaxed coverages throughout and do not depend on this subsection.

\begin{lemma}[Admissibility of the original coverage]
\label{lem:trivial-relaxation-admissible}
The original coverage $\coverage_{\varepsilon}$ (Eq.~\eqref{eqn:coverage}), extended to unnormalized covering measures as in line~\ref{line:compute-lambdat} of \Cref{alg:extended-f-design} is admissible with step-size threshold $\bar{\Delta} = \varepsilon$.
\end{lemma}

\begin{proof}
Recall that $\coverage_{\varepsilon}(p, \lambda; \Gcal) = \sup_{h \in \Dcal\Gcal} \frac{(\EE_{a \sim \lambda}[h(a)])^2}{\varepsilon + \sum_{a \in \Acal} p(a) h(a)^2}$, where every $h \in \Dcal\Gcal$ has range $[-1, 1]$.

\textbf{Properties (i) and (ii).} The covering measure $p$ and the cushion parameter $\varepsilon$ enter each ratio only through the denominator $\varepsilon + \sum_{a \in \Acal} p(a) h(a)^2$, which is nondecreasing in $(p, \varepsilon)$ and scales by $c$ under $(p, \varepsilon) \to (cp, c\varepsilon)$; hence the supremum over $h$ is non-increasing in $(p, \varepsilon)$ and jointly homogeneous of degree $-1$ in $(p, \varepsilon)$.

\textbf{Property (iii).} Denote $J_h(p) := \frac{\del{\EE_{a \sim \lambda}[h(a)]}^2}{\varepsilon + \sum_{a \in \Acal} p(a) h(a)^2}$,
For any two probability measures $p, p'$, denote their total variation distance as $\opnorm{p - p'}{\mathrm{TV}} = \frac12 \sum_{a \in A}\abs{p(a) - p'(a)}$
\[
    \abs{J_h(p) - J_h(p')}
    =
    \del{\EE_{a \sim \lambda}[h(a)]}^2 \cdot \frac{\abs{\sum_{a \in \Acal} \del{p'(a) - p(a)} h(a)^2}}{\del{\varepsilon + \sum_{a \in \Acal} p(a) h(a)^2}\del{\varepsilon + \sum_{a \in \Acal} p'(a) h(a)^2}}
    \leq \frac{2}{\varepsilon^2}\, \opnorm{p - p'}{\mathrm{TV}},
\]
since $\EE_{a \sim \lambda} [h(a)] \leq 1$, $h(a) \leq 1$ and both denominators $\geq \varepsilon$. Hence each $J_h$ is $\frac{1}{\varepsilon^2}$-Lipschitz in $p$ with the same constant for every $h \in \Dcal\Gcal$; that is, the family $\cbr{J_h : h \in \Dcal\Gcal}$ is equi-Lipschitz. The coverage is their pointwise supremum $\coverage_{\varepsilon}(\cdot, \lambda; \Gcal) = \sup_{h \in \Dcal\Gcal} J_h$. Since a pointwise supremum of an equi-$L$-Lipschitz family is itself $L$-Lipschitz,\footnote{For every $h$ and all $p, p'$, $J_h(p) \leq J_h(p') + L\opnorm{p - p'}{\mathrm{TV}} \leq \sup_{h'} J_{h'}(p') + L\opnorm{p - p'}{\mathrm{TV}}$; taking $\sup_h$ on the left and then swapping $p$ and $p'$ gives $\abs{\sup_h J_h(p) - \sup_h J_h(p')} \leq L\opnorm{p - p'}{\mathrm{TV}}$.} taking $L = \frac{1}{\varepsilon^2}$ shows $\coverage_{\varepsilon}(\cdot, \lambda; \Gcal)$ is $\frac{1}{\varepsilon^2}$-Lipschitz, and in particular continuous, in its first argument, as required.

\textbf{Property (iv).} For every $h \in \Dcal\Gcal$, every $\lambda \in \Lambda$, and every $0 \leq \Delta \leq \varepsilon$, since $\sum_{a \in \Acal} \Delta\lambda(a) h(a)^2 = \Delta\, \EE_{a \sim \lambda}\sbr{h(a)^2}$,
\[
\frac{\varepsilon + \sum_{a \in \Acal} p(a) h(a)^2}{\varepsilon + \sum_{a \in \Acal} p(a) h(a)^2 + \Delta\, \EE_{a \sim \lambda}\sbr{h(a)^2}}
\geq
\frac{\varepsilon}{\varepsilon + \Delta\, \EE_{a \sim \lambda}\sbr{h(a)^2}}
\geq
\frac{\varepsilon}{\varepsilon + \Delta}
\geq \fr12
\]
where the first inequality uses the elementary fact that $\frac{A+B}{A+C} \geq \frac{B}{C}$ for $C \geq B \geq 0$ and $A \geq 0$. The second inequality uses $\EE_{a \sim \lambda}\sbr{h(a)^2} \leq 1$, and the third uses $\Delta \leq \varepsilon$.
Then, for every $h \in \Dcal\Gcal$, we have
\[
\frac{(\EE_{a \sim \lambda}[h(a)])^2}{\varepsilon + \sum_{a \in \Acal} (p(a) + \Delta\lambda(a)) h(a)^2}
\geq
\frac{1}{2}\, \frac{(\EE_{a \sim \lambda}[h(a)])^2}{\varepsilon + \sum_{a \in \Acal} p(a) h(a)^2}.
\]
Taking the supremum over $h$ on the above inequality on both sides
gives $\coverage_{\varepsilon}(p + \Delta\lambda, \lambda; \Gcal) \geq \frac{1}{2}\, \coverage_{\varepsilon}(p, \lambda; \Gcal)$, i.e., property (iv) holds with $\bar{\Delta} = \varepsilon$.
\end{proof}

We now state the main result of this section, referenced in Section~\ref{sec:structure-doec}.

\begin{lemma}[Validity and Admissibility of the relaxed coverages for the running examples]
\label{lem:relaxed-coverage-admissible}
In each of the following settings, $\overline{\coverage}_{\varepsilon}$ is a valid relaxed coverage, i.e., $\overline{\coverage}_{\varepsilon}(p, \lambda; \Gcal) \geq \coverage_{\varepsilon}(p, \lambda; \Gcal)$ for every nonnegative measure $p$ and $\lambda \in \Lambda$, and satisfies Assumption~\ref{assum:admissible-relaxed-coverage} with the step-size threshold $\bar{\Delta}$ given below:
\begin{enumerate}
    \item \textbf{(Discrete action space)} $|\Acal| < \infty$, $\Gcal \subseteq [0,1]^{\Acal}$, $\Lambda = \cbr{\delta_a : a \in \Acal}$:
    \[
        \overline{\coverage}_{\varepsilon}(p, \lambda; \Gcal) = \sum_{a \in \Acal} \frac{\lambda(a)}{p(a) + \varepsilon/|\Acal|},
        \quad
        \bar{\Delta} = \frac{\varepsilon}{|\Acal|};
    \]
    \item \textbf{(Per-context generalized linear reward)} $\Lambda = \cbr{\delta_a : a \in \Acal}$, $\Gcal = \cbr{a \mapsto \sigma(\phi(a)^\top \theta) : \theta \in \Theta_x} \subseteq [0,1]^{\Acal}$ with $\Theta_x := \cbr{\theta(x) : \theta \in \Theta}$, link function $\sigma$ satisfying $0 < \underline{L} \leq \sigma' \leq \overline{L}$ and $\kappa := \overline{L}/\underline{L}$, $\norm{\phi(a)}_2 \leq 1$, and $\sup_{\theta, \theta' \in \Theta_x} \norm{\theta - \theta'}_2 \leq B$:
    with $\varepsilon' := \frac{\varepsilon}{\underline{L}^2 B^2}$,
    \[
        \overline{\coverage}_{\varepsilon}(p, \lambda; \Gcal) = \kappa^2 \tr\del{\del{\Sigma_p + \varepsilon' I}^{-1} \Sigma_\lambda},
        \quad
        \bar{\Delta} = \min\cbr{\varepsilon', 1},
    \]
    where $\Sigma_p := \sum_{a \in \Acal} p(a) \phi(a)\phi(a)^\top$; taking $\sigma$ to be the identity ($\kappa = 1$, $\varepsilon' = \frac{\varepsilon}{B^2}$) covers the per-context linear reward;
    \item \textbf{($h$-smoothed regret)} $\Gcal \subseteq [0,1]^{\Acal}$, $\Lambda = \Delta_h^{\mu}(\Acal)$:
    \[
        \overline{\coverage}_{\varepsilon}(p, \lambda; \Gcal) = \frac{1}{h} \EE_{a \sim \mu}\sbr{\frac{\lambda(a)}{p(a) + \varepsilon}},
        \quad
        \bar{\Delta} = h\varepsilon,
    \]
    where $\lambda(a)$ and $p(a)$ denote densities with respect to the base measure $\mu$.
\end{enumerate}
\end{lemma}

\begin{proof}[Proof of Lemma~\ref{lem:relaxed-coverage-admissible}]
Throughout, write $h := g - g'$ for $g, g' \in \Gcal$, so that $\coverage_{\varepsilon}(p, \lambda; \Gcal) = \sup_{h \in \Dcal\Gcal} \frac{(\EE_{a \sim \lambda}[h(a)])^2}{\varepsilon + \sum_{a \in \Acal} p(a) h(a)^2}$. For each setting, we verify in order: the validity of the relaxation $\coverage_{\varepsilon} \leq \overline{\coverage}_{\varepsilon}$; properties (i)--(iii) of Assumption~\ref{assum:admissible-relaxed-coverage} (monotonicity, homogeneity, and continuity); and the dilution-stability property (iv) with the stated step-size threshold $\bar{\Delta}$.

\textbf{Validity of the relaxation}. Each claim follows by re-running the Cauchy--Schwarz argument of Lemma~\ref{lem:relaxed-coverage-valid} (Appendix~\ref{app:relaxed-coverage}) while keeping the cushion parameter and distributing it over the relevant actions or directions.

\begin{itemize}
\item \emph{Discrete:} since $h(a)^2 \leq 1$ and $\lambda(a) \leq 1$,
\begin{align*}
\del{\sum_{a \in \Acal} \lambda(a) h(a)}^2
&\leq \del{\sum_{a \in \Acal} \frac{\lambda(a)^2}{p(a) + \frac{\varepsilon}{|\Acal|}}} \del{\sum_{a \in \Acal} \del{p(a) + \tfrac{\varepsilon}{|\Acal|}} h(a)^2}
\\
&\leq \del{\sum_{a \in \Acal} \frac{\lambda(a)}{p(a) + \frac{\varepsilon}{|\Acal|}}} \del{\sum_{a \in \Acal} p(a) h(a)^2 + \varepsilon},
\end{align*}
and dividing both sides by $\varepsilon + \sum_{a} p(a) h(a)^2$ and taking the supremum over $h$ gives the claim.

\item \emph{Per-context generalized linear:} by the mean value theorem, $h(a) = \sigma'(\xi_a)\, \phi(a)^\top u$ for some $\xi_a$, where $u := \theta - \theta'$; hence $\underline{L}^2 (\phi(a)^\top u)^2 \leq h(a)^2 \leq \overline{L}^2 (\phi(a)^\top u)^2$. Applying Lemma~\ref{lem:trace-quadratic-bound} with $A = \Sigma_p + \varepsilon' I$ and $C = \Sigma_\lambda$, together with Jensen's inequality $\del{\EE_{a \sim \lambda}[\phi(a)^\top u]}^2 \leq u^\top \Sigma_\lambda u$,
\begin{align*}
\del{\EE_{a \sim \lambda}[h(a)]}^2
&\leq \overline{L}^2\, u^\top \Sigma_\lambda u
\leq \overline{L}^2 \tr\del{(\Sigma_p + \varepsilon' I)^{-1} \Sigma_\lambda} \cdot u^\top (\Sigma_p + \varepsilon' I) u
\\
&\leq \kappa^2 \tr\del{(\Sigma_p + \varepsilon' I)^{-1} \Sigma_\lambda} \del{\sum_{a \in \Acal} p(a) h(a)^2 + \varepsilon},
\end{align*}
where the last step bounds the last factor using
$u^\top \Sigma_p u \leq \frac{1}{\underline{L}^2}\sum_{a} p(a) h(a)^2$ and $u^\top (\varepsilon' I) u \leq  \varepsilon' B^2 = \frac{\varepsilon}{\underline{L}^2}$.

\item \emph{$h$-smoothed:} by the Cauchy--Schwarz inequality with respect to $\mu$, using $h(a)^2 \leq 1$, $\lambda(a) \leq \frac{1}{h}$, and that $\mu$ is a probability measure,
\begin{align*}
\del{\EE_{a \sim \mu}\sbr{\lambda(a) h(a)}}^2
&\leq \EE_{a \sim \mu}\sbr{\frac{\lambda(a)^2}{p(a) + \varepsilon}} \cdot \EE_{a \sim \mu}\sbr{(p(a) + \varepsilon) h(a)^2}
\\
&\leq \frac{1}{h} \EE_{a \sim \mu}\sbr{\frac{\lambda(a)}{p(a) + \varepsilon}} \del{\EE_{a \sim \mu}\sbr{p(a) h(a)^2} + \varepsilon}.
\end{align*}
\end{itemize}

\textbf{Properties (i) and (ii).} In all three settings, the covering measure $p$ and the cushion parameter $\varepsilon$ enter only through the denominators $p(a) + \frac{\varepsilon}{|\Acal|}$, $\Sigma_p + \varepsilon' I$, and $p(a) + \varepsilon$, each of which is nondecreasing in $(p, \varepsilon)$; monotonicity follows, in the generalized linear setting via the fact that $A \succeq B \succ 0$ implies $\tr\del{A^{-1} C} \leq \tr\del{B^{-1} C}$ for $C \succeq 0$. Moreover, each denominator scales by $c$ under $(p, \varepsilon) \to (cp, c\varepsilon)$, so $\overline{\coverage}_{\varepsilon}(p, \lambda; \Gcal)$ scales by $\frac{1}{c}$; i.e.\ it is jointly homogeneous of degree $-1$ in $(p, \varepsilon)$, which verifies property (ii).

\textbf{Property (iii).}
We note that the relaxed coverages considered here are all elementary functions of $p$, and thus they are continuous in $p$.

\textbf{Property (iv).} \emph{Discrete} ($\bar{\Delta} = \frac{\varepsilon}{|\Acal|}$): for every $a \in \Acal$ and $0 \leq \Delta \leq \bar{\Delta}$, since $\lambda(a) \leq 1$,
\[
\frac{p(a) + \frac{\varepsilon}{|\Acal|}}{p(a) + \Delta \lambda(a) + \frac{\varepsilon}{|\Acal|}}
\geq \frac{\frac{\varepsilon}{|\Acal|}}{\frac{\varepsilon}{|\Acal|} + \Delta}
\geq \frac{1}{2},
\]
so each term of $\overline{\coverage}_{\varepsilon}(p + \Delta\lambda, \lambda; \Gcal)$ is at least half the corresponding term of $\overline{\coverage}_{\varepsilon}(p, \lambda; \Gcal)$.

\smallskip

\noindent\emph{Per-context generalized linear}: since $\norm{\phi(a)}_2 \leq 1$, we have $\Sigma_\lambda \preceq I$, so for $\Delta \leq \varepsilon'$, $\Sigma_p + \Delta \Sigma_\lambda + \varepsilon' I \preceq 2 (\Sigma_p + \varepsilon' I)$ and thus $\tr\del{(\Sigma_p + \Delta\Sigma_\lambda + \varepsilon' I)^{-1} \Sigma_\lambda} \geq \frac{1}{2} \tr\del{(\Sigma_p + \varepsilon' I)^{-1} \Sigma_\lambda}$.

\smallskip

\noindent\emph{$h$-smoothed} ($\bar{\Delta} = h\varepsilon$): the density of $\Delta\lambda$ with respect to $\mu$ is at most $\frac{\Delta}{h} \leq \varepsilon$, so for every $a$, $\frac{p(a) + \varepsilon}{p(a) + \Delta\lambda(a) + \varepsilon} \geq \frac{\varepsilon}{\varepsilon + \frac{\Delta}{h}} \geq \frac{1}{2}$.
\end{proof}

\subsection{Fast Termination With Large Step Size in Finite Eluder Dimension Setting}
\label{app:large-step-size}
We found that when the $\overline{\SEC}_\varepsilon(\Gcal, \Lambda)$ is small, such as in the finite Eluder dimension case, the step size $\Delta_t$ in \Cref{alg:extended-f-design} can be chosen as large as $\Delta_t = \frac{1}{4\overline{\coverage}_\varepsilon(p_{t-1}, \lambda_t; \Gcal)}$ and the algorithm will terminate in a small number of iterations. Underpinning the proposition is a single-action reweighting bound for the coverage, which we record first.

\begin{lemma}[Single-action reweighting of the coverage]
\label{lem:coverage-reweighting}
Let $\overline{\coverage}_\varepsilon$ be the original coverage $\coverage_\varepsilon$ or one of the discrete and per-context generalized linear relaxed coverages of Lemma~\ref{lem:relaxed-coverage-admissible}. For every measure $p$, action $a \in \Acal$, and scalar $\Delta \geq 0$,
\begin{equation}
\overline{\coverage}_\varepsilon(p + \Delta\, \delta_a, \delta_a; \Gcal) \geq \frac{\overline{\coverage}_\varepsilon(p, \delta_a; \Gcal)}{1 + \Delta\, \overline{\coverage}_\varepsilon(p, \delta_a; \Gcal)},
\label{eqn:coverage-reweighting}
\end{equation}
with equality except for the generalized linear coverage when $\kappa > 1$.
\end{lemma}
\begin{proof}
    Denote $h(a) := g(a) - g'(a)$ for $g, g' \in \Gcal$,
\begin{itemize}
    \item For the original coverage, by the definition of coverage and the monotonicity of $x \mapsto \frac{1}{1/x + \Delta} = \frac{x}{1 + x \Delta}$,
    \begin{align*}
        \coverage_\varepsilon(p + \Delta\delta_a, \delta_a)
        = & \sup_{h \in \Dcal\Gcal} \frac{1}{(\varepsilon + \sum_{a' \in \Acal} p(a')\, h(a')^2 )/h(a)^2 + \Delta }
        \\
        = & \frac{1}{ 1 / \coverage_\varepsilon(p, \delta_a) +  \Delta}
        \\
        = &
        \frac{\coverage_\varepsilon(p, \delta_a)}{1 + \Delta\, \coverage_\varepsilon(p, \delta_a)}.
    \end{align*}

    \item For the Discrete and per-context generalized linear settings, both relaxed coverages take the form $\overline{\coverage}_\varepsilon(p, \delta_a; \Gcal) = \kappa^2\, \phi(a)^\top \bar{\Sigma}_p^{-1} \phi(a)$ with $\bar{\Sigma}_p = \Sigma_p + \varepsilon' I$ -- the discrete coverage $\frac{1}{p(a) + \varepsilon/|\Acal|}$ being the case $\phi(a) = e_a$, $\kappa = 1$, $\varepsilon' = \varepsilon/|\Acal|$. The reweighting $p \mapsto p + \Delta\delta_a$ is the rank-one update $\bar{\Sigma}_{p + \Delta\delta_a} = \bar{\Sigma}_p + \Delta\, \phi(a)\phi(a)^\top$, so the Sherman--Morrison formula gives $\phi(a)^\top \bar{\Sigma}_{p+\Delta\delta_a}^{-1} \phi(a) = \frac{\phi(a)^\top \bar{\Sigma}_p^{-1} \phi(a)}{1 + \Delta\, \phi(a)^\top \bar{\Sigma}_p^{-1} \phi(a)}$; multiplying by $\kappa^2$,
    \[
    \overline{\coverage}_\varepsilon(p + \Delta\delta_a, \delta_a) = \frac{\overline{\coverage}_\varepsilon(p, \delta_a)}{1 + \Delta\, \overline{\coverage}_\varepsilon(p, \delta_a)/\kappa^2} \geq \frac{\overline{\coverage}_\varepsilon(p, \delta_a)}{1 + \Delta\, \overline{\coverage}_\varepsilon(p, \delta_a)},
    \]
    where the inequality uses $\kappa \geq 1$, with equality when $\kappa = 1$.
\end{itemize}
\end{proof}

\begin{proposition}[Fast termination under the aggressive step]
\label{prop:cd-fast-converge}
Suppose $\Lambda = \cbr{\delta_a : a \in \Acal}$ and $\overline{\coverage}_\varepsilon$ is the original coverage $\coverage_\varepsilon$ or one of the discrete and per-context generalized linear relaxed coverages defined in Lemma~\ref{lem:relaxed-coverage-admissible}. Then \Cref{alg:extended-f-design} with the aggressive step size $\Delta_t = \frac{1}{4\, \overline{\coverage}_\varepsilon(p_{t-1}, \lambda_t; \Gcal)}$ terminates in at most $\lfloor 32\, \overline{\SEC}_\varepsilon(\Gcal, \Lambda) \rfloor$ iterations and outputs $p^* \in \co(\Lambda)$ satisfying Eq.~\eqref{eqn:monster-g}, certifying $\doec_{\gamma,\varepsilon}(\Gcal, \Lambda) \leq \frac{10}{\gamma}\, \overline{\SEC}_\varepsilon(\Gcal, \Lambda)$. In particular, the iteration count is $\lfloor 32\, \SEC_\varepsilon(\Gcal, \Lambda) \rfloor$ for the trivial relaxation $\overline{\coverage}_\varepsilon = \coverage_\varepsilon$, and $\iupboundlog{|\Acal|}$ and $\iupboundlog{\kappa^2 d}$ for the discrete and per-context generalized linear relaxed coverages, respectively.
\end{proposition}
\begin{proof}
\emph{Aggressive one-step stability.} We first record the consequence of Lemma~\ref{lem:coverage-reweighting} that drives the analysis: for every measure $p$ and action $a \in \Acal$, taking $\Delta = \frac{1}{4\, \overline{\coverage}_\varepsilon(p, \delta_a)}$ in the reweighting bound~\eqref{eqn:coverage-reweighting} (so $\Delta\, \overline{\coverage}_\varepsilon(p, \delta_a) = \tfrac14$),
\begin{align}
\overline{\coverage}_\varepsilon(p + \Delta\, \delta_a, \delta_a; \Gcal)
&\geq \frac{\overline{\coverage}_\varepsilon(p, \delta_a; \Gcal)}{1 + \Delta\, \overline{\coverage}_\varepsilon(p, \delta_a; \Gcal)}
\nonumber
\\
&\geq \tfrac34\, \overline{\coverage}_\varepsilon(p, \delta_a; \Gcal).
\label{eqn:aggressive-stability}
\end{align}

\emph{Potential function.} We track
\[
\Phi_t = - \underbrace{\sum_{s=1}^t \Delta_s\, \overline{\coverage}_\varepsilon(p_s, \lambda_s; \Gcal)}_{F_t} + \gamma \underbrace{\sum_{s=1}^t \tfrac12 \Delta_s\, \widehat{R}(\lambda_s)}_{G_t \geq 0},
\]
which satisfies $\Phi_0 = 0$ and $\Phi_t \geq -\overline{\SEC}_{\varepsilon/\norm{p_t}}(\Gcal, \Lambda)$; both properties hold for any relaxed coverage, independently of the step size (as in Theorem~\ref{thm:doec-leq-sec}).

\emph{Per-step decrease.} Consider an iteration $t$ at which the violation check (line~\ref{line:check-violation}) fires, and write $\lambda_t = \delta_{a_t}$, so $\Delta_t = \frac{1}{4\, \overline{\coverage}_\varepsilon(p_{t-1}, \delta_{a_t})}$ and $p_t = p_{t-1} + \Delta_t \delta_{a_t}$. Applying~\eqref{eqn:aggressive-stability} at $(p, a) = (p_{t-1}, a_t)$ gives $\overline{\coverage}_\varepsilon(p_t, \lambda_t) \geq \tfrac34\, \overline{\coverage}_\varepsilon(p_{t-1}, \lambda_t)$, so
\begin{align*}
\Phi_{t-1} - \Phi_t
&= \Delta_t\, \overline{\coverage}_\varepsilon(p_t, \lambda_t; \Gcal) - \tfrac{\gamma}{2} \Delta_t\, \widehat{R}(\lambda_t) \\
&\geq \Delta_t \rbr{ \tfrac34\, \overline{\coverage}_\varepsilon(p_{t-1}, \lambda_t; \Gcal) - \tfrac{\gamma}{2}\, \widehat{R}(\lambda_t) }
\geq \tfrac{\Delta_t}{4}\, \overline{\coverage}_\varepsilon(p_{t-1}, \lambda_t; \Gcal)
= \tfrac{1}{16},
\end{align*}
where the second inequality uses $\overline{\coverage}_\varepsilon(p_{t-1}, \lambda_t; \Gcal) \geq \gamma\, \widehat{R}(\lambda_t) + 8\, \overline{\SEC}_\varepsilon(\Gcal, \Lambda) \geq \gamma\, \widehat{R}(\lambda_t)$ at a violating iteration, and the final equality is the choice of $\Delta_t$.

\emph{Termination.} Were iteration $t = \lfloor 32\, \overline{\SEC}_\varepsilon(\Gcal, \Lambda) \rfloor$ reached, the per-step decrease would give
\[
\Phi_t \leq \Phi_0 - \frac{t}{16} < -\overline{\SEC}_\varepsilon(\Gcal, \Lambda).
\]
But $\Delta_s \leq \frac{1}{32\, \overline{\SEC}_\varepsilon(\Gcal, \Lambda)}$ for every $s$, so $\norm{p_t}_1 = \sum_{s=1}^t \Delta_s \leq 1$ and hence
\[
\Phi_t \geq -F_t \geq -\overline{\SEC}_{\varepsilon/\norm{p_t}}(\Gcal, \Lambda) \geq -\overline{\SEC}_\varepsilon(\Gcal, \Lambda),
\]
a contradiction. The algorithm therefore halts within
$\lfloor 32\, \overline{\SEC}_\varepsilon(\Gcal, \Lambda) \rfloor$ iterations with $\norm{ p_{t_0} }_1 \leq 1$, and $p^*$ satisfies the (LR) and (GC) properties, certifying $\doec_{\gamma,\varepsilon}(\Gcal, \Lambda) \leq \frac{10}{\gamma}\, \overline{\SEC}_\varepsilon(\Gcal, \Lambda)$.

\emph{Iteration counts.} The general bound $\lfloor 32\, \overline{\SEC}_\varepsilon(\Gcal, \Lambda) \rfloor$ specializes, via the relaxed SEC bounds of Lemma~\ref{lem:relaxed-sec-running-examples}, to $\lfloor 32\, \SEC_\varepsilon(\Gcal, \Lambda) \rfloor$ for the trivial relaxation and to $\iupboundlog{|\Acal|}$ and $\iupboundlog{\kappa^2 d}$ for the discrete and per-context generalized linear relaxed coverages, respectively.
\end{proof}

\begin{remark}
The above proof utilizes the specific structure that $\Lambda = \cbr{\delta_a: a \in \Acal}$ to allow $\Delta_t \coverage_\varepsilon(p_{t}, \lambda_t) \geq \frac{\Delta_t \coverage_\varepsilon(p_{t-1}, \lambda_t)}{1 + \Delta_t \coverage_\varepsilon(p_{t-1}, \lambda_t)}$.
This is consistent with the step size used in~\citet[][Section 4]{xu2020upper}.
In general, we don't have such inequality and in the proof of Proposition~\ref{prop:potential-function-decrease}, so we instead invoke the dilution-stability property (iv) of Assumption~\ref{assum:admissible-relaxed-coverage}.
This is similar in spirit to the potential decreasing argument in~\citet{agarwal14taming} which requires each member policy to be mixed in with some uniform exploration.
\end{remark}

In light of Proposition~\ref{prop:sec-eluder}, a consequence of Proposition~\ref{prop:cd-fast-converge} is that, when the Eluder dimension of $\Gcal$ is finite and $\Lambda = \cbr{\delta_a: a \in \Acal}$, Alg.~\ref{alg:extended-f-design} terminates in $\iupboundlog{\Edim(\Gcal, \sqrt{\varepsilon})}$ iterations.

\subsection{Computation Costs of Algorithm~\ref{alg:extended-f-design}}
\label{app:alg2-compute-cost}

The discrete and $h$-smoothed running examples do not require \Cref{alg:extended-f-design}: their relaxed exploitative F-designs are the closed-form (capped) inverse-gap-weighting distributions of Lemma~\ref{lem:relaxed-doec-bound}, both computable in $O(|\Acal|)$ time.

In this section, we focus on the per-context generalized linear reward setting, and analyze the running time cost of Algorithm~\ref{alg:extended-f-design}.
We count arithmetic operations, with reading $\hat{g}(a)$ and a feature $\phi(a) \in \RR^d$ costing $O(1)$ and $O(d)$, and assume line~\ref{line:compute-lambdat} is solved by scanning all $|\Acal|$ actions.\footnote{
As for online-oracle efficient algorithms,  ~\cite{zhu2022contextualb} designs an efficient procedure to compute a good exploration distribution assuming access to a linear optimization oracle over the action embedding space. It would be interesting to see if Algorithm~\ref{alg:extended-f-design} can be modified to work efficiently with such an oracle.
}

With $\bar{\Sigma}_p := \Sigma_p + \varepsilon' I$, the relaxed coverage of a point mass is the quadratic form $\overline{\coverage}_{\varepsilon}(p, \delta_a; \Gcal) = \kappa^2\, \phi(a)^\top \bar{\Sigma}_p^{-1} \phi(a)$. Using the Sherman--Morrison formula to maintain $\bar{\Sigma}_p^{-1}$, each iteration takes $O(|\Acal|\, d^2)$ time, matching the per-step cost of the linear F-design of~\citet[][Section~4]{xu2020upper}. Since $\Lambda = \cbr{\delta_a : a \in \Acal}$ here, the aggressive stepsize (Proposition~\ref{prop:cd-fast-converge}) applies, with iteration count equal to the relaxed SEC bound $\widetilde{O}(\kappa^2 d)$ of the linear example (Lemma~\ref{lem:relaxed-sec-running-examples}); a single call of \Cref{alg:extended-f-design} therefore costs $\widetilde{O}(|\Acal|\, \kappa^2 d^3)$.

\paragraph{End-to-end over the full run, and comparison to UCCB.} Within \generalfalcon, the relaxed exploitative F-design is solved once per round, so over a horizon $T$ the action-distribution cost is $T$ times the per-call cost above, while the offline regression oracle is called only $O(\log T)$ times. For the per-context generalized linear example under the aggressive step (one call in $\widetilde{O}(|\Acal|\, \kappa^2 d^3)$), writing $M(n)$ for the cost of one offline-regression solve on $n$ samples, relaxed \generalfalcon spends $\widetilde{O}(|\Acal|\, \kappa^2 d^3 T)$ on action computation and $O(M(T)\log T)$ on regression. The UCCB algorithm of~\citet{xu2020upper} instead spends $O(|\Acal|\, d^2 T^2)$ on action computation since it recomputes counterfactual actions, and $O(M(T)\, T)$ on regression, as it calls the oracle every round. Since $d \ll T$, $\kappa = O(1)$ for well-conditioned links, and $M(n) = \Omega(n)$, relaxed \generalfalcon is cheaper end-to-end.

\subsection{Proof of Proposition~\ref{prop:doec-small-sec-large} and Discussions}
\label{sec:appendix-doec-sec-loose}

\begin{proof}
We use the ``cheating code'' example in~\citet{agarwal2024non,amin2011bandits,jun20crush}.

We define the action space $\Acal$ to have two disjoint parts:  $\Acal_1 = \cbr{a_0, \ldots, a_{2^k-1}}$ and $\Acal_2 = \cbr{b_0, \ldots, b_{k-1}}$. The reward function class $\Gcal = \cbr{g^1, \ldots g^{2^k}}$, such that:
\[
\begin{cases}
g^i(a_j) = I(i = j), & j \in \cbr{0, \ldots, 2^{k}-1} \\
g^i(b_l) = \frac12 \cdot \text{the $l$-th bit of number $i$},&  j \in \cbr{0, \ldots, k-1}
\end{cases}
\]

We first show the upper bound on $\doec_{\gamma,\varepsilon}(\Gcal, \Lambda)$.
For any function $\hat{g} \in \Gcal$, let $\hat{a} = \argmax_{a \in \Acal} \hat{g}(a)$ be its greedy action.
We choose distribution
$p = (1 - \beta) \delta_{\hat{a}} + \beta \mathrm{Uniform}(\Acal_2)$ to certify an upper bound on $\doec_{\gamma,\varepsilon}(\Gcal, \Lambda)$.

Specifically, for any distribution $\lambda \in \Lambda$,
\begin{align*}
\coverage_\varepsilon(p, \lambda; \Gcal)
\leq &
\sup_{g,g' \in \Gcal} \frac{ (\EE_{a \sim \lambda}\sbr{ g(a) - g'(a) } )^2 }{\varepsilon + \beta \EE_{a \sim \mathrm{Uniform}(\Acal_2)} (g(a) - g'(a))^2} \\
\leq & \frac{4k}{\beta}
\end{align*}
where the second inequality uses that for any $g, g' \in \Gcal$ such that $g \neq g'$,
$\EE_{a \sim \mathrm{Uniform}(\Acal_2)} (g(a) - g'(a))^2 \geq \frac{1}{4k}$.

Thus, for any $\lambda \in \Lambda$,
\[
\EE_{a \sim \lambda} \hat{g}(a)
-
\EE_{a \sim \lambda} \hat{g}(a)
+
\frac{\coverage_\varepsilon(p, \lambda; \Gcal)}{\gamma}
\leq \beta + \frac{4k}{\beta \gamma}
\]
and choosing $\beta = \min\rbr{ 2\sqrt{ \frac{k}{\gamma} }, 1}$ implies that $p_\beta$ certifies $\doec_{\gamma,\varepsilon}(\Gcal, \Lambda) \leq 4 \rbr{\sqrt{\frac{k}{\gamma}} + \frac{k}{\gamma} }$.

We next show the lower bound on $\SEC_\varepsilon(\Gcal, \Lambda)$. Consider the sequence of measures $\delta_{a_0}, \ldots, \delta_{a_{2^k-1}}$. Then
\begin{align*}
\SEC_\varepsilon(\Gcal, \Lambda)
\geq &
\sum_{i=0}^{2^k-1} \coverage_{2^k \varepsilon}\rbr{  \sum_{j=0}^i \delta_{a_j}, \delta_{a_i}}
\\
\geq &
\sum_{i=0}^{2^k-1} \frac{ 1 }{ 2^k \varepsilon + 2 }
\\
= &
\frac{2^k}{2^k \varepsilon + 2 }
\\
\geq & \min\rbr{ 2^{k-2}, \frac1{2\varepsilon} }
\end{align*}
Here, the first inequality is by the definition of SEC; the second inequality is by the observation that $\coverage_{2^k \varepsilon}\rbr{  \sum_{j=0}^i \delta_{a_j}, \delta_{a_i}} \geq \frac{1}{2^k \varepsilon + 2}$ --
this is certified by taking $g = g^1$ and $g' = g^0$ for $i=0$
and
$g = g^i$ and $g' = g^0$
for $i \geq 1$; the rest of the calculations are by algebra.
\end{proof}

\paragraph{Implications to the regret bounds of \generalfalcon.} Given that $\doec_{\gamma, \epsilon}(\Gcal, \Lambda) \lesssim \sqrt{\frac{k}{\gamma}} + \frac{k}{\gamma}$, for \generalfalcon, we can set $\tau_m = 2^m$, $\gamma_m = (\frac{\tau_m}{\ln|\Fcal|})^{\frac{2}{3}} k^{\frac13}$, so that Theorem~\ref{thm:main-regret} gives a nontrivial regret bound of
$\iupboundlog{ (k \ln|\Fcal|)^{1/3} T^{2/3} + k \ln|\Fcal| }$.

In contrast, suppose we only use an exploration distribution that certifies the weaker DOEC bound $\doec_{\gamma, \epsilon}(\Gcal, \Lambda) \leq \min( \frac{2^k}{\gamma}, \frac{1}{\gamma \varepsilon})$, we cannot hope for Theorem~\ref{thm:main-regret} to give a regret bound better than $\min(\sqrt{2^k}, T)$.
The reason is as follows:
\begin{itemize}
\item If there exists some $m \geq 1$ with $\varepsilon_m < \frac{1}{2^k}$, then the regret bound is at least
    \begin{align*}
        \frac{\tau_{m+1}}{\gamma_{m+1}}
        \cdot
        \rbr{ 2^k
        +
        \gamma_{m+1}^2 \frac{\ln|\Fcal|}{\tau_{m}}}
        \geq
        \frac{\tau_{m+1}}{\gamma_{m+1}} \cdot \gamma_{m+1} \sqrt{\frac{2^k \ln|\Fcal|}{\tau_{m}}}
        \geq \sqrt{ 2^k }.
    \end{align*}

\item Otherwise, for all $m \geq 1$, $\varepsilon_m \geq \frac{1}{2^k}$. Then, for every $m \geq 2$,
the second term of the regret bound is at least
\[
        \frac{1}{\varepsilon_{m-1}}
        +
        \gamma_m^2 \varepsilon_{m-1}
        \geq \sqrt{ \frac{1}{\varepsilon_{m-1}} \cdot \gamma_m^2 \varepsilon_{m-1} }
        = \gamma_m,
\]
where for the first term, we lower bound the maximum by the $(m-1)$-th term in $\max_{n \in [M]}$, and for the second term, we lower bound the maximum by the $m$-th term in the $\max_{n \in \cbr{2,\ldots,M}}$ operation.
Thus, the regret bound is at least of order
\[ \max\rbr{\tau_1, \max_{m \geq 2} \frac{\tau_m}{\gamma_m} \cdot \max_{m \geq 2} \gamma_m} \geq \max_{m \geq 1} \tau_m \geq \frac{T}{M},
\]
which is vacuous.
\end{itemize}

\section{Proofs for Section~\ref{sec:doec-dec}}
\label{app:regret-analysis-online}

\subsection{Proof of \texorpdfstring{\Cref{thm:dec-leq-doec}}{Theorem 4}}

\decleqdoec*

\begin{proof}[Proof of \Cref{thm:dec-leq-doec}]
    For notational simplicity, we define $\EE_\lambda \sbr{g(\cdot)} \coloneq \EE_{a \sim \lambda}\sbr{g(\cdot)}$.
    Define $\lambda^*$ as the optimal distribution from $\Lambda$ that maximizes the expected reward by the true reward function $g^*$, i.e.,
    $\lambda^* = \argmax_{\lambda \in \Lambda} \EE_{a \sim \lambda}[g^*(a)]$.
    For any $p^*$ that certifies $\doec_\gamma(\Gcal, \Lambda) \leq V$, we decompose the difference into decision-error and estimation error as
    \begin{align*}
        &\quad \EE_{\lambda^*}[g^*(a)] - \EE_{p^*}[g^*(a)]
        \\
        &=
        \del{\EE_{\lambda^*}[\hat{g}(a)] - \EE_{p^*}[\hat{g}(a)]} + \del{\EE_{p^*}[\hat{g}(a)] - \EE_{p^*}[g^*(a)]} + \del{\EE_{\lambda^*}[g^*(a)] - \EE_{\lambda^*}[\hat{g}(a)]}
    \end{align*}

    For the second difference, by Cauchy-Schwarz followed by AM-GM, we have
    \begin{align*}
        \EE_{p^*} \sbr{\hat{g}(a)} - \EE_{p^*} \sbr{g^*(a)}
        \leq
        \sqrt{ \EE_{p^*}\sbr{\del{\hat{g}(a) - g^*(a)}^2} }
        \leq& \frac{1}{\gamma} + \frac{\gamma}{4} \EE_{p^*}\sbr{\del{\hat{g}(a) - g^*(a)}^2}
    \end{align*}
    For the last difference, we use the definition of coverage:
    \begin{align*}
        \EE_{\lambda^*} \sbr{g^*(a)} - \EE_{\lambda^*} \sbr{\hat{g}(a)}
        \leq&
        \sqrt{ \coverage_\varepsilon(p^*, \lambda^*; \Gcal) \cdot \del{\varepsilon + \EE_{p^*}\sbr{\del{\hat{g}(a) - g^*(a)}^2}}}
        \\
        \leq& \frac{1}{\gamma}\coverage_\varepsilon(p^*, \lambda^*; \Gcal) + \frac{\gamma}{4} \EE_{p^*}\sbr{\del{\hat{g}(a) - g^*(a)}^2} + \frac{\gamma \varepsilon}{4}
    \end{align*}

    Combining the bounds for the last two differences, we have
    \begin{align*}
        &\quad \EE_{\lambda^*}[g^*(a)] - \EE_{p^*}[g^*(a)]
        \\
        &\leq
        \del{\EE_{\lambda^*}[\hat{g}(a)] - \EE_{p^*}[\hat{g}(a)]}
        +
        \frac{\gamma}{2} \EE_{p^*}\sbr{\del{\hat{g}(a) - g^*(a)}^2}
        +
        \frac{1}{\gamma} \coverage_\varepsilon(p^*, \lambda^*; \Gcal)
        +
        \frac{1}{\gamma} + \gamma \varepsilon
        \\
        &\leq {V} + \gamma \EE_{p^*}\sbr{\del{\hat{g}(a) - g^*(a)}^2} + \frac{1}{\gamma} + \gamma \varepsilon
            \tag{$p^*$ certifies $\doec_\gamma(\Gcal, \Lambda) \leq V$}
        \\
        &\Rightarrow \EE_{\lambda^*}[g^*(a)] - \EE_{p^*}[g^*(a)] - \gamma \EE_{p^*}\sbr{\del{\hat{g}(a) - g^*(a)}^2} \leq V + \frac{1}{\gamma} + \gamma \varepsilon
        \\
        &\Rightarrow \dec_\gamma(\Gcal, \Lambda) \leq V + \frac{1}{\gamma} + \gamma \varepsilon
    \end{align*}
    which implies that $p^*$ certifies that
    $\dec_\gamma(\Gcal, \Lambda) \leq V + \frac{1}{\gamma} + \gamma \varepsilon$.
    The second part of the lemma follows by taking $V = \doec_{\gamma, \varepsilon}(\Gcal, \Lambda)$.
\end{proof}

\subsection{\squarecbf and its Analysis}
\label{sec:squarecbf}

Since Algorithm~\ref{alg:extended-f-design} finds a distribution $p$ that certifies that $\doec_\gamma(\Gcal, \Lambda) \leq \frac{1}{\gamma} \overline{\SEC}_{\varepsilon}(\Gcal, \Lambda)$, it also certifies that $\dec_\gamma(\Gcal, \Lambda) \leq \frac{1}{\gamma} \overline{\SEC}_{\varepsilon}(\Gcal, \Lambda) + \frac{1}{\gamma} + \gamma \varepsilon$. This motivates an online oracle-efficient algorithm,
\squarecbf (\Cref{alg:squarecb-f}), with the following guarantees:
\begin{corollary}
    \label{corol:regret-bound-squarecb-f}
    For any function class $\Fcal: \Xcal \times \Acal \to [0, 1]$, any set of action distributions $\Lambda$,
    \squarecbf
    is computationally efficient if step~\ref{line:compute-lambdat} of Algorithm~\ref{alg:extended-f-design} can be implemented efficiently.
    In addition, with probability at least $1 - \delta$, its regret is bounded by:
    $
        \REG(T, \squarecbf)
        \leq
        \iupbound{
        T
        \del{
        \frac1\gamma
        \max_{x \in \Xcal} \overline{\SEC}_{\varepsilon}(\Fcal_x, \Lambda)
        +
        \frac{1}{\gamma} + \gamma \varepsilon
        }
        +
        \gamma \offlinereg(\Fcal, T, \delta)
        }
    $
\end{corollary}

When $\overline{\SEC}_{\varepsilon}(\Fcal, \Lambda) = D \polylog(\frac{1}{\varepsilon})$ (as in Proposition~\ref{prop:sec-eluder} and Lemma~\ref{lem:relaxed-sec-running-examples}), \squarecbf achieves a regret of $\iupboundlog{\sqrt{D T \ln |\Fcal|}}$.
If step~\ref{line:compute-lambdat} of Algorithm~\ref{alg:extended-f-design} can be solved efficiently
, \squarecbf may be a computationally attractive variant of \etwod, albeit it may suffer from worse regret.

\begin{algorithm}[t]
    \SetAlgoLined\SetAlgoVlined\LinesNumbered
    \KwIn{total time length $T$, learning parameter $\gamma, \varepsilon=1/T$, function class $\Fcal$, action distributions $\Lambda$, online regression oracle $\onlineoracle$.}
    \For{$t = 1$ \KwTo $T$}{
        Compute $\hat{f}_t = \onlineoracle(\Fcal)\del{(x_i, a_i, r_i)_{i=1}^{t-1}}$ if $t > 1$, and an arbitrary element in $\Fcal$ otherwise.

        Observe context $x_t \in \Xcal$.

        Call \Cref{alg:extended-f-design} to compute a sampling distribution $p_t$ such that:
        \[
            \max_{\lambda \in \Lambda }
            \rbr{
            \EE_{a \sim \lambda}\sbr{\hat{f}_t(x_t,a)}
            -
            \EE_{a \sim p_t}\sbr{\hat{f}_t(x_t,a)}
            +
            \frac{1}{\gamma}\coverage_{\varepsilon}(p_t, \lambda; \Fcal_{x_t})
            }
            \leq \frac{10 \overline{\SEC}_{\varepsilon}(\Fcal_{x_t}, \Lambda)}{\gamma}
        \]

        Sample action $a_t \sim p_t$ and observe reward $r_t$.
    }
    \caption{\generalsquarecb: \squarecb with Exploitative F-designs} \label{alg:squarecb-f}
\end{algorithm}

\begin{proof}[Proof of Corollary~\ref{corol:regret-bound-squarecb-f}]
    The proof follows directly from \citet[Theorem 8.1 and Remark 4.1]{foster2021statistical}, which gives the regret bound for the \etwod algorithm with an inexact minimizer in the contextual bandit setting with general function approximation. According to that theorem,
    since for each $t$, \squarecbf's choice $p_t$ certifies that
    $\doec_{\gamma, \varepsilon}(\Fcal_{x_t}, \Lambda) \leq \frac{10}{\gamma} \overline{\SEC}_{\varepsilon}(\Fcal_{x_t}, \Lambda)$, by Theorem~\ref{thm:dec-leq-doec}, it also certifies that $\dec_\gamma(\Fcal_{x_t}, \Lambda) \leq \frac{10}{\gamma} \overline{\SEC}_{\varepsilon}(\Fcal_{x_t}, \Lambda) + \frac{1}{\gamma} + \gamma \varepsilon$. Thus it achieves the following regret bound:
    \begin{align*}
        \REG(T, \squarecbf) \leq \sup_{x \in \Xcal} \del{ \frac{10}{\gamma} \overline{\SEC}_{\varepsilon}(\Fcal_x, \Lambda) + \frac{1}{\gamma} + \gamma \varepsilon} \cdot T + \gamma \onlinereg(\Fcal, T, \delta)
    \end{align*}
\end{proof}

\section{Experiments}
\label{app:experiments}

We test whether replacing the per-round online regression oracle of \smoothigw~\citep{zhu2022contextuala} with calls to an offline oracle once per epoch and explored through the
closed-form exploitative $F$-design, costs anything in statistical performance.
We compare our \smoothedoetwod (\Cref{sec:offline-oracle-efficient-algorithm})
against
\smoothigw on
continuous-action contextual bandits built from large regression datasets.

\subsection{Setup}
\label{app:exp-setup}

\paragraph{Datasets.}
We use five OpenML regression datasets under the standard reduction
of~\citet{majzoubi2020efficient}: at round $t$ the learner observes features
$x_t$, plays an action $a_t\in[0,1]$, and earns $r_t = 1 - \abs{a_t - y_t}$ for
the normalized target $y_t\in[0,1]$. Reward is maximized at $a_t = y_t$, so the
target must be recovered from bandit feedback alone. Each dataset is processed
in a single online pass of length $T$: \texttt{auto-price}, \texttt{cpu-act},
and \texttt{wisconsin} ($T\!\approx\!10^6$), \texttt{black-friday}
($1.7{\times}10^5$), and \texttt{zurich} ($5.5{\times}10^6$).

\paragraph{Algorithms.}
We compare \smoothedoetwod against \smoothigw~\citep{zhu2022contextuala}, which is the state-of-the-art online regression oracle-efficient algorithm for continuous-action contextual bandits.
Both methods smooth the action space with a smooth-$h$ kernel and share the
same oracle class; they differ only in how the oracle is queried (offline, once per epoch, versus online, every round) and in the exploration rule.\footnote{Our implementation builds on the \textsc{smoothcb} codebase. } We try two regression oracles, a linear regression oracle and a nonlinear regression oracle that lifts the representation with random Fourier features approximating a Laplace kernel before fitting a linear model. Each smoothing parameter $h\in\{0.01,0.02,0.04,0.08,0.16,0.32,0.64\}$ yields a distinct smoothed problem instance, and we evaluate both methods on each.

\paragraph{Hyperparameter Tuning.}
We run a hyperparameter search over $10$ runs with different permutations on the dataset, selecting for each configuration with the best progressive-validation reward~\citep{blum1999beating}, and then evaluate the
selected configurations on another $10$ runs with different permutations.
The hyperparameter is the inverse-gap-weighting learning-rate multiplier $\gamma$, which both methods share and which trades off exploration against exploitation: we search it over a grid of $11$ values and select a separate $\gamma$ for each $(\text{dataset}, \text{oracle}, h)$ combination.
For \smoothigw we follow the pseudocode of \citet{zhu2022contextuala} rather than their released implementation. This is because their released codebase runs a Corral-type~\citep{agarwal17corralling} method over multiple values of $\gamma$ to sample an action and optimistically scales the suboptimality gap inside the IGW step as in~\cite{foster2021efficient}; tuning a single $\gamma$ keeps the comparison on equal footing with \smoothedoetwod.
\subsection{Results}
\label{app:exp-results}
The final reported metric is the realized average reward $R_T = \tfrac{1}{T}\sum_{t=1}^{T} r_t$
over all time steps in each run (mean$\pm$std over the $10$ runs) for each dataset and $h$ cell.
\Cref{tab:exp-linear,tab:exp-laplace} report the final realized reward for the linear and Laplace oracles over all datasets and $h$ values, while \Cref{fig:exp-curves-linear,fig:exp-curves-laplace} show the learning curves for each cell.

\paragraph{Overall comparison.}
Both methods are best at small $h$ and degrade as $h$ grows, but \smoothedoetwod is more sensitive to over-smoothing: at $h=0.64$ its reward collapses (to $\approx 0.80$ on \texttt{auto-price} and \texttt{cpu-act}) while \smoothigw stays near its small-$h$ value. This is because \smoothedoetwod outputs a policy from an $h$-smoothed policy class while \smoothigw's behavior policy is a mixture of policies from $h$-smoothed policy classes with the greedy policy (which is not smoothed).
The Laplace oracle consistently helps \smoothedoetwod and shrinks the gap with lower variance as well.

\Cref{fig:exp-curves-linear,fig:exp-curves-laplace} show the average-reward curves for every table cell. \smoothedoetwod learns fastest in the low-data regime, while \smoothigw, whose online oracle is updated one round at a time, frequently drops during its initial exploration before recovering. \smoothigw then closes the gap as $t$ grows.

\begin{table}[t]
    \centering
    \caption{Final realized average reward $R_T$ (mean$\pm$std over the held-out
    seeds) with the \textbf{linear} oracle, per dataset (rows) and $h$ (columns). \sfa\ $=\smoothedoetwod$ (ours), \sig\ $=\smoothigw$;
    \textbf{bold} marks the larger mean in each $h$ cell. Leading zeros are omitted.}
    \label{tab:exp-linear}
    {\footnotesize\setlength{\tabcolsep}{4pt}\begin{tabular}{lrrrrrrrr}
\toprule
Dataset & \multicolumn{2}{c}{$h=0.01$} & \multicolumn{2}{c}{$h=0.02$} & \multicolumn{2}{c}{$h=0.04$} & \multicolumn{2}{c}{$h=0.08$} \\
\cmidrule(lr){2-3} \cmidrule(lr){4-5} \cmidrule(lr){6-7} \cmidrule(lr){8-9}
 & \sfa & \sig & \sfa & \sig & \sfa & \sig & \sfa & \sig \\
\midrule
auto-price & $.921_{\pm .004}$ & $\mathbf{.930_{\pm .000}}$ & $.917_{\pm .014}$ & $\mathbf{.930_{\pm .001}}$ & $.908_{\pm .024}$ & $\mathbf{.930_{\pm .000}}$ & $.920_{\pm .004}$ & $\mathbf{.930_{\pm .001}}$ \\
cpu-act & $.946_{\pm .003}$ & $\mathbf{.959_{\pm .000}}$ & $.948_{\pm .004}$ & $\mathbf{.959_{\pm .000}}$ & $.946_{\pm .004}$ & $\mathbf{.959_{\pm .000}}$ & $.943_{\pm .005}$ & $\mathbf{.959_{\pm .000}}$ \\
wisconsin & $.855_{\pm .001}$ & $\mathbf{.858_{\pm .000}}$ & $.855_{\pm .001}$ & $\mathbf{.858_{\pm .000}}$ & $.856_{\pm .001}$ & $\mathbf{.858_{\pm .000}}$ & $.856_{\pm .002}$ & $\mathbf{.858_{\pm .000}}$ \\
black-friday & $.835_{\pm .002}$ & $\mathbf{.836_{\pm .000}}$ & $.835_{\pm .001}$ & $\mathbf{.836_{\pm .001}}$ & $.835_{\pm .002}$ & $\mathbf{.836_{\pm .000}}$ & $.836_{\pm .001}$ & $\mathbf{.836_{\pm .001}}$ \\
zurich & $.979_{\pm .001}$ & $\mathbf{.986_{\pm .000}}$ & $.979_{\pm .001}$ & $\mathbf{.986_{\pm .000}}$ & $.978_{\pm .000}$ & $\mathbf{.986_{\pm .000}}$ & $.970_{\pm .000}$ & $\mathbf{.986_{\pm .000}}$ \\
\bottomrule
\end{tabular}

\vspace{6pt}

\begin{tabular}{lrrrrrr}
\toprule
Dataset & \multicolumn{2}{c}{$h=0.16$} & \multicolumn{2}{c}{$h=0.32$} & \multicolumn{2}{c}{$h=0.64$} \\
\cmidrule(lr){2-3} \cmidrule(lr){4-5} \cmidrule(lr){6-7}
 & \sfa & \sig & \sfa & \sig & \sfa & \sig \\
\midrule
auto-price & $.916_{\pm .004}$ & $\mathbf{.930_{\pm .000}}$ & $.892_{\pm .000}$ & $\mathbf{.930_{\pm .001}}$ & $.804_{\pm .000}$ & $\mathbf{.930_{\pm .000}}$ \\
cpu-act & $.936_{\pm .003}$ & $\mathbf{.959_{\pm .000}}$ & $.902_{\pm .001}$ & $\mathbf{.959_{\pm .000}}$ & $.819_{\pm .000}$ & $\mathbf{.959_{\pm .000}}$ \\
wisconsin & $.857_{\pm .000}$ & $\mathbf{.858_{\pm .000}}$ & $.843_{\pm .000}$ & $\mathbf{.858_{\pm .000}}$ & $.781_{\pm .000}$ & $\mathbf{.858_{\pm .000}}$ \\
black-friday & $.834_{\pm .002}$ & $\mathbf{.836_{\pm .000}}$ & $.825_{\pm .000}$ & $\mathbf{.836_{\pm .001}}$ & $.771_{\pm .001}$ & $\mathbf{.836_{\pm .000}}$ \\
zurich & $.951_{\pm .000}$ & $\mathbf{.986_{\pm .000}}$ & $.910_{\pm .000}$ & $\mathbf{.986_{\pm .000}}$ & $.807_{\pm .000}$ & $\mathbf{.986_{\pm .000}}$ \\
\bottomrule
\end{tabular}
}
\end{table}

\begin{table}[t]
    \centering
    \caption{Final realized average reward $R_T$ with the \textbf{Laplace (RFF)}
    oracle. Conventions as in \Cref{tab:exp-linear}.}
    \label{tab:exp-laplace}
    {\footnotesize\setlength{\tabcolsep}{4pt}\begin{tabular}{lrrrrrrrr}
\toprule
Dataset & \multicolumn{2}{c}{$h=0.01$} & \multicolumn{2}{c}{$h=0.02$} & \multicolumn{2}{c}{$h=0.04$} & \multicolumn{2}{c}{$h=0.08$} \\
\cmidrule(lr){2-3} \cmidrule(lr){4-5} \cmidrule(lr){6-7} \cmidrule(lr){8-9}
 & \sfa & \sig & \sfa & \sig & \sfa & \sig & \sfa & \sig \\
\midrule
auto-price & $.906_{\pm .023}$ & $\mathbf{.932_{\pm .000}}$ & $.919_{\pm .002}$ & $\mathbf{.932_{\pm .000}}$ & $.914_{\pm .012}$ & $\mathbf{.932_{\pm .000}}$ & $.917_{\pm .005}$ & $\mathbf{.932_{\pm .000}}$ \\
cpu-act & $.948_{\pm .005}$ & $\mathbf{.960_{\pm .000}}$ & $.946_{\pm .003}$ & $\mathbf{.960_{\pm .000}}$ & $.947_{\pm .004}$ & $\mathbf{.960_{\pm .000}}$ & $.946_{\pm .002}$ & $\mathbf{.960_{\pm .000}}$ \\
wisconsin & $.863_{\pm .002}$ & $\mathbf{.863_{\pm .000}}$ & $.863_{\pm .003}$ & $\mathbf{.863_{\pm .000}}$ & $.862_{\pm .003}$ & $\mathbf{.863_{\pm .000}}$ & $.861_{\pm .001}$ & $\mathbf{.863_{\pm .000}}$ \\
black-friday & $\mathbf{.845_{\pm .008}}$ & $.839_{\pm .002}$ & $\mathbf{.845_{\pm .004}}$ & $.839_{\pm .002}$ & $\mathbf{.845_{\pm .004}}$ & $.839_{\pm .002}$ & $\mathbf{.847_{\pm .005}}$ & $.839_{\pm .002}$ \\
zurich & $.911_{\pm .002}$ & $\mathbf{.949_{\pm .003}}$ & $.911_{\pm .002}$ & $\mathbf{.949_{\pm .003}}$ & $.911_{\pm .002}$ & $\mathbf{.949_{\pm .003}}$ & $.909_{\pm .002}$ & $\mathbf{.949_{\pm .003}}$ \\
\bottomrule
\end{tabular}

\vspace{6pt}

\begin{tabular}{lrrrrrr}
\toprule
Dataset & \multicolumn{2}{c}{$h=0.16$} & \multicolumn{2}{c}{$h=0.32$} & \multicolumn{2}{c}{$h=0.64$} \\
\cmidrule(lr){2-3} \cmidrule(lr){4-5} \cmidrule(lr){6-7}
 & \sfa & \sig & \sfa & \sig & \sfa & \sig \\
\midrule
auto-price & $.917_{\pm .007}$ & $\mathbf{.932_{\pm .000}}$ & $.893_{\pm .001}$ & $\mathbf{.932_{\pm .000}}$ & $.805_{\pm .000}$ & $\mathbf{.932_{\pm .000}}$ \\
cpu-act & $.937_{\pm .001}$ & $\mathbf{.960_{\pm .000}}$ & $.907_{\pm .001}$ & $\mathbf{.960_{\pm .000}}$ & $.819_{\pm .000}$ & $\mathbf{.960_{\pm .000}}$ \\
wisconsin & $.859_{\pm .001}$ & $\mathbf{.863_{\pm .000}}$ & $.848_{\pm .000}$ & $\mathbf{.863_{\pm .000}}$ & $.784_{\pm .000}$ & $\mathbf{.863_{\pm .000}}$ \\
black-friday & $\mathbf{.846_{\pm .004}}$ & $.839_{\pm .002}$ & $.836_{\pm .003}$ & $\mathbf{.839_{\pm .002}}$ & $.779_{\pm .002}$ & $\mathbf{.839_{\pm .002}}$ \\
zurich & $.903_{\pm .002}$ & $\mathbf{.949_{\pm .003}}$ & $.877_{\pm .001}$ & $\mathbf{.949_{\pm .003}}$ & $.792_{\pm .000}$ & $\mathbf{.949_{\pm .003}}$ \\
\bottomrule
\end{tabular}
}
    \end{table}

\clearpage
\begin{figure}[H]
    \centering
    \includegraphics[width=\textwidth]{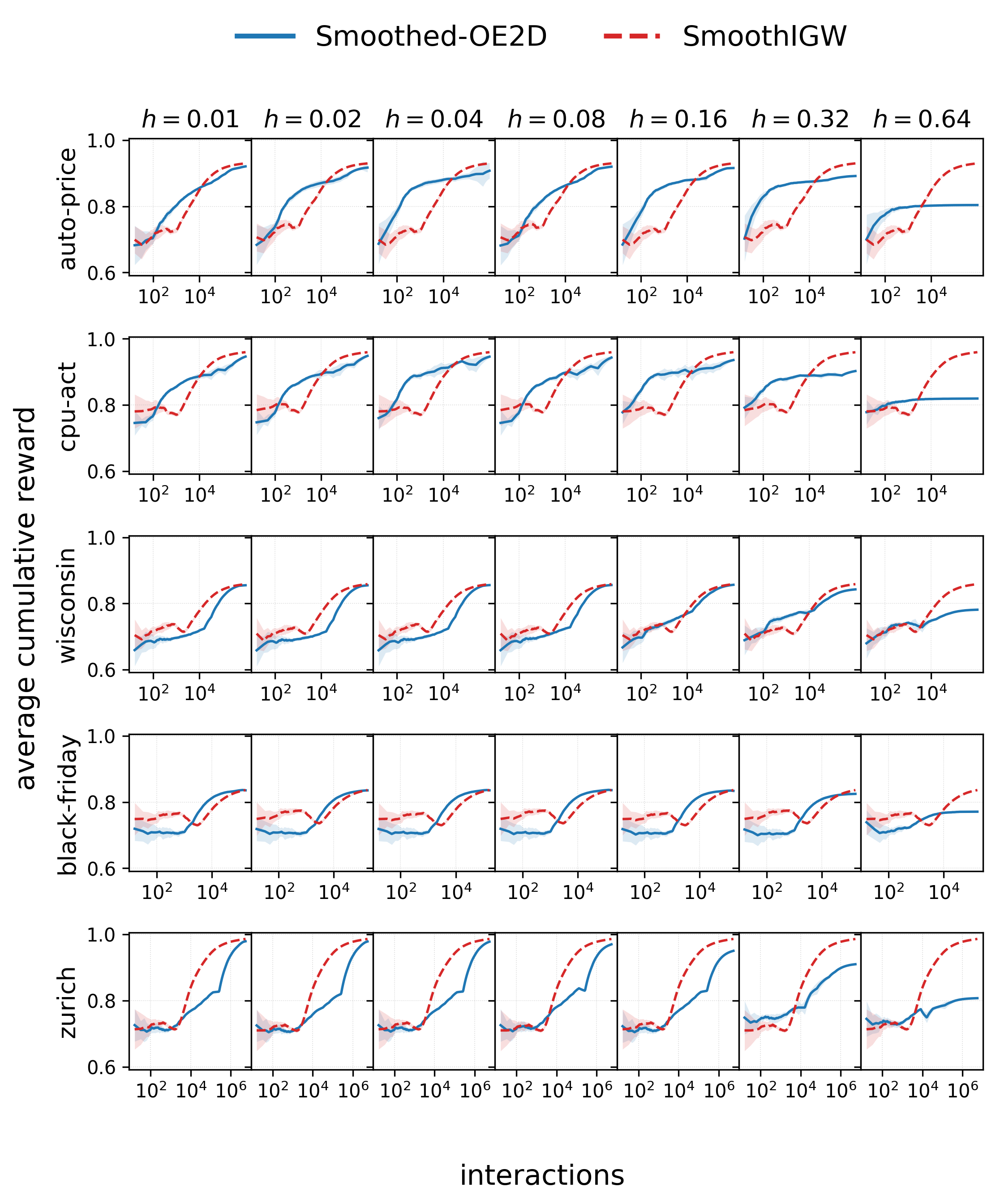}
    \caption{Average-reward learning curves with the \textbf{linear} oracle.}
    \label{fig:exp-curves-linear}
\end{figure}

\clearpage
\begin{figure}[H]
    \centering
    \includegraphics[width=\textwidth]{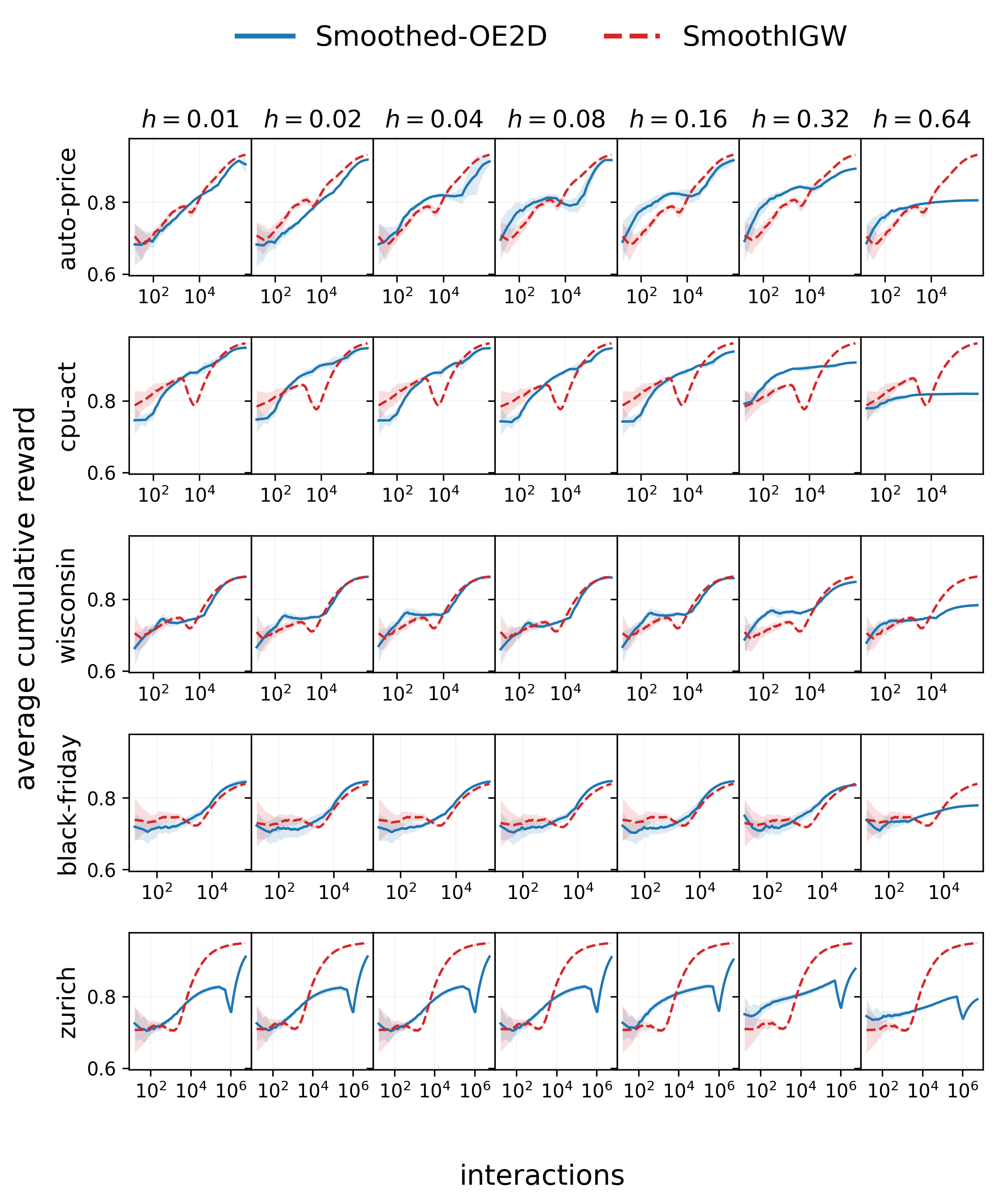}
    \caption{Average-reward learning curves with the \textbf{Laplace (RFF)} oracle.}
    \label{fig:exp-curves-laplace}
\end{figure}

\end{document}